\documentclass{article}
\usepackage[margin=1in]{geometry} 
\usepackage[T1]{fontenc}
\usepackage[utf8]{inputenc}
\usepackage{amsmath,amssymb,amsthm,mathtools}
\allowdisplaybreaks
\usepackage[cal=euler]{mathalpha}
\usepackage{cochineal}
\usepackage{microtype}
\usepackage{bm}
\usepackage{booktabs,longtable,array}
\usepackage{enumitem}
\usepackage{xcolor}
\usepackage{tikz}
\usetikzlibrary{positioning,arrows.meta,fit}
\usepackage{mdframed}
\usepackage{graphicx}
\usepackage[numbers,sort&compress]{natbib}
\usepackage{hyperref}

\newtheorem{theorem}{Theorem}[section]
\newtheorem{proposition}[theorem]{Proposition}
\newtheorem{corollary}[theorem]{Corollary}
\newtheorem{lemma}[theorem]{Lemma}
\newtheorem{conjecture}[theorem]{Conjecture}
\theoremstyle{definition}
\newtheorem{game}[theorem]{Game}
\surroundwithmdframed[linewidth=0.6pt,linecolor=black!70,innertopmargin=6pt,innerbottommargin=8pt,innerleftmargin=8pt,innerrightmargin=8pt,skipabove=10pt,skipbelow=10pt]{game}
\theoremstyle{remark}
\newtheorem{question}[theorem]{Open question}

\newcommand{\R}{\mathbb{R}}

\newcommand{\E}{\mathbb{E}}
\newcommand{\KL}{\mathrm{KL}}
\newcommand{\Var}{\mathrm{Var}}

\newcommand{\Val}{\mathrm{Val}}
\newcommand{\supp}{\mathrm{supp}}
\newcommand{\range}{\mathrm{range}}
\newcommand{\1}{\mathbf{1}}
\newcommand{\DeltaK}{\Delta([K])}
\newcommand{\inner}[2]{\left\langle #1,#2\right\rangle}
\newcommand{\cC}{\mathcal{C}}
\newcommand{\cG}{\mathcal{G}}

\newcommand{\cZ}{\mathcal{Z}}

\newcommand{\cE}{\mathcal{E}}

\newcommand{\argmin}{\mathop{\mathrm{arg\,min}}}
\newcommand{\argmax}{\mathop{\mathrm{arg\,max}}}

\numberwithin{equation}{section}

\title{The concentration game: Bayesian updating, regret, and information}
\author{Akshay Balsubramani \\ {\small \texttt{akshay@vac.bio}}}
\date{}

\newcommand{\eps}{\varepsilon}

\newcommand{\CCE}{\ensuremath{\mathrm{CCE}}}

\newcommand{\dd}{\mathrm{d}}

\providecommand{\one}{\mathbf{1}}
\providecommand{\simplex}{\Delta([K])}
\providecommand{\wt}[1]{\widetilde{#1}}

\newcommand{\akshay}[1]{}
\newcommand{\vac}[1]{}

\providecommand{\dd}{\mathrm{d}}

\providecommand{\simplex}{\Delta}
\providecommand{\wt}[1]{\widetilde{#1}}
\providecommand{\KL}{\mathrm{KL}}
\providecommand{\DeltaK}{\Delta_{K}}
\providecommand{\bm}[1]{\boldsymbol{#1}}

\makeatletter
\@ifundefined{argmax}{\DeclareMathOperator*{\argmax}{arg\,max}}{}
\@ifundefined{argmin}{\DeclareMathOperator*{\argmin}{arg\,min}}{}
\@ifundefined{diam}{}{}
\@ifundefined{rank}{}{}
\@ifundefined{supp}{\DeclareMathOperator{\supp}{supp}}{}
\@ifundefined{tr}{\DeclareMathOperator{\tr}{tr}}{}
\makeatother
 
\providecommand{\cC}{\mathcal{C}} 
\providecommand{\cE}{\mathcal{E}} 
\providecommand{\cG}{\mathcal{G}}

 \providecommand{\cZ}{\mathcal{Z}}

\providecommand{\E}{\mathbb{E}} \providecommand{\R}{\mathbb{R}}

\makeatletter
\@ifundefined{conjecture}{}{}
\@ifundefined{claim}{}{}
\@ifundefined{problem}{}{}
\@ifundefined{question}{}{}
\@ifundefined{fact}{}{}
\makeatother

\makeatletter
\g@addto@macro\normalsize{%
 \setlength{\abovedisplayskip}{3pt plus 1pt minus 1pt}%
 \setlength{\belowdisplayskip}{3pt plus 1pt minus 1pt}%
 \setlength{\abovedisplayshortskip}{0pt plus 1pt}%
 \setlength{\belowdisplayshortskip}{2pt plus 1pt minus 1pt}%
}
\def\thm@space@setup{%
 \thm@preskip=5pt plus 1pt minus 1pt
 \thm@postskip=\thm@preskip
}
\makeatother
\makeatletter
\def\bvapx@title{}
\AtBeginDocument{\providecommand{\phantomsection}{}}
\let\bvapx@maketitle\maketitle
\def\maketitle{\global\let\bvapx@title\@title\bvapx@maketitle}
\newcommand{\appendixtitle}{%
  \clearpage
  \suppressfloats[t]%
  \phantomsection
  \begin{center}
    {\LARGE\bfseries Appendices\ifx\bvapx@title\@empty\else\space to ``\bvapx@title''\fi\par}%
  \end{center}
  \vspace{1.5\baselineskip}%
}
\makeatother

\providecommand{\akshay}[1]{}\renewcommand{\akshay}[1]{}%
\providecommand{\vac}[1]{}\renewcommand{\vac}[1]{}%
\begin{document}
\maketitle
\vspace{-2.2em}

\begin{abstract}
We give a two-player zero-sum repeated game between a learner and nature whose value identity generates Bayesian updating and an exact accounting of exponential-weights regret at once, and supplies the comparator-class variational form that a wide class of concentration phenomena share.
The terminal payoff is the most a comparator can gain at fixed relative entropy from the prior, and the one-step constraint is an information budget on nature's move under the learner's mixed action. With the learner's move otherwise unrestricted, Gibbs/Bayes weights emerge as its unique Bellman equalizer --- the mixed action that makes the per-round loss independent of which direction nature moves --- with log-partition functions playing the role of value functions.
The regret decomposes exactly into three parts: a per-round information loss reflecting the variation in observed outcomes, an additive \emph{retempering drift} that accounts exactly for any change of measurement scale between rounds, and the information the comparator carries relative to the prior. The variance and bounded-range proxies that drive standard regret bounds are looser relaxations of this decomposition, which holds generally and governs them all. Both players' strategies are read off from the decomposition term by term, and repeated play yields an information-theoretic ledger of self-play in place of the usual quadratic-variation surrogate.
The same comparator-class geometry accounts for the classical large-deviation bounds, and methods across bandits, posterior sampling, aggregation, and boosting are specializations of the one regret decomposition.
\end{abstract}

\section{Introduction}
\label{sec:intro}
We develop the idea that Bayesian updating and exponential-weights regret are two faces of one identity: the value of a two-player zero-sum game played over several ($T$) rounds.
Here we call it the concentration game, because a wide class of concentration phenomena against comparators draws its form from that same identity.
This unifies several research communities' separate formalisms for a cluster of interrelated information accounting problems.
Online learning derives exponential-weights regret from a learning-rate tradeoff; Bayesian inference derives the posterior from a likelihood-ratio update; concentration of measure and large-deviation theory derive tail bounds from a moment-generating-function inequality; repeated game theory derives equilibrium convergence from no-regret dynamics; and boosting derives margin guarantees from a weighted-error potential.

The concentration game's value identity accounts for all of these.
Bayes/exponential-weights updating is its equilibrium strategy, the moment-generating function is its per-round loss, and the log-partition function is its value. Repeated-game equilibria, Thompson sampling, log-pooling, and AdaBoost are specializations of it, and concentration of measure and Sanov large deviations draw their comparator-class variational form from it.
The game is sequential by nature, because the phenomena it unifies are progressive: each round, belief is updated and the information budget that drives concentration is used up.
Game~\ref{game:central} below states the basic framework.
Retempering, budget allocation, and partial information are then layered onto the same rounds in Sections~\ref{sec:variable}--\ref{sec:partial}.

\subsection{The game and the central identity}
\label{sec:repeated-game}

Fix a strictly positive prior $\pi$ on a finite set of actions $[K]:=\{1,\dots,K\}$, a comparator complexity budget $\Gamma\ge 0$, and a schedule $(\eta_t,q_t)_{t=1}^T$ of measurement scales and per-round information budgets, fixed in advance or chosen adaptively from history.
A \emph{comparator} is a fixed mixture over the actions against which the learner's cumulative loss is compared, and $\Gamma$ caps its relative entropy from the prior.
The state entering round $t$ is the cumulative centered score $S_{t-1}:=\sum_{s<t}z_s$ of nature's past moves ($S_0=0$; the moves $z_s$ are set out in the game box below).
This has zero mean under the mixture $p_t$ the learner plays that round, $\inner{p_t}{z_t}=0$:
nature's move redistributes among the actions under the learner's own weighting, without shifting them all together,
so a loss shared equally by every action is invisible to the game (Section~\ref{sec:primitive}).
Each round the learner sets a mixture over the actions, nature then moves the running state $S_t$ under a single one-step budget, and one potential --- the log-partition function --- controls the whole run (Section~\ref{sec:fixed}).
Nature's one-round budget is measured by the \emph{rescaled centered cumulant generating function} (RCGF) $Q_\eta(p,z):=\eta^{-2}\log\sum_i p(i)e^{\eta z(i)}$ of its centered move $z$ under the learner's mixture $p$; the unrescaled $\eta^2 Q_\eta=\log\E_p[e^{\eta z}]$ is the centered CGF, and the $\eta^{-2}$ factor makes the small-scale limit half the variance (Section~\ref{sec:primitive}).
The RCGF measures how large nature's move is at the scale $\eta$,
and the scale sets which feature of the move that size responds to:
the variance of $z$ at small $\eta$, its worst tail at large $\eta$.

\begin{game}[The concentration game]
\label{game:central}
Fix a strictly positive prior $\pi\in\DeltaK$, a comparator complexity budget $\Gamma\ge 0$, and a schedule $(\eta_t,q_t)_{t=1}^T$ of measurement scales and per-round RCGF budgets (predictable: $(\eta_t,q_t)$ may depend on $z_1,\dots,z_{t-1}$). Start from $S_0=0$. For each round $t=1,\dots,T$:
\begin{enumerate}[label=\arabic*.,leftmargin=1.6em,topsep=2pt,itemsep=1pt]
\item Learner picks a mixture $p_t\in\DeltaK$ over the actions;
\item Nature, observing $p_t$, picks a centered $z_t\in\R^K$ with $\inner{p_t}{z_t}=0$ and $Q_{\eta_t}(p_t,z_t)\le q_t$;
\item State advances by $S_t:=S_{t-1}+z_t$.
\end{enumerate}
The terminal payoff to nature is $L_\Gamma(S_T)$, the support function of a relative-entropy ball around the prior (Section~\ref{sec:terminal}):
$$ L_\Gamma(S):=\sup\{\inner{\nu}{S}:\nu\in\DeltaK,\ \KL(\nu\|\pi)\le\Gamma\} $$
\end{game}

The game leaves the learner free to play any action distribution $p_{t}$.
We focus on the \emph{Gibbs tilt} $p_t=G_{\eta_t}(S_{t-1})$, that is $p_t(i)\propto\pi(i)e^{\eta_t S_{t-1}(i)}$
(the map $G_\eta(S)\colon i\mapsto \pi(i)e^{\eta S(i)}/\sum_j\pi(j)e^{\eta S(j)}$ is used throughout).
Playing this form parameterizes the learner's move without restricting it: at any scale $\eta>0$, a mixture $p$ giving every action positive weight is reached from the state $S(i):=\eta^{-1}\log\bigl(p(i)/\pi(i)\bigr)$, since $\pi(i)e^{\eta S(i)}=p(i)$ already sums to one and so $G_\eta(S)=p$.

The cumulative regret against any comparator $\nu\in\DeltaK$ is $R_T^c(\nu):=\sum_{t=1}^T\inner{p_t-\nu}{c_t}$, the learner's cumulative loss minus the comparator's, where $c_t$ is the round-$t$ loss in the standard sign convention ($z_t=\inner{p_t}{c_t}\1-c_t$ is the centered version, so that $R_T^c(\nu)=\inner{\nu}{S_T}$).

The main identity (Theorem~\ref{thm:variable}, Section~\ref{sec:variable}) is that this cumulative regret decomposes exactly into three structurally distinct terms:
\begin{equation}\label{eq:intro-exact}
R_T^c(\nu)
=
\sum_{t=1}^T \eta_t Q_t(c)
+
D_T
+
B_T(\nu)
\end{equation}
where the leading sum is the \emph{intrinsic-time loss} \cite{balsubramani2026adaptive}, built from the round-$t$ centered RCGF $Q_t(c)\ge 0$ (the intrinsic-time density) weighted by the scale $\eta_t$.
Intrinsic time measures the run's length in information.
Rounds elapsed do not enter.
Write $W_\eta(S):=\eta^{-1}\log\sum_i\pi(i)e^{\eta S(i)}$ for the scale-$\eta$ log-partition of the state, the potential of Section~\ref{sec:fixed}.
The \emph{retempering drift} $D_T=\sum_{t=1}^{T-1}\bigl(W_{\eta_{t+1}}(S_t)-W_{\eta_t}(S_t)\bigr)$ is the exact loss from changing the measurement scale---retuning the inverse temperature---between rounds, each summand the change in $W_\eta$ at a fixed state, so $D_T=0$ at constant scale.
The \emph{terminal relative-entropy transport} $B_T(\nu)=\eta_T^{-1}\bigl(\KL(\nu\|\pi)-\KL(\nu\|\rho_{T,\eta_T})\bigr)$, with $\rho_{T,\eta_T}:=G_{\eta_T}(S_T)$ the terminal Gibbs posterior, is the comparator information carried from the prior $\pi$ to that posterior.
The difference of the two relative entropies collapses to a single expectation, $\eta_TB_T(\nu)=\E_\nu\bigl[\log(\rho_{T,\eta_T}/\pi)\bigr]$,
the log-likelihood ratio of posterior to prior averaged under the comparator: how far the run has moved the belief in the comparator's own direction (Section~\ref{sec:confidence}).
Taken together the three terms are the run's \emph{ledger}: a closed account of the regret, one entry per mechanism, which sums to it exactly and leaves no remainder.
The word is used in that sense throughout, and each later setting has its own ledger, the same three entries evaluated on whatever score sequence that setting scores.
Because \eqref{eq:intro-exact} is an identity rather than an inequality, both the learner's and nature's strategies are read off term by term: the learner plays the Bellman equalizer (Gibbs at every round), and nature saturates the round budget $Q_t(c)=q_t$ in the worst comparator direction.
Figure~\ref{fig:game-trajectory} follows one realized run, with the three terms accumulating along it.

The same basic structure then survives changing the sufficient statistic (the score sequence the learner actually applies Bayes to, which need not be the raw loss), the comparator class, the observation structure, the geometry, or the action rule (Appendix~\ref{sec:other-geometries}). Each modification changes surface features while the Bellman backbone stays fixed.

The tilt is the exponential-weights / Bayes update (multiplicative weights, Hedge), and equally the entropic Follow-the-Regularized-Leader (FTRL) iterate $\argmax_p\{\inner{p}{S_{t-1}}-\eta_t^{-1}\KL(p\|\pi)\}$ on the cumulative centered scores, the two coinciding by the Gibbs variational identity \eqref{eq:gibbs-var}.
Section~\ref{sec:fixed} derives it as the game's Bellman equalizer: the play that makes the round's loss a deterministic function of nature's budget alone, whichever feasible move nature selects.
Equalizing is weaker than round-value minimaxity, since the minimax play minimizes the round's worst-case value where the equalizer flattens the loss across nature's replies, and the two need not agree (Section~\ref{sec:fixed-k2-exact}).

Of the four ingredients, only the prior $\pi$ is a modeling decision, encoding the prior on action quality.
The measurement scale $\eta_t$ gives ``the size of a move'' its meaning, dual to $\pi$ through the exponential-family identification $p\propto\pi\,e^{\eta S}$; it is an inverse temperature, large $\eta$ cold and concentrating the tilt, small $\eta$ hot and diffuse.
The constraint $Q_{\eta_t}(p_t,z_t)\le q_t$ is an information budget on nature's one-round move: by the identity $\eta^2 Q_\eta(p,z)=\KL(p\|p^+)$ (with $p^+(i)\propto p(i)e^{\eta z(i)}$ the within-round update), $q_t$ bounds the prior-to-posterior relative-entropy transfer at scale $\eta_t$. And $\Gamma$ is a matching budget on the comparator class: $L_\Gamma(S_T)$ is the largest terminal gain any distribution within relative entropy $\Gamma$ of the prior can extract.
An uncapped nature would play unstructured noise; the cap is the only constraint the game imposes (Section~\ref{sec:primitive} reads it as a \emph{coincidence budget}).

The schedule and the horizon are both data of the game, given to both players and optimized by neither, until Section~\ref{sec:variable} promotes the scale to a choice the learner makes (Game~\ref{game:cgf}).
Fixing $T$ changes nothing in the accounting: the decomposition of Theorem~\ref{thm:variable} holds with equality at every prefix, so stopping at any $t\le T$ leaves the same three entries, evaluated at horizon $t$, and Section~\ref{sec:confidence} quantifies the anytime-versus-fixed-horizon difference.

\begin{figure}[tb]
\centering
\begin{tikzpicture}[
 font=\small,
 hub/.style={rectangle, rounded corners=3pt, draw, fill=black!5, align=center, inner sep=7pt},
 dom/.style={rectangle, rounded corners=2pt, draw, align=center, inner sep=4pt, text width=3.4cm},
 link/.style={draw=black!45, line width=0.5pt}
]
\node[hub] (hub) at (0,0) {%
 {\normalsize\bfseries The concentration game}\\[1pt]
 {\footnotesize two-player zero-sum, $T$ rounds}\\[4pt]
 $R_T^c(\nu)=\underbrace{\sum_t \eta_t\,Q_t(c)}_{\substack{\text{intrinsic-time}\\\text{loss}}}\;+\;\underbrace{D_T}_{\substack{\text{retempering}\\\text{drift}}}\;+\;\underbrace{B_T(\nu)}_{\substack{\text{relative-entropy}\\\text{transport}}}$\\[6pt]
 {\footnotesize Bayes\,/\,exp-weights $=$ equilibrium}\\
 {\footnotesize $\log Z =$ value}};
\node[dom] (ol) at (-5.9,2.5) {{\bfseries Online learning}\\{\footnotesize Hedge, FTRL, adaptive regret}};
\node[dom] (cc) at (-5.9,0) {{\bfseries Concentration}\\{\footnotesize Hoeffding, Chernoff, Sanov}};
\node[dom] (av) at (-5.9,-2.5) {{\bfseries Anytime-valid inference}\\{\footnotesize confidence sequences, prior-posterior ratios}};
\node[dom] (ba) at (5.9,2.5) {{\bfseries Bayes \& aggregation}\\{\footnotesize robust Bayes, log-pooling}};
\node[dom] (bp) at (5.9,0) {{\bfseries Bandits \& partial info}\\{\footnotesize feedback graphs, contextual, Thompson}};
\node[dom] (bg) at (5.9,-2.5) {{\bfseries Boosting \& games}\\{\footnotesize AdaBoost, self-play ledger}};
\draw[link] (hub) -- (ol);
\draw[link] (hub) -- (cc);
\draw[link] (hub) -- (av);
\draw[link] (hub) -- (ba);
\draw[link] (hub) -- (bp);
\draw[link] (hub) -- (bg);
\end{tikzpicture}
\caption{The concentration game and its specializations. A two-player zero-sum game (center) has Bayes\,/\,exponential-weights updating as its Bellman equilibrium and the log-partition function as its value; its regret decomposes exactly as $R_T^c(\nu)=\sum_t\eta_t Q_t(c)+D_T+B_T(\nu)$ --- an intrinsic-time loss, a retempering drift, and terminal relative-entropy transport (Theorem~\ref{thm:variable}, terms as in the text). Each surrounding family's regret \emph{decomposition} is recovered by specializing the prior, the measurement scale, the per-round budget, the comparator class, and the feedback structure---one paragraph of bookkeeping around the same equilibrium; the sharp rate in each family still needs a family-specific bound on the estimated per-round RCGF.}
\label{fig:overview}
\end{figure}
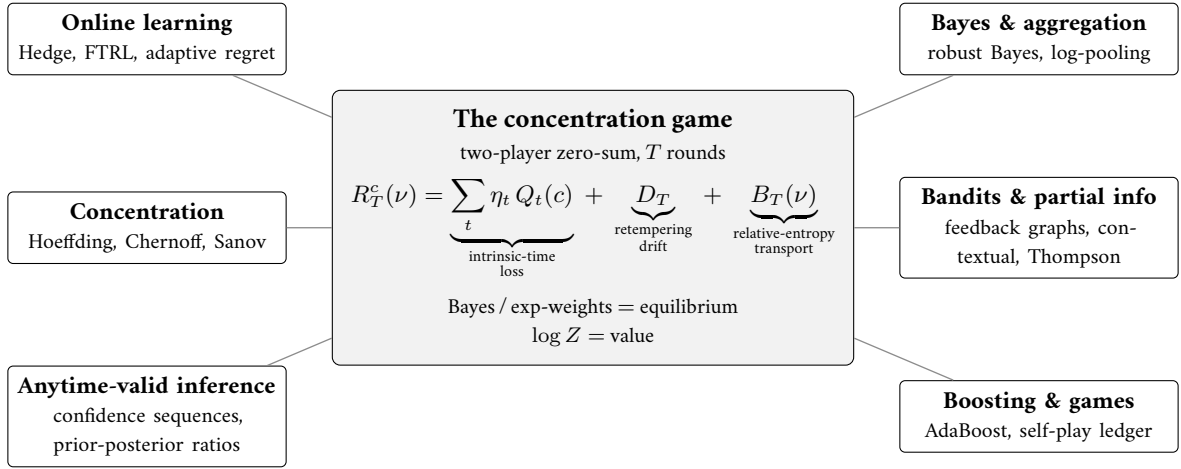

\begin{figure}[t]
\centering
\includegraphics[width=\linewidth]{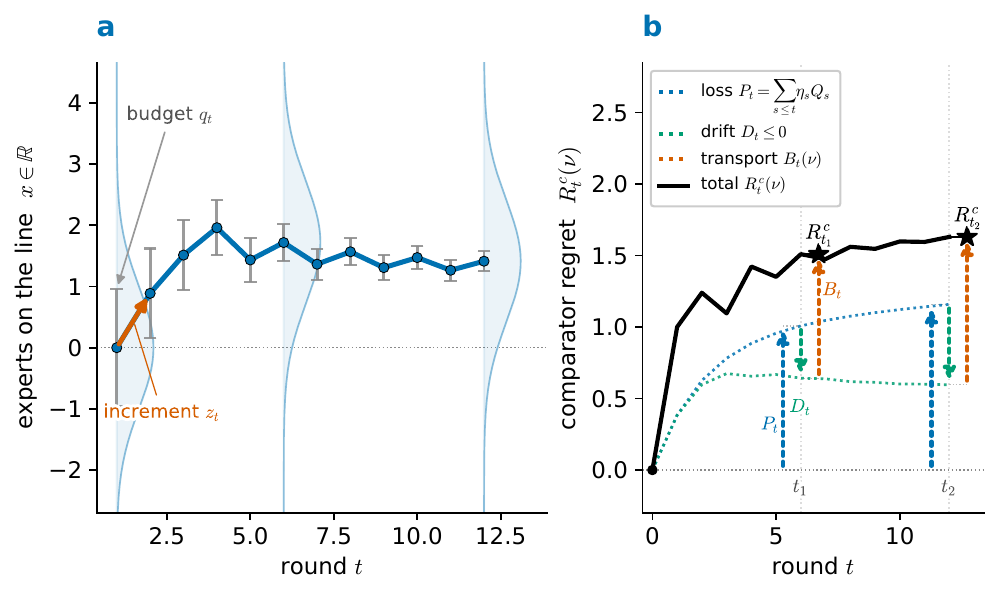}
\caption{\textbf{One trajectory of the concentration game, and its information decomposition.} A representative realized path, with a continuum of experts on the line $x\in\R$ under a Gaussian reference $\pi$ and a schedule $(\eta_t,q_t)$ that grants nature more of its per-round budget in the early rounds. \emph{(a)}~Each round the learner plays the Bellman-equalizer Gibbs tilt $p_t=G_{\eta_t}(S_{t-1})\propto\pi\,e^{\eta_t S_{t-1}}$ (faint densities), whose mean $m_t=\inner{p_t}{x}$ is the realized walk (blue); nature answers with a centered increment $z_t$. The whiskers mark how far nature may move the state in one round under the RCGF budget $q_t$. This reach scales as $\sqrt{q_t}$ (the centered RCGF is second order in the step, so in the small-step regime the budget-saturating move obeys $\delta\approx\sqrt{2q}$; exactly, $\delta=\eta^{-1}\operatorname{arcosh}(e^{\eta^2 q})$ for two symmetric experts, Proposition~\ref{prop:k2-exact}). As the budget is used up---most of it early---the reach shrinks and the state concentrates. \emph{(b)}~The comparator regret decomposes at every horizon, $R_t^c(\nu)=P_t+D_t+B_t(\nu)$ with loss $P_t=\sum_{s\le t}\eta_s Q_s$, starting from the origin ($R_0^c=0$: before any round every term is zero). The solid curve is the total. The dotted waterfall at horizons $t_1$ and $t_2$ splits it into the intrinsic-time loss $P_t$ (up), the retempering drift $D_t\le 0$---the loss from changing scale between rounds, a gain to the learner here (down)---and the terminal relative-entropy transport $B_t(\nu)=\big[\KL(\nu\|\pi)-\KL(\nu\|\rho_{t,\eta_t})\big]/\eta_t$ (up), with $\rho_{t,\eta_t}$ the horizon-$t$ Gibbs posterior. The comparator is $\nu=\rho_{t_2,\eta_{t_2}}$, the posterior the game settles on.}
\label{fig:game-trajectory}
\end{figure}

\subsection{Main results}
\label{sec:headline}

The main result is that the decomposition \eqref{eq:intro-exact} holds with equality at every horizon, for every comparator, and for every predictable schedule: the Bellman equalizer composes across rounds, and each of the three terms is a sum of per-round contributions (Theorem~\ref{thm:variable}).
The drift is the closed-form loss from changing measurement scale between rounds, its per-step form a prior-to-posterior relative entropy (Proposition~\ref{prop:drift-continuous}); at a fixed scale it vanishes and the decomposition reduces to the entropic mix-loss account carried by the log-partition potential, classical in its inequality-relaxed form.

The variance and bounded-range quantities that drive standard regret bounds are derived relaxations of that per-round primitive. The RCGF tends to $\tfrac12\Var_p(z)$ as $\eta\downarrow 0$ (Corollary~\ref{cor:variance-shadow}) and at fixed scale is bounded above by a variance proxy and by a range proxy, neither refining the other (Section~\ref{sec:hierarchy}). Substituting either for the centered RCGF yields a looser regret upper bound, and quadratic-variation bounds appear as its small-$\eta$ limit.

Under a single cumulative budget $V$ on nature's total information, in place of the per-round caps, the horizon drops out of the value.
At constant scale it equals $\Gamma/\eta+\eta V$ from a finite horizon on, so the square-root envelope $2\sqrt{\Gamma V}$ is attained and the shortfall the one-round game exhibits belongs to that single round (Proposition~\ref{prop:horizon-free}).
Fixing the horizon therefore concedes nothing to nature.
A known horizon instead settles the scale, and it releases the slack an anytime guarantee holds in reserve for the times it is not consulted --- the drawdown of the comparator's transport trajectory, how far it has ever fallen below its current level (Section~\ref{sec:confidence}).
Those two savings separate an anytime-valid confidence sequence from a fixed-horizon interval read once.

The same identity gives the regret decomposition of a wide family of methods (Figure~\ref{fig:overview}): Gibbs conditioning, Sanov's theorem, low-noise luckiness, confidence sets, repeated-game equilibria, bandits, feedback graphs, contextual bandits, Thompson sampling, and AdaBoost (Sections~\ref{sec:concentration}--\ref{sec:boosting}), and logarithmic pooling (Appendix~\ref{app:log-pooling}).
Each is obtained by specializing the prior, the measurement scale, the per-round budget, the comparator class, and the feedback structure --- one paragraph of bookkeeping at the level of the \emph{decomposition}.
The contribution is the accounting; the rates themselves are not new, and the sharp rate in each family still needs a family-specific bound on the per-round RCGF.
Converting a decomposition into a sharp rate, such as the minimax bandit $O(\sqrt{KT})$ dependence, is settled by the estimated per-round RCGF, itself a closed form in the true losses and observation geometry (Proposition~\ref{prop:bandit-exact-value}).
That closed form recovers the classical variance- and range-based envelopes as small-scale relaxations the worst case never improves on.
The self-play ledger of Section~\ref{sec:games} instead reads the per-round loss off the realized path, giving equilibrium rates faster than the $1/\sqrt{T}$ worst case whenever intrinsic time is sublinear.

The specific contribution is that these separate literatures share one primitive --- the terminal relative-entropy-ball payoff under a one-scale RCGF constraint --- from which each draws its comparator-class variational form.
Earlier work gets variance surrogates because it starts downstream; starting at the RCGF recovers those surrogates as relaxations and makes the identities visible.

\section{Coordinates, the terminal payoff, and the per-round primitive}
\label{sec:primitive}

The game runs on two primitives --- the terminal payoff $L_\Gamma$ and the per-round RCGF --- both expressed in one set of centered coordinates. We work on a finite set $[K]=\{1,\dots,K\}$ with a strictly positive prior $\pi\in\DeltaK$. The learner's move each round is an \emph{unrestricted} mixed action $p_t\in\DeltaK$---the game constrains nature alone---recorded in exponential-family coordinates as the Bayes/exponential-weights update of a composite score sequence $c_t\in\R^K$:
\begin{equation}\label{eq:bayes-update}
p_t(i)
=
\frac{\pi(i)e^{-\eta_t C_{t-1}(i)}}{\sum_{j=1}^K \pi(j)e^{-\eta_t C_{t-1}(j)}},
\qquad
C_t(i):=\sum_{s=1}^t c_s(i),
\qquad
C_0\equiv 0
\end{equation}
The correspondence is with the \emph{cumulative} score: distinct interior plays come from distinct score classes, while $c_t=C_t-C_{t-1}$ is pinned, again up to an additive constant, only once the next play and the next scale are given as well. Only the product $\eta_t C_{t-1}$ enters $p_t$, so $\eta_t$ is a choice of scale that couples the move to nature's information budget. This same form is the game's Bellman equalizer (Section~\ref{sec:fixed}), the play the analysis singles out. The scores may be literal losses, predictable offsets, importance-weighted estimates, optimistic residuals, signed margins, or any statistic to which the learner applies Bayes---the game is defined relative to the chosen sufficient statistic, with no ambient ``true'' loss behind it.

This costs no generality: \eqref{eq:bayes-update} is a bijection from $\R^K/\R\1$ onto the interior of $\DeltaK$, so it \emph{reparameterizes} arbitrary play instead of restricting it, and its inverse reads the state as the round's posterior-to-prior log-ratio at the round's own scale (Appendix~\ref{app:reparameterization}).

\subsection{Centered excess-score coordinates}

For round $t$, define the learner mean $\mu_t:=\inner{p_t}{c_t}$, the centered move $z_t(i):=\mu_t-c_t(i)$, and the intrinsic state $S_t:=\sum_{s=1}^t z_s$ with $S_0:=0$.
Each move is centered under the learner's play: $\inner{p_t}{z_t}=0$.

\begin{proposition}[Centered-coordinate form of Bayes updating]
\label{prop:centered}
For every $t\ge 1$ and $i\in[K]$,
\begin{equation}\label{eq:centered-state}
S_{t-1}(i)=\sum_{s=1}^{t-1}\mu_s-C_{t-1}(i),
\qquad
p_t(i)
=
\frac{\pi(i)e^{\eta_t S_{t-1}(i)}}{\sum_{j=1}^K \pi(j)e^{\eta_t S_{t-1}(j)}}
\end{equation}
Moreover, for every comparator $\nu\in\DeltaK$,
$ R_T^c(\nu)
:=
\sum_{t=1}^T \bigl(\inner{p_t}{c_t}-\inner{\nu}{c_t}\bigr)
=
\inner{\nu}{S_T}
=
\sum_{t=1}^T \inner{\nu}{z_t}
$.
\end{proposition}

Only relative performance matters.
Replacing $c_t$ by $c_t+b_t\1$ leaves every $z_t$ unchanged: centered coordinates quotient out the global-bias direction from the start, and the game lives on the orthogonal complement of~$\1$.

\subsection{The terminal concentration functional}
\label{sec:terminal}

The game ends with nature gaining whatever the best admissible comparator can extract from the final state.
The one remaining modeling choice is which comparators count as admissible.
We focus on a relative-entropy ball: a comparator may reweight the actions within an information budget $\Gamma$ of the prior.

For $\Gamma\ge 0$ and a state $S\in\R^K$, define
\begin{equation}\label{eq:terminal-support}
L_\Gamma(S)
:=
\sup\left\{\inner{\nu}{S}:\nu\in\DeltaK,\ \KL(\nu\|\pi)\le \Gamma\right\}
\end{equation}

This number can be read in multiple ways.

\begin{enumerate}[nosep,leftmargin=1.6em]
\item \textbf{Benchmark}: $\inner{\nu}{S_T}$ is the cumulative regret against the comparator $\nu$ (Proposition~\ref{prop:centered}), so $L_\Gamma(S_T)$ is the worst regret over the admissible class.
\item \textbf{Support function}: it is how far the relative-entropy ball $\{\nu:\KL(\nu\|\pi)\le\Gamma\}$ extends in the direction $S$.
\item \textbf{Statistical displacement}: it is the largest mean shift attained by a change of measure costing $\Gamma$ in relative entropy, the quantity Chernoff and Sanov bounds compute (Section~\ref{sec:concentration}).
\end{enumerate}
At $\Gamma=0$ no reweighting is permitted and $L_0(S)=\inner{\pi}{S}$, the prior mean; once $\Gamma$ is large enough for the ball to fill the simplex, $L_\Gamma(S)=\max_i S(i)$, the best single action (Proposition~\ref{prop:saturated-reduction}).
In between, $\Gamma$ interpolates between the two.

Fix $\eta>0$.
The dual description of $L_\Gamma$ rests on two quantities built from the prior and the state: the Gibbs tilt $\rho_{\eta,S}$ of $\pi$ toward $S$ at scale $\eta$, and the associated scale-$\eta$ log-partition $W_\eta(S)$.
For $S\in\R^K$ define
\begin{equation}\label{eq:gibbs-map}
[\rho_{\eta,S}]_i
:=
\frac{\pi(i)e^{\eta S(i)}}{\sum_j \pi(j)e^{\eta S(j)}},
\qquad
W_\eta(S)
:=
\frac{1}{\eta}\log\sum_{i=1}^K \pi(i)e^{\eta S(i)}
\end{equation}

\begin{proposition}[Dual form of the terminal concentration functional]
\label{prop:terminal-dual}
For every $S\in\R^K$ and $\Gamma\ge 0$,
\begin{equation}\label{eq:terminal-dual}
L_\Gamma(S)
=
\inf_{\eta>0}
\left\{
\frac{\Gamma}{\eta}+W_\eta(S)
\right\}
\end{equation}
Equivalently, the Gibbs variational identity gives
\begin{equation}\label{eq:gibbs-var}
W_\eta(S)
=
\sup_{\nu\in\DeltaK}
\left\{
\inner{\nu}{S}-\frac{1}{\eta}\KL(\nu\|\pi)
\right\}
\end{equation}
and the optimizer is $q=\rho_{\eta,S}$.
\end{proposition}

The function $W_\eta(S)$ is the scaled log-moment generating function of the score~$S$ under the prior~$\pi$, while $L_\Gamma(S)$ says how much terminal mass a comparator of complexity budget~$\Gamma$ can concentrate in the direction~$S$.
Identity \eqref{eq:terminal-dual} is the Cram\'er--Chernoff mechanism written as an equality. With $\Lambda(\eta):=\eta W_\eta(S)=\log\E_{i\sim\pi}\bigl[e^{\eta S(i)}\bigr]$ the cumulant generating function of the state under the prior and $\Lambda^*$ its Legendre conjugate, \eqref{eq:terminal-dual} reads $L_\Gamma(S)=\inf_{\eta>0}(\Gamma+\Lambda(\eta))/\eta=\sup\{s\in\R:\Lambda^*(s)\le\Gamma\}$, the inverse Cram\'er transform: the largest displacement attainable at rate $\Gamma$.
A classical tail bound performs the same optimization over $\eta$ and keeps an inequality; here the optimized value is equal to the benchmark. Any looseness downstream therefore comes from committing to one scale in advance --- the fixed $\eta$ need not be the minimizer at the realized state --- while the mechanism itself contributes none, a distinction the fixed-scale analysis of Section~\ref{sec:fixed} makes quantitative.

\subsection{Robust Bayes, maximum entropy, and exponential-family structure}
\label{sec:robust-bayes}

The optimizer $\rho_{\eta,S}$ of the Gibbs variational identity has a second characterization, as the maximum-entropy distribution matching the exposed moment---the one feature of a distribution $\nu$ that a linear terminal payoff can see, its mean score $\inner{\nu}{S}$.
That second characterization is the robust-Bayes reading of the terminal functional.

\begin{proposition}[Robust Bayes / maximum entropy]
\label{prop:robust-bayes}
Fix $S\in\R^K$ and $\eta>0$. Then:
\begin{enumerate}[label=(\alph*)]
\item $\rho_{\eta,S}$ uniquely maximizes $\inner{\nu}{S}-\eta^{-1}\KL(\nu\|\pi)$ over $\nu\in\DeltaK$.
\item If $m_\eta:=\inner{\rho_{\eta,S}}{S}$, then $\rho_{\eta,S}$ is the unique minimizer of $\KL(\nu\|\pi)$ subject to the linear moment constraint $\inner{\nu}{S}=m_\eta$.
\end{enumerate}
\end{proposition}

Thus the posterior is simultaneously a best response, a Bellman equalizer, and the maximum-entropy distribution consistent with the currently exposed sufficient statistic.
The exponential family appears because the concentration game is linear in the state, so the Gibbs variational identity \eqref{eq:gibbs-var} applies at every scale, and varying $\eta$ traces out a one-dimensional exponential family in which the intrinsic-time density $Q_t(c)$ is the one-step analogue of the Fisher information (Appendix~\ref{app:exp-family}).

\subsection{The one-scale centered RCGF}

For $p\in\DeltaK$ and a centered move $z\in\R^K$ with $\inner{p}{z}=0$, define
\begin{equation}\label{eq:Q-centered}
Q_\eta(p,z)
:=
\frac{1}{\eta^2}\log\sum_{i=1}^K p(i)e^{\eta z(i)}
\end{equation}
Equivalently, in loss coordinates,
\begin{equation}\label{eq:Q-loss}
Q_t(c)
=
\frac{1}{\eta_t^2}
\psi_{p_t}(-\eta_t;c_t),
\qquad
\psi_p(\lambda;c)
:=
\log\sum_i p(i)e^{\lambda(c(i)-\inner{p}{c})}
\end{equation}
The argument of $Q_t$ names the loss sequence it is evaluated on, while the play and scale are those of the round: $Q_t(c)$ abbreviates $Q_{\eta_t}(p_t,z_t)$, and later sections write $Q_t(g)$ and $Q_t(\ell)$ for the same round-$t$ functional on other score sequences.

Here $\eta^2 Q_\eta(p,z)=\log\E_p[e^{\eta z}]$ is the centered cumulant generating function (CGF) of nature's score $z$ under the learner's mixed action $p$ at scale $\eta$. The quantity carried through the paper is not that CGF but its $\eta^{-2}$ rescaling $Q_\eta$, the normalization under which $Q_\eta\to\tfrac12\Var_p(z)$ as $\eta\downarrow0$; it is called the \emph{rescaled centered cumulant generating function}, abbreviated \emph{RCGF}, or equivalently the \emph{intrinsic-time density}. The distinction is not cosmetic: the two differ by a factor that itself depends on the scale, so a statement about $Q_\eta$ across scales is not a statement about the CGF across scales. ``CGF'' below always means the unrescaled $\eta^2Q_\eta$.
It is the round's primitive information-theoretic quantity in the framework: it measures how big nature's move is, at the learner's chosen scale~$\eta$, in units that make the Bellman update an algebraic identity, and the per-round constraint $Q_\eta(p,z)\le q$ caps how concentrated nature's centered tilt can be at that scale.
It is not a variance surrogate: at small $\eta$ it reduces to half the variance $\tfrac12\Var_p(z)$ (Corollary~\ref{cor:variance-shadow}), at large $\eta$ it captures the worst-tail behavior of $z$ under $p$, and for any~$\eta$ it is the unique non-negative quantity for which $W_\eta(S+z)-W_\eta(S)=\eta Q_\eta(\rho_{\eta,S},z)$ holds (Theorem~\ref{thm:fixed-bellman}).
Bounding $Q_\eta(p,z)\le q$ therefore caps nature's coincidence budget at the chosen scale, and is the only constraint the game requires.

Three readings of the same functional --- a coincidence budget over the moments $\E_p[z^n]$, the constraint half of a robust-Bayes min-max value, and a prior-to-posterior relative-entropy increment --- are developed in Appendix~\ref{app:moment-expansion}.
There the moment expansion also exhibits variance ($n=2$) and range ($n=\infty$) as truncations, which is why they appear later as derived relaxations (Section~\ref{sec:hierarchy}).
The same primitive recurs at every level of the framework: nature's one-round move (here), the comparator's tilt of the prior (Proposition~\ref{prop:terminal-dual}), the bandit estimator's importance-weighted score (Theorem~\ref{thm:bandit}), and the Thompson-sampled posterior's likelihood ratio (Section~\ref{sec:ppr}).

The centered RCGF also admits an integral representation as a weighted average of tilted variances.

\begin{proposition}[Tilted-variance representation]
\label{prop:tilted-var}
Let $p_s^{(\eta,z)}(i)\propto p(i)e^{s\eta z(i)}$ for $s\in[0,1]$.
Then
\begin{equation}\label{eq:tilted-var}
Q_\eta(p,z)
=
\int_0^1 (1-s)\Var_{i\sim p_s^{(\eta,z)}}(z(i))\,ds
\end{equation}
Consequently $Q_\eta(p,z)\ge 0$, with equality iff $z$ is $p$-a.s.\ constant.
\end{proposition}

The tilt family $p_s^{(\eta,z)}$ interpolates from the learner's play $p$ at $s=0$ to the within-round update at $s=1$, and the representation makes the non-negativity of $Q_\eta$ manifest and exhibits both the small-$\eta$ variance limit and the bounded-range bound as immediate corollaries.

\begin{corollary}[Variance and bounded-range relaxations]
\label{cor:variance-shadow}
If $\eta\downarrow 0$ then $Q_\eta(p,z)=\frac12 \Var_p(z)+o(1)$.
If $z(i)\in[a,b]$ for all~$i$, then $Q_\eta(p,z)\le (b-a)^2/8$.
\end{corollary}

The same representation also controls how the primitive responds to a perturbation of the mixed action.
Contaminating the play with a point mass moves the centered RCGF by an influence that is bounded below by a linear term and above by the exponential of the largest centered score: linear in the depth of the score range, exponential in its height (Corollary~\ref{cor:rcgf-influence}, Appendix~\ref{app:further-primitive}).
An estimate of the round's centered RCGF from draws of $p$ therefore has bounded per-draw influence whenever the centered scores are bounded, with the bound set by the endpoints of their range and the scale alone; the number of experts does not enter.
The primitive is estimable at moderate scale in exactly this sense.

\section{The fixed-scale concentration game}
\label{sec:fixed}

The fixed-scale game holds the measurement scale $\eta$ fixed across all rounds. A single log-partition potential then governs the entire run: its one-step increment is exactly the centered RCGF, and the transport identity, the dual regularized form (Appendix~\ref{app:fenchel}), and the variance and range relaxations all read off that one potential. The exact one-round value closes the fixed-scale picture before the scale is allowed to vary in Section~\ref{sec:variable}.

\subsection{Nature's feasible set and the Bellman equalizer}

The fixed-scale game confines nature's per-round move to a one-round information budget on the centered score.
For $p\in\DeltaK$, $q\ge 0$, and $\eta>0$, set
\begin{equation}\label{eq:feasible-moves}
\cZ_{\eta,q}(p)
:=
\left\{
z\in\R^K:\inner{p}{z}=0,\ Q_\eta(p,z)\le q
\right\}
\end{equation}

\begin{game}[Fixed-scale concentration game]
\label{game:fixed}
Game~\ref{game:central} with one datum frozen: the scale is held constant, $\eta_t\equiv\eta$,
confining nature's round-$t$ move to $\cZ_{\eta,q_t}(p_t)$.
Move order, state, and the terminal payoff $L_\Gamma(S_T)$ at time $T+1$ are unchanged.
\end{game}

A single potential runs the game: the log-partition function $W_\eta(S)=\eta^{-1}\log\sum_i\pi(i)e^{\eta S(i)}$, the entropic counterpart of a drifting-game potential --- the potential-function device of boosting-style min-max analyses \cite{luo2014driftinggames}. When the learner plays the Gibbs weights $p=G_\eta(S)=\nabla W_\eta(S)$, one round changes the potential by the amount $W_\eta(S+z)-W_\eta(S)=\eta\,Q_\eta(p,z)$---the round's centered RCGF and nothing more---which nature's budget caps at $\eta q_t$.
The equalizer simplifies the solution: the round's loss is capped by nature's budget whatever the state and whichever feasible direction nature moves, so the rounds do not interact and the potential telescopes from $W_\eta(0)=0$ to the closed form of Theorem~\ref{thm:fixed-bellman}.
A sequential game is otherwise solved by a backward recursion from the horizon; against the equalizer that recursion collapses to a forward sum.
The game is constrained: there is no entropy penalty in the primitive description, and regularization appears only after dualizing the terminal comparator class, which turns the game into the entropic FTRL update on the cumulative scores (Appendix~\ref{app:fenchel}).

The value-to-go defined next is an \emph{admissible Bellman relaxation}: it dominates the game's value from every state, and at the equalizer play its round decrement is nonnegative on the feasible set and vanishes at nature's budget-saturating reply.
Verifying it therefore takes one round's inequality, which every round inherits, where computing the value outright would take the backward recursion (Appendix~\ref{app:bellman-certificate}).

\begin{theorem}[Bellman relaxation and equalizer]
\label{thm:fixed-bellman}
Define
\begin{equation}\label{eq:fixed-bellman}
U_t(S)
:=
\frac{\Gamma}{\eta}+W_\eta(S)+\eta\sum_{s=t}^T q_s,
\qquad
1\le t\le T+1
\end{equation}
Then $U_t$ is an admissible Bellman relaxation of Game~\ref{game:fixed}.
In particular,
$\Val_t^{\eta,\Gamma}(S)\le U_t(S)$,
and the Bellman equalizer is the Gibbs distribution $p=\rho_{\eta,S}=\nabla W_\eta(S)$.
Moreover, for that choice of~$p$ and any feasible~$z$,
\begin{equation}\label{eq:bellman-increment}
W_\eta(S+z)-W_\eta(S)=\eta Q_\eta(p,z)\le \eta q_t
\end{equation}
\end{theorem}

The inequality $\eta Q_\eta(p,z)\le \eta q_t$ in \eqref{eq:bellman-increment} is by hypothesis (nature's per-round budget); the equality $W_\eta(S+z)-W_\eta(S)=\eta Q_\eta(p,z)$ is exact and is the content of the theorem.
At nature's best response $Q_\eta(p,z)=q_t$ the increment inequality is tight, so the Bellman relaxation $U_t$ closes at the Bellman increment.
The relaxation is an equality against the game value only when the committed scale also happens to be the dual optimizer of $L_\Gamma$ at the realized terminal state, which the learner cannot guarantee in advance (Appendix~\ref{app:bellman-certificate}).
In particular, the Bellman bound $\Gamma/\eta+\eta V$ on the cumulative budget regret is strict in general: the exact minimax value can be smaller, since equalizing is weaker than round-value minimaxity and the closed form is a consequence of equalizing.

Nature's best response to the equalizer follows from the same increment. It is the centered move that saturates the budget, $Q_\eta(p,z)=q_t$, and aligns with the worst comparator direction: by the Gibbs variational identity (Propositions~\ref{prop:terminal-dual} and~\ref{prop:robust-bayes}) that comparator is the Gibbs tilt $\rho_{\eta,S}$ of the prior by $\eta S$, and the saturating $z$ is the exponential-family tilt of $p$ toward it. This is the robust-Bayes picture of Section~\ref{sec:robust-bayes} read from nature's side: the worst comparator is a posterior at inverse temperature $\eta$ over the comparator class, and the conditional law the equalizer tracks runs along the same exponential family.

\begin{corollary}[Bellman bound on the budgeted fixed-scale game]
\label{cor:fixed-budget}
If nature is constrained only by a cumulative budget $\sum_{t=1}^T q_t\le V$, then for every $\eta>0$
\begin{equation}\label{eq:fixed-budget-value}
\Val_1^{\eta,\Gamma}(0)\le \frac{\Gamma}{\eta}+\eta V,
\qquad\text{so}\qquad
\inf_{\eta>0}\Val_1^{\eta,\Gamma}(0)\le \inf_{\eta>0}\!\left[\frac{\Gamma}{\eta}+\eta V\right]=2\sqrt{\Gamma V}
\end{equation}
the infimum attained at $\eta=\sqrt{\Gamma/V}$ by AM--GM ($\Gamma/\eta+\eta V\ge 2\sqrt{\Gamma V}$).
At the horizon $T=1$ the same statement reads $\Gamma/\eta+\eta q$, minimized at $\eta=\sqrt{\Gamma/q}$, where it equals $2\sqrt{\Gamma q}$: a sub-Gaussian concentration profile in the budget pair $(\Gamma,q)$.
\end{corollary}

The first inequality is Theorem~\ref{thm:fixed-bellman} summed across rounds; the second line records the standard AM--GM optimization.
The bound $2\sqrt{\Gamma V}$ is the Bellman upper envelope; the exact minimax value can be strictly smaller (cf.\ the remark following Theorem~\ref{thm:fixed-bellman}), by an amount computed in closed form for the two-expert instance (Proposition~\ref{prop:k2-exact}, Appendix~\ref{app:one-round-exact}). That slack is a short-horizon effect: under the cumulative budget the two sides meet from a finite horizon on, so \eqref{eq:fixed-budget-value} holds with equality at constant scale (Proposition~\ref{prop:horizon-free}).

\subsection{One-step transport identity}

The fixed-scale value bound of Theorem~\ref{thm:fixed-bellman} came from telescoping the potential; tracking the same telescoping against a fixed comparator~$\nu$ instead gives a per-round accounting.
It is the constant-scale case of \eqref{eq:intro-exact}, in which the retempering drift vanishes and only the intrinsic-time loss and the terminal transport remain; the general form, with the scale free to move, is Theorem~\ref{thm:variable}.

\begin{proposition}[One-step transport identity]
\label{prop:fixed-transport}
Fix $\eta>0$, $p\in\DeltaK$, and centered~$z$.
Define $p^+(i):=p(i)e^{\eta z(i)}/\sum_j p(j)e^{\eta z(j)}$.
Then for every comparator $\nu\in\DeltaK$,
\begin{equation}\label{eq:fixed-transport}
\inner{\nu}{z}
=
\eta Q_\eta(p,z)
+
\frac{\KL(\nu\|p)-\KL(\nu\|p^+)}{\eta}
\end{equation}
Summing over a fixed-scale trajectory gives
\begin{equation}\label{eq:fixed-cumulative}
\inner{\nu}{S_T}
=
\eta\sum_{t=1}^T Q_\eta(p_t,z_t)
+
\frac{\KL(\nu\|\pi)-\KL(\nu\|p_{T+1})}{\eta}
\end{equation}
\end{proposition}

This is the fixed-scale information balance.
The immediate loss is the one-scale centered RCGF~$\eta Q_\eta$; the remainder is transported as terminal relative entropy.
Read \eqref{eq:fixed-transport} from the left: the comparator's gain on the round is the learner's information loss, plus the ground the comparator gains on the learner's belief, measured as the drop in relative entropy from the comparator to it.

\subsection{Variance and range as upper bounds on the centered RCGF}
\label{sec:hierarchy}

The centered RCGF has the tilted-variance representation of Proposition~\ref{prop:tilted-var}, and its two classical proxies are read off from it directly. Retained as pointwise upper bounds they are
\begin{equation}\label{eq:hierarchy}
\begin{aligned}
Q_{\eta_t}(p_t,c) &\;\le\; (e-2)\,\Var_{p_t}(c)
 \qquad\text{(for $\eta_t\le 1$ and $c\in[0,1]^K$),}\\
Q_{\eta_t}(p_t,c) &\;\le\; \range(c)^2/8
\end{aligned}
\end{equation}
the first the non-asymptotic entropy-version second-order bound \cite{cesa2006prediction} (its sharper $\eta_t\downarrow0$ form is the variance-limit equality $Q_{\eta_t}\to\tfrac12\Var$ of Corollary~\ref{cor:variance-shadow}), the second Hoeffding's lemma.
Replacing the exact $Q_t$ in the regret decomposition \eqref{eq:exact-variable} by either proxy gives a looser upper bound on regret, recovering the classical second-order and Hoeffding-style envelopes. Each trades information about $c$'s tail shape under $p_t$ for an easier-to-evaluate quantity, and the inequality is strict whenever $c$ is non-Gaussian under $p_t$.
The two proxies do not refine each other: as constraint sets on nature, $\{(e-2)\Var_{p_t}(c)\le\beta\}$ and $\{\range(c)^2\le 8\beta\}$ both sit inside $\{Q_{\eta_t}(p_t,c)\le\beta\}$, but neither contains the other, and which is smaller depends on how concentrated $p$ is (Appendix~\ref{app:proxy-order}).
Retaining the remainder turns the range relaxation into an identity: at a fixed rate and bounded range the classical Hedge bound is the range-relaxed game plus an explicit nonnegative slack, the terminal posterior mismatch together with the cumulative gap between the Hoeffding proxy and the true RCGF increment (Appendix~\ref{app:proxy-order}).

The learner chooses $\eta_t$ \emph{before} seeing $c_t$ and observes only $Q_\eta(p_t,z_t)$ at that single scale, leaving the rest of the profile $\eta\mapsto Q_\eta(p_t,z_t)$ unobserved.
Different inverse temperatures reveal different structure --- variance-like at small $\eta$, tail-like at large $\eta$ --- so intrinsic time is algorithm-dependent, and the RCGF constraint set \eqref{eq:local-cgf} moves with $\eta_t$ where the range constraint is scale-free.
The tradeoff the schedule must equalize is visible in one instance: for a binary centered move on $\{-1,+1\}$ under $p_t=(\tfrac12,\tfrac12)$, $Q_t(\eta)=\log\cosh(\eta)/\eta^2$ falls from $\tfrac12$ at $\eta\to0^+$ to $0$ as $\eta\to\infty$ even as the round's intrinsic-time loss $\eta Q_t=\log\cosh(\eta)/\eta$ rises.

\subsection{The exact one-round value}
\label{sec:fixed-k2-exact}

The Bellman certificate of Corollary~\ref{cor:fixed-budget} is an upper bound, and it is strict already at a single round.
The one-round value is known exactly on two slices --- two experts at every comparator budget (Proposition~\ref{prop:k2-exact}), and every $K$ at a comparator budget large enough that the relative-entropy ball is the whole simplex (Proposition~\ref{prop:saturated-exact}) --- and on both slices it sits strictly below the Bellman envelope $\Gamma/\eta+\eta q$.
The Gibbs play equalizes at every state but need not be minimax: under a non-uniform prior the two part already in the first round, and from a skewed state the exact minimax play over-tilts relative to the Gibbs iterate, so no elementary closed form survives past $T=1$.
Appendix~\ref{app:one-round-exact} develops the one-round game in full: the nature-side certificate, the two exact values, a pointwise alignment identity confining the Bellman gap to a single bounded terminal functional, and the exact two-round computation.

\section{Changing scale: retempering, allocation, and active constraints}
\label{sec:variable}

The fixed-scale game is clean because one Bellman potential governs all rounds.
Variable scale means switching Bellman potentials.
The loss from switching them is the drift term.

\subsection{From per-round caps to a cumulative budget}

The scale changes roles here: data of Game~\ref{game:central}, it becomes the learner's own move, selected round by round with the play pinned to the Gibbs parameterization of \eqref{eq:bayes-update}, so the scale is the learner's only remaining choice. Alongside that promotion, nature's difficulty is capped once over the whole run instead of round by round, which lets it choose where to concentrate.

{\mdfsetup{nobreak=true}%
\begin{game}[The self-tuned concentration game $\cG^{\mathrm{st}}(\Gamma,V)$]
\label{game:cgf}
Game~\ref{game:central} with two parts changed.
Learner's move: the learner also selects the scale $\eta_t>0$ each round (predictably, from $c_1,\dots,c_{t-1}$),
its play pinned to the Gibbs parameterization $p_t(i)\propto\pi(i)e^{-\eta_t C_{t-1}(i)}$ of \eqref{eq:bayes-update}.
Nature's constraint: the per-round caps $q_t$ are replaced by a single budget over the run,
\begin{align}
\sum_{t=1}^T Q_t(c) &\le V
 \qquad\text{(intrinsic-time budget)}, \label{eq:it-budget}\\
\KL(\nu\|\pi) &\le \Gamma
 \qquad\text{(comparator budget, for $\nu\in\DeltaK$)}, \label{eq:comp-budget}
\end{align}
with $\pi$ of full support and payoff $R_T^c(\nu)=\sum_{t=1}^T\inner{p_t-\nu}{c_t}$ to nature.
The payoff rule is unchanged: $R_T^c(\nu)=\inner{\nu}{S_T}$ (Proposition~\ref{prop:centered}), so its supremum over the comparator budget \eqref{eq:comp-budget} is the terminal payoff $L_\Gamma(S_T)$.
The game's minimax value is
\begin{equation}\label{eq:minimax-value}
\Val\bigl(\cG^{\mathrm{st}}(\Gamma,V)\bigr)
=
\inf_{(\eta_t)\text{ predictable}}
\;\sup_{\substack{\nu:\,\KL(\nu\|\pi)\le\Gamma \\
 c_{1:T}:\,\sum_t Q_t \le V}}
R_T^c(\nu)
\end{equation}
\end{game}}

Written round-by-round, Game~\ref{game:cgf} is the alternating min-max $\inf_{\eta_1,p_1}\sup_{c_1}\cdots\inf_{\eta_T,p_T}\sup_{c_T}\sup_{\nu}\sum_t\inner{p_t-\nu}{c_t}$ over predictable $(\eta_t,p_t)$, per-round moves constrained by an allocation $(q_t)$ summing to at most $V$, and a terminal comparator-selection step inside the budget $\Gamma$.

The fixed-scale analysis already bounds this chain with no horizon dependence. At any constant scale $\eta_t\equiv\eta$ the drift vanishes; at the Gibbs response the round's potential increment is exactly $\eta_t Q_t$, capped at $\eta q_t$ by the budget (Theorem~\ref{thm:fixed-bellman}); and the terminal transport costs at most $\Gamma/\eta$. Hence
\begin{equation}\label{eq:horizon-free-bound}
\Val\bigl(\cG^{\mathrm{st}}(\Gamma,V)\bigr)\;\le\;\inf_{\eta>0}\Bigl\{\frac{\Gamma}{\eta}+\eta V\Bigr\}=2\sqrt{\Gamma V}
\end{equation}
a bound depending only on the budget pair $(\Gamma,V)$: doubling $T$ at fixed $V$ leaves it unchanged. The global constraint $\sum_t Q_t\le V$ lets nature allocate difficulty unevenly across rounds; a per-round cap $Q_t\le\beta$ is the case $V=T\beta$, and the per-round maximum $Q_T^*(c)=\max_t Q_t(c)$ appears in Theorem~\ref{thm:sqrt-envelope} as an edge term for budget concentrated in one round.

The value is horizon-free in the same sense from a finite horizon on, so fixing $T$ alongside the budgets constrains nature less than it appears to.

\begin{proposition}[Horizon-freeness of the constant-scale value]
\label{prop:horizon-free}
Fix $\eta>0$, $\Gamma\in(0,\Gamma_{\max})$ with $\Gamma_{\max}:=\max_i\log(1/\pi(i))$ the budget at which the relative-entropy ball fills the simplex, and $V>0$, and write $\mathrm{Val}_T^\eta(\Gamma,V)$ for the value of Game~\ref{game:cgf} from $S_0=0$ at the constant scale $\eta_t\equiv\eta$.
Then there is a finite $T_0(\eta,V,\Gamma)$ with
\begin{equation}\label{eq:horizon-free-value}
\mathrm{Val}_T^\eta(\Gamma,V)=\frac{\Gamma}{\eta}+\eta V
\qquad\text{for every }T\ge T_0
\end{equation}
\end{proposition}

The mechanism is that nature must do two things at once: exhaust the budget, and finish on the sphere where the Gibbs posterior meets the worst comparator.
One round couples them rigidly and the value falls short by the mismatch; a second round decouples them, and from there on both can be met together.
Appendix~\ref{app:further-variable} carries the chain characterization this rests on, together with the terminal alignment functional $A_\Gamma(p):=\sup_{\nu:\,\KL(\nu\|\pi)\le\Gamma}\bigl[\KL(\nu\|\pi)-\KL(\nu\|p)\bigr]$ that it optimizes.
It also carries the closed form for the two-step gate, the numerical margin by which a varying scale improves on a constant one, and the unbounded-adaptation example that forces the clip at $\eta_t=1$ in the schedule below.

Past $T_0$ the envelope $2\sqrt{\Gamma V}$ is the constant-scale value exactly, with $T_0=2$ in every instance computed at $K\le5$, uniform and skewed priors alike, and $T_0=3$ only within a few percent of saturation.
The strictness at $T=1$ is thus a one-round phenomenon and does not persist.
What remains open at $T\ge2$ is the per-round-capped game, whose reachable set still grows with the horizon through $V=T\beta$ (Section~\ref{sec:discussion}).

The constant-scale argument supplies no guarantee for a learner that must set its scale online, without using $V$; the decomposition proved next quantifies that plug-in.

\subsection{The exact variable-temperature decomposition}

Let $\rho_{t,\eta}(i)\propto \pi(i)e^{-\eta C_t(i)}$.
Define
\begin{equation}\label{eq:A-potential}
A_t(\eta):=
-\frac{1}{\eta}\log\sum_i \pi(i)e^{-\eta C_t(i)},
\qquad
D_T:=\sum_{t=1}^{T-1}\bigl(A_t(\eta_t)-A_t(\eta_{t+1})\bigr)
\end{equation}
Equivalently, in centered coordinates, $D_T = \sum_{t=1}^{T-1}\bigl(W_{\eta_{t+1}}(S_t)-W_{\eta_t}(S_t)\bigr)$.

The drift $D_T$ is the algebraic loss from switching potentials between rounds.
Telescoping the per-round transport identity along the trajectory and grouping the switching points produces three independent terms: a centered RCGF per round (from each measurement-scale-fixed step), a drift (from round-to-round retempering), and a terminal relative-entropy transport (from comparing the comparator with the final posterior).
Each term is a clean Bregman / relative-entropy gap; no quadratic-variation surrogate is needed.
The decomposition and its sign rule are stated in~\cite{balsubramani2026adaptive}; the game coordinates here, and the drift's continuous-time limit (Proposition~\ref{prop:drift-continuous}), are the present development.
The sign rule $D_T\le 0$ for non-increasing schedules follows from $A_t(\eta)$ being nonincreasing in $\eta$ at fixed $C_t$.
This is visible from the variational form $A_t(\eta)=\min_{\nu\in\DeltaK}\{\inner{\nu}{C_t}+\eta^{-1}\KL(\nu\|\pi)\}$: increasing $\eta$ relaxes the relative-entropy penalty, so the minimum value cannot increase.

\begin{theorem}[Exact variable-temperature decomposition]
\label{thm:variable}
For the prior-retempered update \eqref{eq:bayes-update}, define
\begin{equation}\label{eq:Qt-def}
Q_t(c)
=
\frac{1}{\eta_t^2}\psi_{p_t}(-\eta_t;c_t),
\qquad
V_T(c):=\sum_{t=1}^T Q_t(c),
\qquad
B_T(\nu):=\frac{\KL(\nu\|\pi)-\KL(\nu\|\rho_{T,\eta_T})}{\eta_T}
\end{equation}
Then for every comparator $\nu\in\DeltaK$,
\begin{equation}\label{eq:exact-variable}
R_T^c(\nu)
=
\sum_{t=1}^T \eta_t Q_t(c)
+
D_T
+
B_T(\nu)
\end{equation}
If $(\eta_t)$ is nonincreasing, then $D_T\le 0$.
\end{theorem}

This identity replaces the backward induction for the Gibbs family: the realized regret factors, pathwise and for every predictable schedule, into three terms that are then bounded separately, with no recursive Bellman equation to solve.

At fixed temperature the sequential structure collapses to a single log-partition: $D_T=0$, and the cumulative \emph{mix loss} $A_T(\eta_T)$---the cumulative loss of the Bayesian/Hedge mixture forecaster \cite{derooij2014adahedge}---is the terminal log-partition function, the negative log Bayes marginal likelihood of the data under prior $\pi$ and inverse temperature $\eta$.

The identity \eqref{eq:exact-variable} reads two ways.
The \emph{primal} reading is a path-information accounting, $\sum_t\eta_t Q_t$, $D_T$, and $B_T(\nu)$ the predictable book-keeping quantities incurred along the realized trajectory (Figure~\ref{fig:game-trajectory}).
The \emph{dual} reading evaluates the same quantities inside the Bellman relaxation of Game~\ref{game:cgf}, where at each round the Gibbs equalizer and nature's budget-saturating move form a saddle point of the relaxed one-step value worth $\eta_t q_t$ (Theorem~\ref{thm:fixed-bellman}), with $W_\eta$ acting as both primal potential and dual value function.
Section~\ref{sec:pressure} takes the symmetric dual from the other side; the retempering drift's continuous-time It\^o form subsumes the two as the traversal orders of one exact potential (Appendix~\ref{app:continuous-time}).

\subsection{The second-order equalizer and nature's allocation}

The constant-scale slice fixes the benchmark and requires no assumption on the schedule: each constant $\eta$ guarantees $R_T^c(\nu)\le\Gamma/\eta+\eta V$, and the oracle scale $\eta^\star(V)=\sqrt{\Gamma/V}$ equalizes the two costs at the AM--GM value $2\sqrt{\Gamma V}$ of \eqref{eq:horizon-free-bound}.
The oracle scale consumes one input the online learner does not have, the total intrinsic time.
The \emph{second-order equalizer} is the plug-in rule that removes that input, substituting the running total $V_{t-1}(c)$---the part of the intrinsic time already revealed---for $V$ in the oracle scale: fix $\Gamma>0$ and $C>0$ and let
\begin{equation}\label{eq:schedule}
\eta_t
=
\begin{cases}
1, & V_{t-1}(c)=0,\\[0.4em]
\min\left\{1,\ C\sqrt{\Gamma/V_{t-1}(c)}\right\}, & V_{t-1}(c)>0.
\end{cases}
\end{equation}
The clip at $1$ handles the cold start, and the schedule is computable from realized data with no knowledge of the horizon or of $V$.

Neither of the two properties of \eqref{eq:schedule} is assumed; both are derived.
First, $V_{t-1}(c)$ is nondecreasing in $t$ and $v\mapsto\min\{1,C\sqrt{\Gamma/v}\}$ is nonincreasing, so the scale is nonincreasing along every loss sequence---in temperature terms, the schedule warms monotonically.
Second, the sign rule of Theorem~\ref{thm:variable} then makes the retempering drift a gain to the learner, $D_T\le 0$, so the regret is bounded by the intrinsic-time loss plus the transport.
The schedule earns its name by pinning the ratio of marginal loss to marginal transport at $2C^2$ along the whole path, equal round by round at $C=1/\sqrt2$ (Appendix~\ref{app:riemann-loss}).

The intrinsic-time loss is then a Riemann sum of a falling scale, evaluated at each interval's left endpoint, so it sits above the corresponding integral and exceeds it by at most the largest single increment $Q_T^*(c)$.
That bracketing is the two-sided envelope below, whose two sides differ by the discretization cost of sum against integral.
$V_T(c)$ is the cumulative centered RCGF along the realized trajectory, which is a different quantity from a sum of variances. The small-$\eta$ approximation $V_T\approx\tfrac12\sum_t \Var_{p_t}(c_t)$ recovers classical second-order rates. Away from that approximation, when nature plays heavy-tailed centered moves, the variance-relaxed game is loose while the exact RCGF schedule remains tight.

\begin{theorem}[Two-sided square-root envelope]
\label{thm:sqrt-envelope}
Let $Q_T^*(c):=\max_{1\le t\le T}Q_t(c)$.
Under \eqref{eq:schedule},
\begin{equation}\label{eq:two-sided-envelope}
2C\sqrt{\Gamma V_T(c)}-C^2\Gamma
\;\le\;
\sum_{t=1}^T \eta_t Q_t(c)
\;\le\;
Q_T^*(c)+2C\sqrt{\Gamma V_T(c)}
\end{equation}
Consequently, if $\KL(\nu\|\pi)\le \Gamma$,
\begin{equation}\label{eq:pathwise-regret}
R_T^c(\nu)
\le
\Gamma
+
Q_T^*(c)
+
(2C+C^{-1})\sqrt{\Gamma V_T(c)}
\end{equation}
At $C=1/\sqrt{2}$ the leading coefficient is $2C+C^{-1}=2\sqrt2$: every feasible sequence of Game~\ref{game:cgf} incurs
\[
R_T^c(\nu)
\le
2\sqrt{2}\,\sqrt{\Gamma V_T(c)}+O\bigl(\Gamma+Q_T^*(c)\bigr).
\]
\end{theorem}

The unit coefficient on $Q_T^*(c)$ meets the necessity floor of Lemma~\ref{rem:minimax} with equality; the same envelope, with an additional $C^2\Gamma$ slack on the upper side, is established in~\cite{balsubramani2026adaptive}.
This envelope is the bound form of the exact identity \eqref{eq:exact-variable}: the left and right sides sandwich the cumulative loss $\sum_t\eta_t Q_t(c)$ that the identity isolates, and the upper side, added to $D_T\le 0$ and the terminal transport $B_T(\nu)\le\Gamma/\eta_T$, yields \eqref{eq:pathwise-regret}.
The gap between the two sides---the edge term $Q_T^*(c)$ and the $C^2\Gamma$ initialization overhead---is the discretization cost identified above.
Against the known-$V$ benchmark \eqref{eq:horizon-free-bound}, the constant oracle scale attains $2\sqrt{\Gamma V}$ using $V$ as an input; the plug-in schedule attains the same square-root rate \emph{pathwise}, at the realized $V_T(c)\le V$ and without $V$ as an input, up to the edge terms.
A regularized variant $\eta_t=C\sqrt{\Gamma/(V_{t-1}(c)+\Gamma)}$ replaces the clipping at $\eta_t=1$ by the offset $\Gamma$ and satisfies the same envelope with $V_T(c)$ replaced by $V_T(c)+\Gamma$ (Appendix~\ref{app:proofs-variable}).
It is the form used in the self-play ledger of Section~\ref{sec:games}.

Nature's side of the schedule is the allocation problem $\max\{\sum_t\eta_t Q_t: Q_t\ge0,\ \sum_t Q_t\le V\}$.
Against a pre-committed nonincreasing sequence this is a linear program solved by front-loading, filling the largest-scale early rounds first.
Against the self-tuned schedule \eqref{eq:schedule} the Riemann-sum bracketing makes the loss allocation-insensitive up to the edge terms --- the equalizer property in time --- so nature's remaining leverage sits exactly at the edges, placing its mass inside the clipped cold start or concentrating in a single round.
Both attacks are quantified by the one-round witnesses below.

\begin{lemma}[Edge-term necessity]
\label{rem:minimax}
Both edge terms in Theorem~\ref{thm:sqrt-envelope} are forced by discrete effects, each by a one-round witness: the initialization constant $C^2\Gamma$ is attained exactly, and the single-jump witness shows the $Q_T^\star$ coefficient cannot fall below~$1$.
\end{lemma}

\subsection{The active-constraint game and pressure targets}
\label{sec:pressure}

A dual sequential game reverses the round order and drops nature's cap:
nature moves first, unconstrained,
and the learner answers by choosing the scale that makes the one-step constraint active --- holding with equality --- at a pressure target of its choosing.

The target is the round's loss itself.
For a centered move that is not $p$-a.s.\ constant and any target $b\in(0,\max_{i\in\supp(p)}z(i))$, there is a unique scale $\eta>0$ solving $\log\sum_i p(i)e^{\eta z(i)}=\eta b$ (Proposition~\ref{prop:active-scale}); the left side is $\eta^2Q_\eta(p,z)$ by the definition \eqref{eq:Q-centered}, so the equation asks for $b=\eta\,Q_\eta(p,z)$ and the ratio $b/\eta$ is the round's centered RCGF exactly.
The learner names the loss it will take this round, and the equation returns the scale that delivers it.

Game~\ref{game:pressure} is Game~\ref{game:central} with three parts changed: nature reveals its centered move first, no cap is imposed, and the learner answers with the target-induced scale, its update pinned to the Gibbs iterate $p_{t+1}(i)\propto p_t(i)e^{\eta_t z_t(i)}$.
The terminal payoff is unchanged, the comparator-worst loss along any trajectory again $L_\Gamma(S_T)$ (Proposition~\ref{prop:terminal-collapse}).

\begin{theorem}[Exact active-constraint identity]
\label{thm:pressure-identity}
Under Game~\ref{game:pressure}, for every comparator $\nu\in\DeltaK$,
\begin{equation}\label{eq:pressure-identity}
\sum_{t=1}^T \eta_t\bigl(\inner{\nu}{z_t}-b_t\bigr)
=
\KL(\nu\|p_1)-\KL(\nu\|p_{T+1})
\end{equation}
Equivalently, with $\kappa_t:=b_t/\eta_t$,
\begin{equation}\label{eq:pressure-quadratic}
\sum_{t=1}^T \eta_t\inner{\nu}{z_t}
=
\sum_{t=1}^T \kappa_t\eta_t^2
+
\KL(\nu\|p_1)-\KL(\nu\|p_{T+1})
\end{equation}
Moreover, $\kappa_t=\frac{1}{2}\Var_{p_t}(z_t)+O(\eta_t)$ as $\eta_t\to 0$, so the usual quadratic-variation penalties are the small-$\eta$ limit of the exact pressure-target line search.
\end{theorem}

The small-scale expansion of $\kappa_t$ is Corollary~\ref{cor:variance-shadow} read in these coordinates, since $\kappa_t$ is the round's centered RCGF.

Two dual games therefore share the one-step transport identity. 
In \emph{Game~A}, the retempered second-order game, the learner commits to $\eta_t$ first and benefits from warming ($D_T\le0$ for nonincreasing $\eta_t$).
In \emph{Game~B}, the pressure game, nature reveals its move first and the learner solves a local line search, favored by cooling.
The pressure target $b_t$ is the operational variable and the scale $\eta_t$ the dual variable that opens up to satisfy it, which puts boosting (pressure $=$ per-round margin) and concentration (pressure $=$ per-round RCGF) on the same footing.
The two readings exchange which quantity is primal and which the Lagrange multiplier, leaving the equilibrium unchanged.

\begin{proposition}[Terminal collapse of the variable-temperature ledger]
\label{prop:terminal-collapse}
Run the variable-temperature game (Game~\ref{game:central}) with any predictable schedule $(\eta_t)_{t=1}^T$ and let $S_T=\sum_{t=1}^T z_t$ be the realized terminal score, $\rho_{T,\eta_T}(i)\propto\pi(i)e^{\eta_T S_T(i)}$ the terminal posterior.
With $\sum_t\eta_t Q_t$, $D_T$, and $B_T(\nu)$ as in \eqref{eq:exact-variable},
\begin{equation}\label{eq:loss-drift-telescope}
\sum_{t=1}^T \eta_t Q_t + D_T = W_{\eta_T}(S_T)
\end{equation}
and consequently the comparator-worst regret depends on the realized state alone,
\begin{equation}\label{eq:regret-collapse}
\sup_{\nu:\,\KL(\nu\|\pi)\le\Gamma} R_T^c(\nu) = L_\Gamma(S_T)
\end{equation}
\end{proposition}

The three horizon terms are therefore not independent: against the worst comparator they telescope onto a function of the realized state alone, so $D_T$ is pinned by the state and the losses, and the schedule enters the worst-case regret only through which terminal state it lets nature reach.
The same collapse holds for the pressure game, so Games~A and~B give the same comparator-worst regret along any common score trajectory and the horizon comparison $\Val(\mathrm{B})\le\Val(\mathrm{A})$ reduces to reachability.

\section{Comparator classes, concentration, and luckiness}
\label{sec:concentration}

Changing the comparator class, the sufficient statistic the learner scores, or the terminal event turns the same game into quantile and tracking regret, side information and optimism, event-restricted concentration (Gibbs conditioning, Sanov, Chernoff, Cram\'er), stochastic luckiness (guarantees that sharpen on easy loss sequences), and confidence sets. These are all terminal slices packaged by the terminal payoff $L_\Gamma$.

\subsection{Changing comparator classes}

Quantile regret is the simplest comparator change.
If $\pi$ is uniform and $\nu_A$ is uniform on a set $A\subseteq[K]$, then $\KL(\nu_A\|\pi)=\log(K/|A|)$.
Every PAC-Bayes/concentration statement immediately yields the corresponding fixed-budget quantile statement by choosing~$A$ as a hindsight set of good experts.

A simple dyadic meta-controller makes the guarantee simultaneous over all quantiles: one worker copy of the second-order game per budget $\Gamma_j=2^j$, aggregated by a logarithmic-budget controller playing the mixture of worker predictions. This adds a controller term of $O(\log\log\log K)$ plus a controller intrinsic-time term $O(\sqrt{\Gamma^{\mathrm{ctl}}V_T^{\mathrm{ctl}}})$ in the controller's own budget and intrinsic time.
The dyadic grid reaching the saturated budget has $O(\log\log K)$ workers, and the additive cost is the comparator complexity of a uniform prior over them, under the envelope of Theorem~\ref{thm:sqrt-envelope} run at the controller level.

A dynamic comparator path $\nu_1,\dots,\nu_T$ is handled by the same calculus:
\begin{equation}\label{eq:dynamic-comparator}
R_T^{c,\mathrm{dyn}}(\nu_{1:T})
=
D_T+B_T(\nu_T)+\sum_{t=1}^T \eta_t Q_t(c)
+
\sum_{t=1}^{T-1}\inner{\nu_{t+1}-\nu_t}{C_t}
\end{equation}
Only the last term is new.
It is exact before being bounded, so switching/tracking regret enters as a comparator-path transport term, controlled for losses in $[0,1]$ by the path's total variation, $|\sum_{t=1}^{T-1}\inner{\nu_{t+1}-\nu_t}{C_t}|\le 2\sum_{t=1}^{T-1}t\,\mathrm{TV}(\nu_{t+1},\nu_t)$ under the convention $\mathrm{TV}(p,q)=\tfrac12\|p-q\|_1$.

\subsection{The sufficient statistic as a design choice}
\label{sec:sufficiency}
The concentration game is defined relative to whatever sufficient statistic the learner scores, with no ambient loss vector behind it; varying that choice recovers side information, optimism, specialist experts, sparsity, and bounded-influence robustness as special cases of one game.
If the learner applies Bayes to raw losses, optimistic residuals, importance-weighted estimates, signed margins, one-sided shortfalls, or bounded-influence transforms, then the resulting state, intrinsic time, and move constraint are genuinely different.

Optimism scores a residual against a predictable baseline. Suppose the original losses are~$\ell_t$ and a predictable baseline~$m_t$ is available.
Setting $c_t=\ell_t-m_t$ yields
\begin{equation}\label{eq:optimistic}
\sum_{t=1}^T \inner{p_t-\nu}{\ell_t}
=
D_T+B_T(\nu)+\sum_{t=1}^T \eta_tQ_t(\ell-m)
+
\sum_{t=1}^T \inner{p_t-\nu}{m_t}
\end{equation}
Only the residual sequence $\ell_t-m_t$ contributes to the intrinsic-time loss; the final term records the predictable mismatch.
Slow temporal drift, optimistic game play, model-based predictions, and martingale compensators are all instances.

Six further choices follow the same template, each entering the entropic FTRL update as the cumulative statistic $C_{t-1}(i)$ of a per-round score, and each with its own consequence: a confidence-weighted score $\sum_{s<t}\beta_s(i)\ell_s(i)$ gives soft specialists, the framework for confidence-rated experts \cite{chaudhuri-freund-hsu} and for the side-information device of some adaptive analyses.
A predictable offset gives variance corrections; the excess-loss positive part $e_t^+(i):=(\inner{p_t}{c_t}-c_t(i))_+$ seeks sparsity, with $R_T^c(\nu)\le\sum_t\inner{\nu}{e_t^+}$ and $\range(e_t^+)\le\range(e_t)$; and a Winsorization or Catoni transform \cite{catoni2007pacbayesian} caps the loss on heavy tails.
In each case $p_t$ is the max-entropy distribution given the chosen statistic, $Q_t$ is the RCGF of the new observation in the corresponding exponential family, and a mismatch term records exactly how the game on the raw losses differs from the reduced game on the transformed ones.
Heavy-tailed scores need the bounded-influence transforms.
A raw move with unbounded coordinates can send the one-round RCGF to $+\infty$, and running the same game on a truncated score restores a finite round while leaving every identity intact.
The truncation is then carried as an exact additive bias in the comparator's gain.
Table~\ref{tab:sufficient-statistics} collects them; specialist aggregation also scales, combining an ensemble's predictions on unlabeled data through a minimax game of the same kind \cite{balsubramanifreund15bspecialistminimax,BF16}.
The framework absorbs each as a sufficient-statistic choice, and no separate algorithm is needed.

\subsection{Event-restricted terminal games, Gibbs conditioning, and Sanov}
\label{sec:event-games}

The terminal support problem may restrict comparators to an event $A\subseteq[K]$, where the standard form uses a relative-entropy ball.

\begin{proposition}[Event-restricted terminal game]
\label{prop:event-restricted}
Fix an event $A\subseteq[K]$ and $\eta>0$.
Then
\begin{equation}\label{eq:event-restricted}
\sup_{\nu\in\DeltaK,\ \supp(\nu)\subseteq A}
\left\{
\inner{\nu}{S}-\frac{1}{\eta}\KL(\nu\|\pi)
\right\}
=
\frac{1}{\eta}\log\sum_{i\in A}\pi(i)e^{\eta S(i)}
\end{equation}
Equivalently,
\begin{equation}\label{eq:gibbs-conditioning}
\frac{1}{\eta}\log\sum_{i\in A}\pi(i)e^{\eta S(i)}
=
\frac{\log\pi(A)}{\eta}
+
\frac{1}{\eta}\log
\E_{i\sim \pi(\cdot\mid A)}e^{\eta S(i)}
\end{equation}
The optimizer is the Gibbs tilt of the conditioned prior $\pi(\cdot\mid A)$.
\end{proposition}

This is the discrete Gibbs-conditioning principle in game form --- conditioned on an event, inference proceeds from the prior conditioned on that event: the event cost is $-\log\pi(A)$, and conditional concentration inside the event is handled by the same log-partition as before.

The same comparator-class geometry underlies the classical large-deviation bounds.
On the simplex of types (empirical distributions) $\Delta([K])$, a closed set $E$ has empirical-measure exponent $\inf_{\nu\in E}\KL(\nu\|p)$, the Lagrange dual of the same comparator-class variational problem that $L_\Gamma$ solves (Proposition~\ref{prop:terminal-dual}); when $E$ is a face $\{\nu:\supp(\nu)\subseteq A\}$ the dual has the closed form of Proposition~\ref{prop:event-restricted}.
Chernoff, Cram\'er, Gibbs conditioning, and Sanov are then each a comparator-class concentration game, sharing one derivation that differs only in where it is entered (Table~\ref{tab:event-restricted}; for the standard treatment see \cite{dembo2010large}).
Fixing the triple $(\pi,S,A)$, the comparator-class dual supplies the variational value, the closed form of Proposition~\ref{prop:event-restricted} evaluates it on a face, and optimizing the scale converts it into the rate function --- one argument, proved once in the two propositions, and the whole of the game's contribution to every row.
Only the closing step from rate function to probability statement differs across rows, and each is classical: exponential Markov after tilting for Chernoff and Cram\'er, the conditional-limit theorem for the I-projection, and $n$-scaling with closure and interior conditions for Sanov (Appendix~\ref{app:four-classical}).

\begin{table}[t]
\centering
\small
\setlength{\tabcolsep}{4pt}
\renewcommand{\arraystretch}{1.25}
\begin{tabular}{@{}p{2.15cm}p{2.75cm}p{5.05cm}p{4.55cm}@{}}
\toprule
\textbf{Classical bound} & \textbf{Triple $(\pi,S,A)$} & \textbf{Game quantity} & \textbf{Classical step, taken as given} \\
\midrule
Chernoff, iid sum $\bar X_n$ & $\pi=P$, $S=nc$, $A=\{i:c(i)\ge t\}$ & variational tilt $\eta^{-1}\log\sum_i P(i)e^{\eta n c(i)}$ $=W_\eta(nc)$ (Prop.~\ref{prop:terminal-dual}) & Markov / exponential tilting $\Rightarrow\Pr(\bar X_n\ge t)\le e^{-nI(t)}$ \\
Cram\'er large deviations & same $(\pi,S,A)$ & Lagrange dual $\inf_{\nu:\,\E_\nu X\ge t}\KL(\nu\|P)=I(t)$ (Prop.~\ref{prop:terminal-dual}) & tilted-law representation, exponential Markov \\
I-projection / conditional limit & Gibbs conditioning on $\supp\subseteq A$ & $\eta\downarrow0$ optimizer $=$ I-projection of $\pi$ onto $A$ (Prop.~\ref{prop:event-restricted}) & conditional-limit theorem \\
Finite-state Sanov & types $\Delta([K])$, $\pi=P$, $E$ closed & Lagrange dual $\inf_{\nu\in E}\KL(\nu\|P)$ (Prop.~\ref{prop:terminal-dual}) & $n$-scaling, closure/interior conditions \\
\bottomrule
\end{tabular}
\caption{Classical bounds as one comparator-class terminal game. Each fixes the triple $(\pi,S,A)$; the comparator-class dual of Proposition~\ref{prop:terminal-dual} supplies the variational form, which takes the closed form of Proposition~\ref{prop:event-restricted} on a face; the remaining ingredient is a standard large-deviation step. The shared ingredient is a Gibbs / I-projection of the conditioned base measure. Appendix~\ref{app:four-classical} carries each row from the game to the classical statement, closing step included.}
\label{tab:event-restricted}
\end{table}

\subsection{Comparator-centered luckiness}
\label{sec:luckiness}

A \emph{luckiness} guarantee is one whose regret bound adapts to the difficulty of the realized loss sequence, sharpening on low-noise instances below the worst-case rate.
For losses in $[0,1]$, the exact intrinsic increment can be upper bounded by a comparator-centered second moment.

\begin{proposition}[Comparator-centered second-order envelope]
\label{prop:luckiness-envelope}
Assume $c_t(i)\in[0,1]$ and $\eta_t\in(0,1]$.
Define $\Psi_t(\nu):=\sum_{i=1}^K p_t(i)(c_t(i)-\inner{\nu}{c_t})^2$.
Then
\begin{equation}\label{eq:e-2-envelope}
Q_t(c)\le (e-2)\Psi_t(\nu)
\end{equation}
Hence for a nonincreasing schedule $(\eta_t)$,
$R_T^c(\nu)
\le
\eta_T^{-1}\KL(\nu\|\pi)
+
(e-2)\sum_{t=1}^T \eta_t \Psi_t(\nu)$.
\end{proposition}

For i.i.d.\ losses $c_1,c_2,\dots$ in $[0,1]^K$ with mean $\mu=\E c_1$, a comparator $\nu\in\DeltaK$ satisfies a \emph{low-noise condition} with constant $\kappa_\nu \ge 0$ when $\E[(c_t(i)-\inner{\nu}{c_t})^2] \le \kappa_\nu (\mu(i)-\inner{\nu}{\mu})$ for every $i\in[K]$.

\begin{theorem}[Fixed-rate stochastic luckiness]
\label{thm:luckiness}
Assume the low-noise condition above with constant $\kappa_\nu$, and run fixed-rate Hedge on the composite losses with $\eta\in(0,1]$.
If $(e-2)\kappa_\nu\eta<1$, then
\begin{equation}\label{eq:luckiness}
\E R_T^c(\nu)
\le
\frac{\KL(\nu\|\pi)}{\eta(1-(e-2)\kappa_\nu\eta)}
\end{equation}
In particular, choosing $\eta=\min\{1,[2(e-2)\kappa_\nu]^{-1}\}$ gives a constant-in-$T$ expected regret bound:
$\E R_T^c(\nu)\le 2(1+2(e-2)\kappa_\nu)\KL(\nu\|\pi)$.
\end{theorem}

The same phenomenon survives the second-order schedule: under \eqref{eq:schedule},
\begin{equation}\label{eq:predictable-luckiness}
\E R_T^c(\nu)
\le
2\Gamma+2\,\E[Q_T^*(c)]+(2C+C^{-1})^2(e-2)\kappa_\nu\Gamma
\end{equation}
so hard and easy sequences are not different theorems but different realized fixed points of the same game.
The low-noise condition matches the standard point-mass condition at $\nu=\delta_{k^*}$, and is satisfiable with a finite constant only when the comparator ties the minimal mean $\mu^*=\min_j\mu(j)$, so Theorem~\ref{thm:luckiness} applies on the minimal-mean face.

A comparator-mean monotonicity carries the guarantee past that restriction.
Every comparator above the minimal mean is beaten by a linear margin and so bounded a fortiori, and one that ties it inherits the best expert's constant-in-$T$ bound (Proposition~\ref{prop:diffuse-reduces}).
For finite $K$ with a strict off-face mean gap the remaining tied-minimum face is handled separately (Appendix~\ref{app:luckiness-faces}).
The no-gap regime, a continuum of experts or means accumulating at $\mu^*$, remains open.

\subsection{Confidence regions and confidence trajectories}
\label{sec:confidence}

Because the terminal term is a relative-entropy transport, exponential-family slices of the game trace out level sets of the comparator's terminal information transport.
For a parametric family $\{\nu_\theta:\theta\in\Theta\}$ define
\begin{equation}\label{eq:confidence-fiber}
\Lambda_T(\theta)
:=
\KL(\nu_\theta\|\pi)-\KL(\nu_\theta\|\rho_{T,\eta_T})
=\E_{\nu_\theta}\!\left[\log\frac{\rho_{T,\eta_T}}{\pi}\right]
\end{equation}
the amount the data path has moved toward $\nu_\theta$ in relative entropy, and in the exact regret identity $\inner{\nu_\theta}{S_T}=\sum_t\eta_t Q_t(c)+D_T+\Lambda_T(\theta)/\eta_T$ it is literally the leftover description cost of $\theta$ after processing the sequence.
Higher $\Lambda_T$ marks comparators the data favors over the prior, so the data-supported region is $\cC_T(\Gamma):=\{\theta:\Lambda_T(\theta)\ge -\Gamma\}$, the comparators not yet rejected at log-Bayes-factor level $\Gamma$.
At point masses this is the standard posterior-tail confidence set, and for a general exponential family it is a super-level set of $\Lambda_T$ in the log-partition geometry, which need not be convex.
Where the prior-posterior ratio is an $e$-process --- a nonnegative process with expectation at most one at every stopping time --- inverting $\mathrm{PR}_T(\theta):=\pi(\theta)/\rho_{T,\eta_T}(\theta)$ yields anytime-valid sets $\cC_T^{\mathrm{any}}(\alpha):=\{\theta:\sup_{t\le T}\mathrm{PR}_t(\theta)<1/\alpha\}$.
The hypotheses under which the ratio is an $e$-process, and the cost of a post-hoc reading of such sets, are not pursued here.

Reading these sets at a horizon fixed in advance separates two terms that an anytime construction incurs together. Both invert the same statistic and differ only in which of its values is consulted: $\cC_T^{\mathrm{any}}(\alpha)$ takes the running supremum, where the fixed-time set $\cC_T^{\mathrm{fix}}(\alpha):=\{\theta:\mathrm{PR}_T(\theta)<1/\alpha\}$ takes the endpoint. Writing $\mathrm{DD}_T(\theta):=\Lambda_T(\theta)-\inf_{t\le T}\Lambda_t(\theta)\ge0$ for the drawdown of the transport trajectory,
\begin{equation}\label{eq:anytime-drawdown}
\cC_T^{\mathrm{any}}(\alpha)=\bigl\{\theta:\Lambda_T(\theta)>\mathrm{DD}_T(\theta)-\log(1/\alpha)\bigr\}
\end{equation}
the fixed-time set read at the tightened level $\log(1/\alpha)-\mathrm{DD}_T(\theta)$.
The budget an anytime set holds back for the times it is not consulted is the drawdown, comparator by comparator, and it is read off the realized path without being estimated.
Retuning the scale decides whether it keeps growing with the horizon: at a scale held fixed, the drawdown tracks the $e$-process's accumulated log-moment deficit, which is linear in $T$; retuned to each horizon it stays flat (Appendix~\ref{app:drawdown-numerics}).

The other fixed-horizon saving sits in the scale.
Proposition~\ref{prop:horizon-free} licenses taking it: the constant scale $\eta=\sqrt{\Gamma/V}$ attains the value from the finite horizon $T_0$ on, and $T_0=2$ at the comparator budgets of interest.
A learner told the horizon in advance therefore pre-commits that scale and incurs no drift, $D_T=0$.
One that is not must instead run the plug-in schedule of Theorem~\ref{thm:sqrt-envelope}, whose leading constant is $2\sqrt2$ in place of $2$, or mix over the scale at the Jensen surplus of Section~\ref{sec:exact-values-rates}.
Mixing is therefore wider than a fixed-horizon interval read once, at both its own best tuning and the conventional one, and the extra width is visible twice: once directly, and once as coverage the sequence never uses (Appendix~\ref{app:drawdown-numerics}).
The residual conservatism common to all three is the exponential-moment step itself, which the game accounts for in $A_\Gamma$ and no choice of horizon removes.
Fixing the horizon therefore concedes nothing to nature, by Proposition~\ref{prop:horizon-free}; it recovers the scale.

\section{Repeated games, equilibria, and the information ledger}
\label{sec:games}

Every setting so far has one learner facing nature.
Let each player of a repeated game run the concentration game instead, each with its own prior, its own schedule, and its own ledger.
The distance from equilibrium is then not merely bounded by those ledgers but equal to a sum of their entries: the duality gap in a zero-sum game (Theorem~\ref{thm:self-play-ledger}), and the distance to coarse correlated equilibrium in a general one (Theorem~\ref{thm:k-player-ledger}).
Two things follow from having the equality rather than a bound.
The rate is read off the play that actually happened, so it improves automatically on the sequences where the players settle, and costs nothing on the sequences where they do not.
And the classical guarantees are recovered by relaxing the entries one at a time, so what each classical step discards is visible as a specific term.
No quadratic-variation surrogate enters at any point; the standard variance bounds are the small-scale limit of what follows.

\subsection{Repeated zero-sum play and intrinsic-time regret matching}

Consider a repeated zero-sum matrix game with row-loss matrix $M\in[0,1]^{K\times L}$.
If the column player plays $\omega_t\in\Delta([L])$, the row-loss vector is $\ell_t(i)=(M\omega_t)_i$, and the centered regret vector is $g_t(i):=\inner{p_t}{\ell_t}-\ell_t(i)$ with $\inner{p_t}{g_t}=0$.
Running the retempered concentration game on $(g_t)$ in place of the raw losses is \emph{intrinsic-time regret matching}, the strategy used throughout this section.

A superscript or argument on a ledger entry names the sequence that entry is computed on.
So $Q_t(g)$, $D_T^{g}$, and $B_T^{g}(\nu)$ are the three entries of Theorem~\ref{thm:variable} --- per-round centered RCGF, retempering drift, and terminal relative-entropy transport --- evaluated on the centered regret vectors $(g_t)$ in place of the raw losses, and $V_T(g):=\sum_{t=1}^T Q_t(g)$ is the corresponding intrinsic time.

Nothing about the ledger changes under this substitution, and the square-root envelope of Theorem~\ref{thm:sqrt-envelope} comes with it.

\begin{theorem}[Intrinsic-time regret matching]
\label{thm:regret-matching}
If the learner runs the retempered concentration game on the centered regret vectors~$g_t$, then for every comparator $\nu\in\DeltaK$,
\begin{equation}\label{eq:regret-matching}
\sum_{t=1}^T \bigl(\inner{p_t}{\ell_t}-\inner{\nu}{\ell_t}\bigr)
=
\sum_{t=1}^T \eta_t Q_t(g)
+
D_T^{g}
+
B_T^{g}(\nu)
\end{equation}
In particular, for each pure action~$i$ and budget $\Gamma\ge -\log\pi(i)$,
$\sum_{t=1}^T(\inner{p_t}{\ell_t}-\ell_t(i))
\le
\Gamma + Q_T^*(g)+(2C+C^{-1})\sqrt{\Gamma V_T(g)}$.
\end{theorem}

The schedule needs no knowledge of the loss scale.
Scaling all loss vectors by a constant $a>0$ scales $Q_t(g)$ and $V_T(g)$ by $a^2$ and the learning rate by $1/a$, so the schedule form is scale-free in the unclipped regime (Appendix~\ref{app:pf-games}).

Several candidate opponent models are accommodated in the prior, before play begins.
Pooling them multiplicatively in prior space is ordinary exponential response to the averaged opponent model, and the worst case over poolings is a robust center; running the game on the residuals against that center converges quickly whenever one candidate is close to the opponent's actual strategy (Appendix~\ref{app:robust-center}).

\subsection{An exact information ledger: two-player duality gap}
\label{sec:self-play-ledger}

Now let both players run intrinsic-time regret matching, with respective priors $\pi^{\mathrm{row}}\in\DeltaK$ and $\pi^{\mathrm{col}}\in\Delta([L])$ and schedules $(\eta_t^{\mathrm{row}})$ and $(\eta_t^{\mathrm{col}})$.
Each keeps its own centered regret vectors, centered under its own play: the row player's are $g_t^{\mathrm{row}}(i):=\inner{p_t}{M\omega_t}-(M\omega_t)_i$, and the column player's are $g_t^{\mathrm{col}}(j):=(M^\top p_t)_j-\inner{\omega_t}{M^\top p_t}$, in the zero-sum convention in which the column player's loss vector is $-M^\top p_t$.

How far the averaged play sits from the game's value is then the two players' ledgers added together, entry by entry, with nothing left over.

\begin{theorem}[Exact information ledger of self-play]
\label{thm:self-play-ledger}
With
\[
P_T^{\mathrm{row}}:=\sum_{t=1}^T \eta_t^{\mathrm{row}}Q_t^{\mathrm{row}}(g),
\qquad
P_T^{\mathrm{col}}:=\sum_{t=1}^T \eta_t^{\mathrm{col}}Q_t^{\mathrm{col}}(g),
\]
let $D_T^{\mathrm{row}},D_T^{\mathrm{col}}$ be the corresponding retempering drifts and $B_T^{\mathrm{row}}(\nu),B_T^{\mathrm{col}}(\nu')$ the terminal relative-entropy transports.
Then for every comparator pair $(\nu,\nu')\in\DeltaK\times\Delta([L])$,
\begin{equation}\label{eq:self-play-ledger}
\bar p_T^\top M\nu'-\nu^\top M\bar\omega_T
=
\frac{1}{T}\Bigl(P_T^{\mathrm{row}}+P_T^{\mathrm{col}}+D_T^{\mathrm{row}}+D_T^{\mathrm{col}}+B_T^{\mathrm{row}}(\nu)+B_T^{\mathrm{col}}(\nu')\Bigr)
\end{equation}
Suppose further that the complexity budgets admit every pure deviation, $\Gamma^{\mathrm{row}}\ge\max_i\log(1/\pi^{\mathrm{row}}(i))$ and $\Gamma^{\mathrm{col}}\ge\max_j\log(1/\pi^{\mathrm{col}}(j))$. Specializing the right-hand side to maximizing comparators inside those budgets then yields the duality gap
\begin{equation}\label{eq:saddle-gap-exact}
\max_{\omega'}\bar p_T^\top M\omega' - \min_{p'} p'^\top M\bar\omega_T
=
\frac{P_T^{\mathrm{row}}+D_T^{\mathrm{row}}+B_T^{\mathrm{row},\star}+P_T^{\mathrm{col}}+D_T^{\mathrm{col}}+B_T^{\mathrm{col},\star}}{T}
\end{equation}
where $B_T^{\mathrm{row},\star}:=\sup\{B_T^{\mathrm{row}}(\nu):\KL(\nu\|\pi^{\mathrm{row}})\le\Gamma^{\mathrm{row}}\}$ and analogously for the column player.
\end{theorem}

Every term on the right of \eqref{eq:saddle-gap-exact} is an RCGF or a relative entropy drawn from one of the two players' own ledgers, and none is a variance.
Under smaller budgets the same right-hand side is exact for the relative-entropy-restricted deviation classes, and the unrestricted duality gap on the left may exceed it.
The classical self-play gap bound \cite{cesa2006prediction} and its quadratic-variation refinements are relaxations of it: setting $C=1/\sqrt{2}$ and applying the variance bound $Q_t\le(e-2)\Var_{p_t}(g_t)$ (Proposition~\ref{prop:luckiness-envelope}) to each player's loss recovers the standard $O(\sqrt{\log K /T})$ self-play rate as the small-$\eta$ limit.

The exact form adds a rate that responds to the play.
The horizon-growing part of \eqref{eq:saddle-gap-exact} is governed by the loss $P_T^{(k)}=\sum_t\eta_t^{(k)}Q_t^{(k)}$ and the transport $B_T^{(k),\star}\le\Gamma^{(k)}/\eta_T^{(k)}$, both of order $\sqrt{\Gamma^{(k)}V_T^{(k)}}$ in the cumulative intrinsic time $V_T^{(k)}:=\sum_t Q_t^{(k)}$ under the optimal schedule for a complexity budget $\Gamma^{(k)}$; the drift is a gain.
Sublinear $V_T^{(k)}$ in either player therefore produces a strictly faster duality-gap rate: at $V_T^{(k)}=O(\log T)$ the exact gap is $O(\sqrt{\log T}/T)$ in expectation, strictly faster than the $\Theta(1/\sqrt T)$ worst case and with no hidden logarithmic factors.

Two scope conditions bound that reading.
The equality holds for the present family, so it constrains what that family realizes and is not a minimax lower bound over all algorithms; an update rule outside it can realize a smaller gap on a given sequence, and the identity neither predicts nor precludes this.
And within the family a gap below the loss's lower envelope is available only through the drift gain, which leaves it undecided whether $V_T^{(k)}=\Theta(T)$ forces a $\Theta(1/\sqrt T)$ gap at $C=1/\sqrt2$ (Appendix~\ref{app:cce-converse}).

\subsection{General-sum play and equilibrium consequences}

Beyond two players the target is a correlated equilibrium notion, and the one-player game composes in the standard way.
External regret at most $\varepsilon_m T$ for every player of a finite normal-form game makes the empirical distribution of play an $\varepsilon$-coarse correlated equilibrium (CCE).
Swap regret at that level makes it an $\varepsilon$-correlated equilibrium \cite{hart2000simpleb}, at $\varepsilon=\max_m\varepsilon_m$ in both cases (Propositions~\ref{prop:cce} and~\ref{prop:swap-ce}).

Those are one-way implications, and measuring the distance directly makes them quantitative.
For a joint distribution $\sigma\in\Delta\bigl(\prod_k [K_k]\bigr)$ on action profiles, define the \emph{CCE-deviation distance}
\begin{equation}\label{eq:cce-distance}
d(\sigma,\mathrm{CCE})
:=
\max_{k}\;\sup_{\nu_k\in\Delta([K_k])}
\E_{\bm a\sim\sigma}\bigl[u_k(\nu_k,\bm a_{-k})-u_k(\bm a)\bigr]
\end{equation}
where $u_k:\prod_{k'}[K_{k'}]\to\R$ is player~$k$'s utility and we adopt the convention $u_k=-c_k$ for cost-minimizing players.
This is the most any single player could gain by deviating on its own; $\sigma\in\mathrm{CCE}$ iff $d(\sigma,\mathrm{CCE})\le 0$ (the gain can be strictly negative when every unilateral deviation strictly loses), and $\sigma$ is an $\varepsilon$-CCE iff $d(\sigma,\mathrm{CCE})\le\varepsilon$ \cite{hart2000simpleb,blum2007external}.

Evaluated at the empirical joint $\sigma_T:=\frac{1}{T}\sum_{t=1}^T\bigotimes_k p_t^{(k)}$ of the players' mixed actions, that distance equals the largest of the players' ledgers.

\begin{theorem}[Exact per-player external-regret ledger and CCE distance]
\label{thm:k-player-ledger}
For a finite game, let each player $k$ run intrinsic-time regret matching with prior $\pi^{(k)}$, schedule $(\eta_t^{(k)})$, and comparator budget $\Gamma^{(k)}\ge\max_i\log(1/\pi^{(k)}(i))$, so that every pure deviation lies in the comparator class.
Each player's external regret admits the exact decomposition (Theorem~\ref{thm:variable}) $R_T^{(k)}(\nu)=P_T^{(k)}+D_T^{(k)}+B_T^{(k)}(\nu)$ into RCGF losses, retempering drift, and terminal relative-entropy transport.
The corresponding worst-case-over-comparator sum is $P_T^{(k)}+D_T^{(k)}+B_T^{(k),\star}$, with $B_T^{(k),\star}:=\sup\{B_T^{(k)}(\nu_k):\KL(\nu_k\|\pi^{(k)})\le\Gamma^{(k)}\}$.
Then
\begin{equation}\label{eq:k-player-ledger}
d(\sigma_T,\mathrm{CCE})
=
\frac{1}{T}\max_{k}\Bigl(P_T^{(k)}+D_T^{(k)}+B_T^{(k),\star}\Bigr)
\end{equation}
\end{theorem}

Both sides are exact, so for the mixed-action joint the one-way implication above becomes an identity.
With a smaller $\Gamma^{(k)}$ the right-hand side governs only the relative-entropy-restricted deviation class $\{\nu_k:\KL(\nu_k\|\pi^{(k)})\le\Gamma^{(k)}\}$.
If instead the players sample pure profiles $\bm a_t$ and $\sigma_T$ is read as $\frac{1}{T}\sum_t\delta_{\bm a_t}$, the same ledger holds in expectation up to an additive $O(\sqrt{\log(\sum_k K_k)/T})$ from the play martingales, which are mean-zero at each fixed deviation.

The identity turns a rate question into a question about whether the play settles.
Under the self-tuned schedule $\eta_t^{(k)}=C\sqrt{\Gamma^{(k)}/(V_{t-1}^{(k)}+\Gamma^{(k)})}$, a trajectory converging to a \emph{strict} pure Nash profile freezes the scale at a positive value and decays the per-round RCGF geometrically, with no assumption on the rate of approach.
Each player's realized intrinsic time is then \emph{bounded}, $V_T^{(k)}=O(1)$, and hence $d(\sigma_T,\mathrm{CCE})=O(1/T)$ --- strictly faster than the classical $\Theta(1/\sqrt T)$ (Theorem~\ref{thm:unconditional-V}, Appendix~\ref{app:proofs-games}).

That statement is conditional on convergence to a strict pure profile, which a finite potential or congestion game need not produce: fixed-step multiplicative weights admits limit cycles and chaos there \cite{palaiopanos2017congestion}.
And while the schedule's negative feedback --- growing $V^{(k)}$ forces $\eta$ downward, toward the small-step replicator-dynamics limit --- heuristically self-stabilizes, whether $V_T^{(k)}$ stays sublinear for every such game is a convergence-of-dynamics question the schedule does not settle.

The link between stabilization and the improved rate runs in one direction only.
The CCE distance is a per-player \emph{maximum}, exactly the worst single player's unilateral deviation gain, and $V_T^{(k)}$ governs every horizon-growing term of that player's regret, so $\max_k V_T^{(k)}=o(T)$ implies $d(\sigma_T,\CCE)=o(T^{-1/2})$.
Since $Q_t^{(k)}$ vanishes exactly as player $k$'s centered regret vector becomes constant on the support of its last-iterate play (Proposition~\ref{prop:tilted-var}), that hypothesis is a last-iterate stabilization statement.
The play itself settles, beyond its time average, either by concentrating on a single action or by the regrets flattening at an interior equilibrium.
The converse is unavailable from the ledger, since the residual coefficient $2C-C^{-1}$ is exactly zero at the headline $C=1/\sqrt2$ (Appendix~\ref{app:cce-converse}).

Which regime a game falls into is therefore settled by the realized intrinsic time alone.
On potential and congestion instances converging to a pure equilibrium, plain Gibbs stabilizes and \eqref{eq:k-player-ledger} gives the instance-adaptive rate.
On interior-equilibrium games such as zero-sum polymatrix games, plain Gibbs need not stabilize: a cyclic loss sequence holds the per-round centered RCGF away from zero, giving $V_T^{(k)}=\Theta(T)$ and the classical $O(1/\sqrt T)$, whose small-$\eta$ variance relaxation recovers the bound of \cite{cesa2006prediction}.
The guaranteed $O(1/T)$ improvement there needs an optimistic predictable schedule \cite{daskalakis2018training} outside the Gibbs family, and unlike the \emph{sum} of regrets, whose loss-movement telescopes across players under such a schedule \cite{syrgkanis2015neurips}, the CCE maximum admits no cross-player cancellation.
Appendix~\ref{app:proofs-games} works the canonical instances through the ledger; in two-player zero-sum, $d(\sigma_T,\mathrm{CCE})$ is the larger of the two exact deviation gains, at most the Nash duality gap of \eqref{eq:saddle-gap-exact}.

\section{Partial information: bandits, graphs, context, and robust scores}
\label{sec:partial}

In the full-information game of Sections~\ref{sec:primitive}--\ref{sec:games} the learner sees nature's entire loss vector $c_t$, and its regret decomposes exactly into intrinsic-time loss, retempering drift, and terminal relative-entropy transport. Partial information changes one thing (Game~\ref{game:bandit}): the learner no longer observes $c_t$, only the coordinate it samples, and must commit both a play distribution $p_t$ and a possibly distinct sampling law $\mu_t$ \emph{before} nature moves. It then runs the same intrinsic-time game on an \emph{estimated} score $\widehat c_t$; the only new terms are the bandit-specific play and estimation martingales, an explicit exploration cost, and any predictable estimation bias, which separate cleanly from that game (Theorem~\ref{thm:bandit}).

\begin{game}[The bandit concentration game]
\label{game:bandit}
Game~\ref{game:central}'s data --- a strictly positive prior $\pi\in\DeltaK$, a comparator budget $\Gamma\ge0$, and per-round RCGF budgets $(q_t)$ --- with the state, the moves, and the terminal payoff changed.
The state is the cumulative \emph{estimated} centered score $\widehat S_t:=\sum_{s\le t}\widehat z_s$ (from $\widehat S_0=0$), and the learner selects its own scale as in Game~\ref{game:cgf}.
Each round, after observing the context $x_t$ in the contextual case, the learner picks $\eta_t>0$ predictably, plays the Gibbs tilt $p_t=G_{\eta_t}(\widehat S_{t-1})$ of the estimated state, and picks a sampling distribution $\mu_t\in\Delta([K])$, which may differ from $p_t$ for exploration.
Nature then commits the true loss $c_t\in\R^K$, unseen by the learner and held to a per-round budget on the estimated centered RCGF; the learner samples $A_t\sim\mu_t$, observes only $c_t(A_t)$, forms an estimate $\widehat c_t$ of $c_t$, and advances the state by the centered estimated score $\widehat z_t:=\inner{p_t}{\widehat c_t}\1-\widehat c_t$.
The payoff to nature is the sampled composite-loss regret $\mathrm{Reg}_T^c(u):=\sum_{t=1}^T c_t(A_t)-\inner{u}{C_T}$, its supremum over the comparator ball $\{u\in\DeltaK:\KL(u\|\pi)\le\Gamma\}$ replacing the terminal payoff $L_\Gamma(S_T)$ of Game~\ref{game:central}.
\end{game}

That budget is imposed on the learner's \emph{estimated} centered RCGF $\widehat Q_t:=Q_{\eta_t}(p_t,\widehat z_t)$: since $c_t$ is committed before the arm is drawn, the constraint is on the conditional mean $\E[\widehat Q_t\mid\mathcal F_{t-1}]\le q_t$, the analogue of the full-information cap.
Because that expected estimated RCGF is an explicit closed form in the true losses and geometry $\mu_t$, it pins a feasible set for the true $c_t$ once $\mu_t$ is fixed.

\subsection{Generic bandit decomposition}

For a comparator $u\in\DeltaK$ the sampled composite-loss regret decomposes pathwise (Theorem~\ref{thm:bandit}, Appendix~\ref{app:proofs-partial}) as
\[
\mathrm{Reg}_T^c(u)=M_T^{\mathrm{play}}+\Xi_T-M_T^{\mathrm{est}}(u)+\mathrm{Bias}_T(u)+\widehat D_T+\widehat B_T(u)+\sum_{t=1}^T\eta_t\widehat Q_t
\]
into the full-information game on estimated scores (the last three terms) plus an observation penalty of play/estimation martingales, exploration cost $\Xi_T=\sum_t\inner{\mu_t-p_t}{c_t}$, and predictable bias.

The standard bandit templates are all estimator-specific instances of this game: explicit-exploration inverse-propensity scoring (IPS; bias vanishes, and $\Xi_T$ records the exploration cost), implicit-exploration EXP3-IX, predictable-offset control variates, and feedback graphs.
Graph topology enters only through the observation geometry $p_t(a)/o_t(a)$, with $o_t(a)$ the probability that arm~$a$ is observed on the round; for undirected graphs that geometry is bounded by the independence number, the size of the largest set of mutually non-adjacent vertices \cite{alon2017journalb}. Each is a choice of estimated score $\widehat c_t$ fed to the otherwise-unchanged game (Appendix~\ref{app:proofs-partial}).

\subsection{An identity for the estimated RCGF}

The one term this leaves implicit, the estimated per-round RCGF $\widehat Q_t$, is not to be bounded but has a closed form. Conditioned on the past its only randomness is the drawn arm, so $\E[\widehat Q_t\mid\mathcal F_{t-1}]=\eta_t^{-2}\sum_a\mu_t(a)\log\sum_i p_t(i)e^{\eta_t\widehat z_t^{(a)}(i)}$, written entirely in the true losses and the observation geometry $\mu_t$, with small-scale limit the propensity-inflated variance $\tfrac12\sum_a\tfrac{p_t(a)(1-p_t(a))}{\mu_t(a)}c_t(a)^2$ (Proposition~\ref{prop:bandit-exact-value}, Appendix~\ref{app:proofs-partial}).

Taking expectations of \eqref{eq:bandit-decomp} annihilates the two martingales, leaves the exploration term and predictable bias as functions of the true loss, and leaves the estimated intrinsic time $\sum_t\eta_t\E[\widehat Q_t]$ in the closed form just given; the estimated drift and transport are the same prior-transport ledger as the full-information game. Partial information therefore introduces no new obstruction: the exact minimax value is the min-max of an explicit function of the true losses and observation geometry, and solving it in closed form for general $K$ is the full-information exact-value problem itself (solved for $K=2$ over one round by Proposition~\ref{prop:k2-exact}). The worst-case slice $\mu_t=p_t$ caps the per-round value at $K/2$, recovering the $O(\sqrt{KT\log K})$ exponential-weights rate as the small-scale relaxation; closing the remaining $\sqrt{\log K}$ factor remains open (Section~\ref{sec:discussion}).

\subsection{Contextual wrappers and identification scores}

Only the estimated score fed to the otherwise-unchanged game changes across the remaining partial-information settings: the proper contextual-bandit wrapper chooses $\widehat c_t$ by lifting a policy class, and the testing-exponent scores choose it as a log-likelihood-ratio.

Treating a finite policy class as experts and lifting it by duplication runs the full-information game on importance-weighted losses, at a comparator complexity equal to the aggregate prior mass of a designated policy's loss-identical copies.
The resulting wrapper bound holds with exact constants, and its sharp rate closes on the estimated per-round RCGF exactly as the plain-bandit rate does (Proposition~\ref{prop:duplication}, Corollary~\ref{cor:duplication-reduction}).

Scoring a log-likelihood ratio instead identifies a latent model, and the game's own quantities are the classical testing ones.
The Chernoff testing coefficient \cite{chernoff1952measure} is minus the per-observation Bellman potential, so the Gibbs equalizer is the Chernoff tilt and optimizing the scale is Chernoff information.
The pairwise testing exponent is then the edge-restricted multi-way coincidence radius of the arm-induced laws \cite{balsubramani2026information}, equal to the MAP error exponent for identifying the hypothesis from repeated pulls of an arm (Theorem~\ref{thm:testing-bellman}, Appendix~\ref{app:proofs-partial}).

\section{Thompson sampling and prior-posterior-ratio martingales}
\label{sec:ts}

Randomizing the learner's play exposes two martingale readings of the same game. The first is the sampling martingale of Thompson sampling: posterior sampling attains the deterministic concentration-game value in expectation, and the entire effect of sampling is a single mean-zero martingale on the realized path. The second is the prior-posterior-ratio martingale, which recasts the game's own Bayes update as the test supermartingale of anytime-valid inference. Their step-by-step derivations are collected in Appendix~\ref{app:pf-ts}.

\subsection{Thompson sampling as the same game plus a martingale}

If the learner samples $I_t\sim p_t$ and plays that expert, the sampled composite-loss regret $\widehat R_T^c(\nu):=\sum_{t=1}^T c_t(I_t)-\inner{\nu}{C_T}$ decomposes exactly as
\begin{equation}\label{eq:ts}
\widehat R_T^c(\nu)
=
M_T^{\mathrm{sam}}
+
D_T
+
B_T(\nu)
+
\sum_{t=1}^T \eta_t Q_t(c)
\end{equation}
with $M_T^{\mathrm{sam}}:=\sum_{t=1}^T (c_t(I_t)-\inner{p_t}{c_t})$ a martingale difference: Thompson sampling incurs the same deterministic concentration-game value \emph{in expectation} as exponential weights, plus a mean-zero martingale whose per-realization size is a genuine path cost. Taking $\E[\,\cdot\mid\mathcal F_0]$ and $\E[M_T^{\mathrm{sam}}]=0$,
\begin{equation}\label{eq:ts-expected}
\E\bigl[\widehat R_T^c(\nu)\bigr]
=
D_T+B_T(\nu)+\sum_{t=1}^T\eta_t Q_t(c)
=
R_T^c(\nu)
\end{equation}
for every comparator and predictable loss sequence---the sampling step contributes nothing in expectation. Randomizing and then scoring in expectation therefore cancel exactly, pinning down an \emph{expected sampled concentration game} (Game~\ref{game:ts}) whose saddle point is posterior sampling (Theorem~\ref{thm:ts-exact}, Appendix~\ref{app:proofs-ts}). This is a value identity: $w_t=p_t$ attains the deterministic value in expectation as the \emph{unique} equalizer at value $\Gamma/\eta+\eta\sum_t q_t$, while any $w_t\neq p_t$ over-drives the Gibbs increment and strictly loses.

The resolution is exact \emph{in expectation}: per realization the sampled regret fluctuates by the mean-zero martingale $M_T^{\mathrm{sam}}$, whose typical magnitude---the square root of the cumulative running variance $V^{\mathrm{sam}}_T:=\sum_t\Var_{p_t}(c_t)$---is a genuine cost on the realized path that the construction leaves in place.
No scoring rule removes it: there is no fluctuation-free realized posterior-sampling game, and the sharpest high-probability control uses a different scale from the in-expectation optimum (Proposition~\ref{prop:ts-no-pathwise}).
A certificate for its size comes from Ville's inequality applied to the exponential supermartingale on the variance clock $V^{\mathrm{sam}}_t$, which bounds $M_t^{\mathrm{sam}}$ uniformly in time.
At a constant scale the prior-posterior ratio of Section~\ref{sec:ppr} is the companion $e$-process under the usual hypotheses, and inverting it gives the anytime-valid comparator sets of Section~\ref{sec:confidence} (Appendix~\ref{app:pf-ts}; an exact decomposition of the slack in Ville's inequality is obtained in concurrent work \cite{delapena2026exact}).

The sampling martingale concentrates at the standard rates --- a high-probability $\sqrt{2V^{\mathrm{sam}}_T\log(1/\delta)}$, an anytime bound via Ville's inequality, and a $O(\sqrt{V^{\mathrm{sam}}_T\log\log V^{\mathrm{sam}}_T})$ law of the iterated logarithm (Appendix~\ref{app:pf-ts}).
In the worst case its cost matches the exponential-weights regret up to the iterated-logarithm factor, a genuine cost of comparable size; in the low-noise regime it vanishes in expectation, while the regret is $O(\Gamma)$.

\subsection{The prior-posterior-ratio martingale}
\label{sec:ppr}

The second martingale reading turns the game's own update into a test statistic: track how far a fixed comparator $\nu$ sits from the running posterior in relative entropy, and each Bayes step moves that distance by exactly the comparator's per-round margin.

Writing $m_t(\eta_t):=-\eta_t^{-1}\log\sum_i p_t(i)e^{-\eta_t c_t(i)}$ for the round's mix loss, an elementary expansion of $\log(p_{t+1}/p_t)$ for the update $p_{t+1}(i)\propto p_t(i)e^{-\eta_t c_t(i)}$ (Appendix~\ref{app:pf-ts}) gives the transport identity
\begin{equation}\label{eq:kl-transport-ppr}
\KL(\nu\|p_{t+1})-\KL(\nu\|p_t)
=
\eta_t(\inner{\nu}{c_t}-m_t(\eta_t))
\end{equation}
for every fixed comparator $\nu$.
Read the right-hand side as the comparator's \emph{margin}: the relative entropy from $\nu$ to the posterior falls by $\eta_t(m_t-\inner{\nu}{c_t})$ on every round the comparator's loss beats the mix.
Summing over the run, $\KL(\nu\|\pi)-\KL(\nu\|p_{T+1})=\sum_t\eta_t(m_t-\inner{\nu}{c_t})$, and nonnegativity of the terminal relative entropy caps the comparator's total accumulated margin at $\KL(\nu\|\pi)\le\Gamma$ --- the same information budget that identifies the comparator elsewhere in the game.

Sampling turns this deterministic ledger into a test. Under Thompson sampling with $I_t\sim p_t$, the log-weight of the sampled expert moves by $\log p_{t+1}(I_t)-\log p_t(I_t)=-\eta_t(c_t(I_t)-m_t)$, whose conditional mean is the nonpositive drift $-\eta_t\delta_t(c)$; the excess over that drift is the prior-posterior-ratio (PPR) martingale, connecting the concentration game to anytime-valid inference \cite{ramdas2023statistical}.

This martingale is central to a dual, testing-facing reading of the game.
The predictable play $p_t=G_{\eta_t}(S_{t-1})$ is the dual variable in a Lagrangian--martingale duality, the game's $\min_p\max_z$ matching the min-over-predictable / max-over-process shape of safe anytime-valid testing \cite{ramdas2023statistical,shafer2021testingb}.
The prior-posterior-ratio supermartingale is the process the learner bets against, and the Gibbs play the equalizing bet that fixes the round's loss whichever direction nature moves.
In the time-uniform regime the concentration-game identity itself reads as the capital process of a bettor against the forecasts --- the Skeptic of game-theoretic probability.
Such a guarantee must control the maximum over the process instead of its endpoint, and the ledger quantifies that difference exactly: the two readings of the same supermartingale differ by the drawdown of the comparator's transport trajectory (Section~\ref{sec:confidence}).

\section{Boosting and the pressure game}
\label{sec:boosting}

Boosting is exactly about the pressure game on signed margins. 

Let $g_t(i)=y_i h_t(x_i)\in[-1,1]$ be the signed margin score of weak learner~$h_t$ on example~$i$, and $M_T(i)=\sum_{t=1}^T \alpha_t g_t(i)$ the cumulative weighted margin. The roles reverse relative to the expert game: the booster plays the learner's role, with the $N$ training examples as its actions and the uniform weighting $u_N$ as its prior; the weak learner supplies nature's move $g_t$; and a comparator is a reweighting of the examples.
The $\varepsilon$-quantile margin below corresponds to the comparator uniform on the hardest $\lceil\varepsilon N\rceil$ examples.

\subsection{From regret to the exponential-loss terminal game}

Running the concentration game on example scores $\ell_t(i)=(1+y_ih_t(x_i))/2$ with round-$t$ edge $\gamma_t:=\tfrac12\inner{p_t}{g_t}$, the standard boosting conventions give, for any posterior $\nu$, $R_T^\ell(\nu)=\tfrac{T}{2}(2\bar\gamma_T-\inner{\nu}{m})$, i.e.\ $\inner{\nu}{m}=2\bar\gamma_T-(2/T)R_T^\ell(\nu)$ with $\bar\gamma_T$ the average edge and $m_i=T^{-1}\sum_t y_ih_t(x_i)$ the unit-weight average margin ($M_T$ is its $\alpha_t$-weighted, unaveraged counterpart).
Hence any regret bound $R_T^\ell(\nu)\le U_T(\nu)$ implies $\inner{\nu}{m}\ge 2\bar\gamma_T-2U_T(\nu)/T$.

In particular the $\varepsilon$-quantile margin is controlled by comparator complexity $\log(1/\varepsilon)$ and the intrinsic time of the example-loss sequence: a comparator uniform on the $\lceil\varepsilon N\rceil$ examples of smallest margin has $\KL(\nu\|u_N)\le\log(1/\varepsilon)$.
The envelope \eqref{eq:pathwise-regret} at $\Gamma=\log(1/\varepsilon)$ and $C=1/\sqrt2$ then lets that margin fall short of twice the average edge by at most $2\sqrt{\log(1/\varepsilon)/T}$ plus additive $O(\log(1/\varepsilon)/T)$ edge terms in the worst case, and by less whenever the realized intrinsic time is smaller (Appendix~\ref{app:quantile-margin}).

The exponential training loss also admits a direct mixed-coincidence identity, exhibiting it as an exact terminal game.
For $L_T^{\exp}:=N^{-1}\sum_{i=1}^N e^{-M_T(i)}$, if $p_{t+1}(i)\propto p_t(i)e^{-\alpha_t g_t(i)}$ with $p_1=u_N$, then for every comparator $\nu\in\Delta([N])$,
\begin{equation}\label{eq:boosting-variational}
-\log L_T^{\exp}
+
\KL(\nu\|p_{T+1})
=
\KL(\nu\|u_N)
+
\inner{\nu}{M_T}
\end{equation}
Equivalently, applying the Donsker--Varadhan variational form to $-M_T$,
\[
-\log L_T^{\exp} = \inf_{\nu\in\Delta([N])}\bigl\{\inner{\nu}{M_T}+\KL(\nu\|u_N)\bigr\},
\]
the infimum attained at $\nu=p_{T+1}$, where $\KL(\nu\|p_{T+1})=0$ in \eqref{eq:boosting-variational}.
Training error, margin tails, and quantile margins are therefore controlled by the same terminal concentration functional.

For the margin-tail bound: for every $\theta\in\R$, the fraction of examples with normalized margin at most $\theta$ satisfies $E_T(\theta)\le\exp(\theta A_T)\cdot L_T^{\exp}=\exp\!\bigl(\theta A_T-\inf_\nu\{\inner{\nu}{M_T}+\KL(\nu\|u_N)\}\bigr)$, where $A_T=\sum_t\alpha_t$ (by the Markov/Chernoff step $\mathbf{1}\{M_T(i)\le\theta A_T\}\le e^{-M_T(i)+\theta A_T}$).

\subsection{Pressure-target boosting and AdaBoost}

If $\alpha_t$ is chosen by the local pressure rule $-\alpha_t^{-1}\log\sum_i p_t(i)e^{-\alpha_t g_t(i)}=a_t$, then
\begin{equation}\label{eq:boost-pressure}
-\log L_T^{\exp}=\sum_{t=1}^T \alpha_t a_t,
\qquad
\sum_{t=1}^T \alpha_t(\inner{\nu}{g_t}-a_t)
=
\KL(\nu\|p_{T+1})-\KL(\nu\|u_N)
\end{equation}
The pressure target is an exact exponential-loss descent statement.

For binary weak hypotheses $g_t(i)\in\{-1,+1\}$ the coefficient minimizing the one-step normalizer $\log\sum_i p_t(i)e^{-\alpha g_t(i)}$ --- the unrescaled CGF at scale $\alpha$ --- is the classical AdaBoost weight $\alpha_t^\star=\tfrac12\log\tfrac{1-\varepsilon_t}{\varepsilon_t}$ ($\varepsilon_t=\Pr_{i\sim p_t}\{g_t(i)=-1\}$), the binary special case of the pressure game. It makes the pressure identity evaluate the exact exponential-loss ledger $-\log L_T^{\exp}=\sum_t-\tfrac12\log(1-4\gamma_t^2)$, the Gaussian-channel information of the per-round mean margin. The convexity relaxation $-\log(1-x)\ge x$ then recovers the classical training-error rate $\widehat{\mathrm{err}}_T\le e^{-2\gamma^2 T}$ as a one-line corollary, sharper because it holds for adaptive non-uniform edges (Proposition~\ref{prop:adaboost-cgf}, Appendix~\ref{app:proofs-boosting}).

The exact ledger~\eqref{eq:adaboost-ledger} is strictly stronger than $e^{-2\gamma^2 T}$: it requires no uniform lower bound on the edge, and the per-round gap $-\tfrac12\log(1-4\gamma_t^2)-2\gamma_t^2=4\gamma_t^4+O(\gamma_t^6)$ is exactly what the classical bound discards.

A one-sided shortfall statistic reweights toward the examples still below target. Replacing the signed score $g_t(i)$ by $(a_t-g_t(i))_+$ produces sparse example weights: examples already above target receive zero cost.
This is the boosting analogue of the positive-part sufficient-statistic reduction---a change of statistic, leaving the game itself untouched.

\section{Related work}
\label{sec:related}

The minimax view of online learning as a repeated game is classical \cite{shafer2001probability,ShaferVovk2019,cesa2006prediction,abernethy2008optimal}, and the multiplicative-weights update on the entropic side of it is itself a meta-algorithm whose disparate instances share a single exponential potential \cite{arora2012theory}.
The drifting-games framework \cite{schapire2001drifting,luo2014driftinggames} defines a potential, lets the adversary's feasible set depend on it, and computes the value by backward induction, reformulating Hedge, bandits, and boosting in minimax terms.
The fixed-scale concentration game is the entropic version of that program: the state potential is the log-partition function and the one-step Bellman increment is the centered RCGF $\eta_t Q_t$, where that program chooses a quadratic penalty externally and the variance and range games are its relaxations.
The effective state is $(p_t,V_{t-1})$, and Theorem~\ref{thm:variable} solves the backward induction in closed form without numerical recursion.
The algorithmic ingredients treated here --- entropic FTRL regularizers, relative-entropy-anchored mirror maps, the second-order rate schedule, drifting-game potentials --- are downstream consequences of that constrained min-max value, where a design-first account would install them as upstream choices.
Under the RCGF choice the variable-temperature drift has a clean exponential-family relative-entropy interpretation, where the dual-feasibility / mirror-map choice yields no immediate Bregman expression, which leaves the RCGF as the upstream quantity inside the entropic class.
For repeated play, the equilibrium consequences are the standard external- and swap-regret ones \cite{hart2000simpleb,blum2007external}, derived here in the present notation.
The same no-regret dynamics drive practical equilibrium computation: counterfactual regret minimization solves large imperfect-information games through time-averaged play \cite{brown2019solving}, and optimistic dynamics additionally secure last-iterate convergence in zero-sum games \cite{daskalakis2018training}.

The framework closest in spirit is the game-theoretic probability of Shafer and Vovk \cite{ShaferVovk2019,shafer2001probability}, which casts the laws of large numbers, the central limit theorem, the law of the iterated logarithm, and large deviations as outcomes of a perfect-information game built around a nonnegative-(super)martingale capital process.
The distinction is in the primitive: there a Skeptic bets against a forecaster and its wealth multiplies unless the asserted event holds \cite{shafer2021testingb}, and that line solves for a capital-maximizing bet.
Here the primitive is the comparator-class regret value of an aggregating learner, whose value function is the log-partition and whose solution is a Gibbs equalizing equilibrium.
The two programs are dual, meeting where the prior-posterior-ratio martingale of Section~\ref{sec:ppr} is a Skeptic-style capital process: in the time-uniform regime the log-partition value and the testing supermartingale describe the same process.

The paper sits at the intersection of adaptive online learning, concentration of measure, and robust Bayes.
For adaptive regret it connects with AdaHedge \cite{derooij2014adahedge}, NormalHedge \cite{chaudhuri-freund-hsu}, Squint \cite{koolen2015squint}, coin-betting and parameter-free methods \cite{orabona-pal,Cutkosky19Combining}, data-dependent adaptive-geometry methods \cite{DuchiHazanSinger2011}, and the sequential-complexity program that measures adversarial difficulty by a martingale-indexed capacity \cite{rakhlin2013optimization,rakhlin2014onlineb}.
For concentration it connects with PAC-Bayes and the thermodynamics of learning \cite{catoni2007pacbayesian,alquier2016properties} and with the time-uniform / $e$-process viewpoint \cite{howard2020b,waudbysmith2024,ramdas2023statistical,kaufmann2021mixture}, whose safe anytime-valid constructions the prior-posterior-ratio identity makes explicit at the level of the game.
Closest of all is the line that converts regret into concentration directly: the regret of universal portfolio yields concentration inequalities and confidence sequences tight to the leading constant \cite{orabona2024tight}, and any $e$-process whose betting regret is sublinear is log-optimal for a composite alternative \cite{waudbysmith2025universal}.
That conversion is one-directional --- a regret guarantee yields a tail bound, and the conversion factor between them remains implicit --- and the decomposition here states that factor: the comparator transport $B_T(\nu)$ is the tail statement's contribution, and the per-round loss $\eta_tQ_t$ the sequence's own.
The log-optimality those results establish is the extremal property the reverse information projection selects \cite{larsson2024numeraire}, which is the terminal transport of \eqref{eq:intro-exact} read as a projection onto the comparator class.
The same conversion runs outside the tail-bound setting, in generalization games \cite{lugosi2023onlinetopac} and in sequential likelihood mixing \cite{kirschner2025mixing}.
Each constructs a game to prove a statistical statement; here the statistical statements and the regret are two readings of one value identity, so the construction step only names a comparator class and a per-round budget, and manufactures no bound.
On the adaptive-regret side, the equality-form analyses of adaptive FTRL \cite{mcmahan2017ftrl} decompose regret into a comparator-regularizer term and a per-round stability term, the two roles played here by $B_T(\nu)$ and $\eta_tQ_t$.
The scale enters there as a regularizer the analyst chooses and here as nature's information budget, which is why the retempering drift surfaces as a third term with a closed form instead of being absorbed into the schedule.
For bandits it connects with EXP3, EXP3-IX, Tsallis-INF, and feedback-graph algorithms \cite{auer2002siam,KocakEtAl14,zimmert2021jmlr,alon2017journalb,RouyerEtAl2022FeedbackGraphs}; for Thompson sampling, with the information-directed tradition \cite{RussoVanRoy14,RussoEtAl18,ramdas2023statistical}; and for stochastic low-noise / fast rates, with mixability-based theory \cite{erven2015fast}.
The classical tail inequalities that the terminal payoff accounts for exactly---Markov, Chebyshev, and Chernoff---have themselves recently been sharpened by randomization and exchangeability into never-worse, typically strictly tighter forms \cite{ramdas2026randomized}. In the time-uniform regime the analogous sharpening is exact: in concurrent and independent work, \cite{delapena2026exact} decompose the slack in Ville's inequality for a nonnegative supermartingale into explicit overshoot, predictable-loss, and residual-survival terms, the anytime-valid counterpart to the terminal accounting isolated here.
For dynamic / shifting comparators, the foundational work is \cite{herbster1998tracking,kalai2005efficientb}; the comparator-path identity \eqref{eq:dynamic-comparator} fits inside that program as the pathwise statement that those papers' inequality forms approximate.

The parallel Fenchel-game program for convex optimization \cite{wang2024noregret} runs the other way, converting a static optimization into a game where the concentration game converts a sequential inferential procedure into one. A further parallel is the natural-gradient variational-Bayes program \cite{khan2023bayesian}, which derives a broad range of learning algorithms from a single Bayesian regularized update primitive, where here the corresponding primitive is instead a constrained min-max game value.
For boosting, the main classical reference is \cite{FreundSchapire97,SchapireFreund12BoostingBook}; the pressure-target perspective in Section~\ref{sec:boosting} aligns with the smoothed-boosting line \cite{servedio2003smooth} and the entropy-regularized boosting tradition.

A second adjacent line is the literature on \emph{budgeted adversaries}, where an exogenous budget --- on switches, cumulative loss variation, or adversarial strength relative to a stochastic baseline --- constrains nature's per-round move and regret scales in that budget in place of $T$ alone \cite{abernethy2010repeated}.
Here the round-budget allocation is instead derived from the RCGF/log-partition Bellman structure: $\sum_t Q_t\le V$ is a sufficient statistic of nature's feasibility under a fixed potential, and front-loading is forced by the nonincreasing scale schedule.
That literature's budget is the cumulative intrinsic time $V_T(c)$, and its regret bounds are the variance relaxation of the present identity, with the same $\sqrt{\Gamma V_T}$ leading constant up to relaxation slack.

\section{Discussion}
\label{sec:discussion}

With the game made explicit, the online-learning and concentration primitives treated here share one structure: within the relative-entropy setting each is a constrained min-max value with the same Bellman equalizer, the per-round constraint the centered RCGF and the comparator regularizer relative entropy to a prior.
The variable-temperature decomposition \eqref{eq:exact-variable} makes this operational, both players' strategies read off term by term, and repeated play yields an exact information-theoretic ledger of self-play in place of the usual quadratic-variation surrogate.
The variance and bounded-range quantities of standard regret theory enter only as small-$\eta$ relaxations of that identity, and the derivation does not run in the reverse direction.
The gain is conceptual, and no new algorithm comes with it; the limitation is the restriction to entropic geometry, the complementary non-entropic case being the province of the Fenchel-game program \cite{wang2024noregret}.

\subsection{Common structure across settings; the scaling-time question}

Every setting treated above instantiates the same abstract game, differing only in the observation structure (full, bandit, feedback graph, contextual), the action space (finite, continuous, policy class), and the sufficient statistic (raw loss, residual, estimated, signed margin).
The RCGF constraint always takes the same form $Q_t(\cdot)\le\beta_t$; what changes is the loss sequence the RCGF is evaluated on, and how that sequence relates to the true losses.
The same economy holds at the level of the protocol, where Table~\ref{tab:game-variants} compares the boxed games on the parts that vary.

\begin{table}[t]
\centering
\small
\setlength{\tabcolsep}{4pt}
\renewcommand{\arraystretch}{1.25}
\begin{tabular}{@{}llllll@{}}
\toprule
\textbf{Game} & \textbf{Scale set by} & \textbf{Nature's constraint} & \textbf{Move order} & \textbf{State score} & \textbf{Terminal payoff} \\
\midrule
Game~\ref{game:central} (base game) & schedule & per-round cap $q_t$ & learner first & true & $L_\Gamma(S_T)$ \\
Game~\ref{game:fixed} (fixed scale) & frozen & per-round cap $q_t$ & learner first & true & $L_\Gamma(S_T)$ \\
Game~\ref{game:cgf} (self-tuned) & learner & cumulative $V$ & learner first & true & $L_\Gamma(S_T)$ \\
Game~\ref{game:pressure} (pressure) & learner, via $b_t$ & none (made active) & nature first & true & $L_\Gamma(S_T)$ \\
Game~\ref{game:bandit} (bandit) & learner & estimated-RCGF cap & learner first & estimated & sampled-regret sup \\
Game~\ref{game:ts} (expected sampled) & frozen & per-round cap $q_t$ & learner first & true & expected sampled regret \\
\bottomrule
\end{tabular}
\caption{Every boxed variant is Game~\ref{game:central} with at most three parts changed. Columns are the parts that vary across the variants --- who sets the measurement scale, what constrains nature's move, the order of moves within a round, which score drives the state, and the terminal payoff; the base-game row gives Game~\ref{game:central}'s own values. }
\label{tab:game-variants}
\end{table}

The historical scaling-time question \cite{Freund2016} asks for quantile regret controlled by an internal variance-like process in place of horizon and expert count alone. We settle the fixed-budget PAC-Bayes version for the exact cumulant intrinsic time, with a logarithmic meta-controller giving simultaneous quantile guarantees. Simultaneous adaptivity in both comparator complexity and intrinsic time, for a single copy and up to only mild lower-order terms, remains open: NormalHedge together with contextual-bandit model-selection lower bounds \cite{Freund2016,MarinovZimmert2021} show that self-variance, quantiles, and parameter-freeness cannot all be demanded at once.
In the face of these impossibility results, the concentration game makes the tradeoffs visible.

\subsection{The identities in neighboring fields}
\label{sec:new-connections}

Several of the framework's identities read cleanly in the languages of neighboring areas: the retempering drift as a summation by parts of the log-partition's $\eta$-derivative, whose It\^o form turns the committed and reactive games into dual stochastic-control problems on the information manifold.
The Bellman potential as a negative free energy at inverse temperature $\eta$ with external field $S$.
In the bandit reading, the observation-noise terms of \eqref{eq:bandit-decomp} are the axis along which the algorithms trade off against one another.
EXP3-IX makes the estimation bias small and the martingale term slightly larger; Tsallis-INF makes the entropic information small and the sample-dependent drift larger (Appendix~\ref{app:neighboring-readings}).

The one-round ($T=1$) specialization of the fixed-scale game is the classical redundancy--capacity saddle point \cite{haussler1997mutual,MerhavFeder1998,ClarkeB94}: a channel's input distribution is nature's comparator prior, the log-partition value $W_\eta$ is the mixture codelength, and the Gibbs equalizer $\rho_{\eta,S}=\nabla W_\eta(S)$ is the capacity-achieving input, with capacity the saddle-point value $\mathsf C=\min_q\max_\theta\KL(P_\theta\|q)=\max_{p_0}I(p_0)$.
When an external process pins the input prior, leaving no freedom to equalize, the shortfall from capacity is the same terminal relative-entropy transport the horizon ledger isolates in $B_T(\nu)$, now read \emph{across channel inputs}.
It is computable term by term as a Bregman divergence generated by the log-partition potential (Appendix~\ref{app:neighboring-readings}).

An algorithm gains from better-behaved losses exactly the higher-cumulant tail $\kappa_{\ge3}$ that the variance and range proxies discard, and the relaxation remainder --- the heavy-tailed moves the game admits that a proxy forbids --- has the closed form of Proposition~\ref{prop:tower-remainder}.
A change of sufficient statistic $z\mapsto\phi(z)$ leaves every identity intact and shrinks that remainder insofar as it shrinks $\{\kappa_n(\phi(z))\}_{n\ge3}$. Bounded-influence transforms and optimism are both strategies to reduce the residual cumulants.

The Bellman equalizer is also the maximum-entropy distribution consistent with the exposed sufficient statistic (Proposition~\ref{prop:robust-bayes}), and the game-theoretic reading derives both properties from one minimax argument on the RCGF-constrained feasible set.
The coincidence holds well beyond the entropic case: for any strictly convex Legendre-type potential $W$ with conjugate entropy $-\Psi$, $\nabla W(S)$ is simultaneously the Bellman equalizer, since the budget-sphere increment is constant and nature is indifferent across directions, and the maximum-$(-\Psi)$-entropy distribution.
Both therefore hold across the regular exponential/Bregman family.
The Shore--Johnson axioms then single out the centered-RCGF member, pinning the inference rule uniquely; deriving that invariance from the equalizer postulate alone is beyond the present scope.

\subsection{Exact values and rates}
\label{sec:exact-values-rates}

These readings sharpen what remains open, and one quantity is the residual behind most of it: the \emph{exact value}.
The one round is solved on two slices, each strictly below the Bellman bound --- two experts at every comparator budget (Proposition~\ref{prop:k2-exact}), and every $K$ at a saturated one (Proposition~\ref{prop:saturated-exact}) --- and their common refinement is missing.
So is any $K$ over $T\ge 2$ rounds under per-round caps, beyond the edge terms $O(\Gamma+Q_T^*)$.
Under a cumulative budget the multi-round question closes at constant scale from the finite horizon $T_0$ on (Proposition~\ref{prop:horizon-free}), and varying the scale lowers the value by under a tenth of a percent at the budgets computed (Appendix~\ref{app:varying-scale-worth}).
The per-round-capped game remains open, its budget $V=T\beta$ growing with the horizon.
Even at $K=2$ with a uniform prior, where the terminal payoff is an exact support function, nature's reachable set solves a state-dependent recursion whose optimum over-tilts the Gibbs iterate, and no elementary closed form survives past $T=1$ (Appendix~\ref{app:exact-value-residual}).
The same quantity settles several others.
The partial-information value is the same function with the estimated per-round RCGF of Proposition~\ref{prop:bandit-exact-value} in place of the true one.
The pressure and second-order games differ at the horizon only through which terminal state each lets nature reach, their drifts agreeing by Proposition~\ref{prop:terminal-collapse}.
But the loss couples the rounds through the path-dependent $p_t$, so the optimal schedule is again this control problem and a matching lower bound for the surrogate is unknown.

Simultaneous adaptation is separate.
A single fixed-budget run is valid for every comparator simpler than its budget but not adaptive to it, losing a factor $\sqrt{\Gamma_{\max}/\Gamma}$ against true complexity $\Gamma$.
A dyadic meta-controller restores adaptivity at an additive $O(\log\log\log K)$ plus the controller's own intrinsic-time term, and a single copy attaining the ideal $\sqrt{\Gamma V_T}$ with no additive term for all complexities at once is ruled out by the model-selection lower bound \cite{MarinovZimmert2021,Freund2016}.
Enlarging the class to expert--scale pairs under a log-scale prior over $\eta$ --- the continuous-mixture device of \cite{koolen2015squint} and the parameter-free family \cite{orabona-pal} --- does achieve simultaneous adaptivity at a strictly-lower-order $O(\log\log V_T)$ overhead, leaving only the leading constant.
The exact value above also decides whether the mixture's per-round Jensen surplus preserves the envelope constant $2\sqrt2$.

Correlated-equilibrium rates rest on a different quantity.
An improved CCE rate for a structured game is equivalent to worst-player last-iterate stabilization, $\max_k V_T^{(k)}=o(T)$; the stabilizing sub-case is settled for potential, congestion, and polymatrix games under Gibbs-equalizer dynamics (Theorem~\ref{thm:unconditional-V}), where the interior-equilibrium sub-case fails for plain Gibbs and needs an optimistic schedule.
An exact growth-rate law for $V_T^{(k)}$, and a matching lower bound on the non-stabilizing polymatrix class, remain open.

\appendix
\appendixtitle
\section{Further results}
\label{app:gap-proofs}

Results whose full detail the main text relegates to the appendix are stated and discussed here, grouped by source section; the proofs of the numbered statements, the body's and this appendix's alike, are collected in Appendix~\ref{app:proofs}.

\subsection{Further results from Section~\ref{sec:primitive}}
\label{app:further-primitive}

\subsubsection{The exponential-family parameterization costs no generality}
\label{app:reparameterization}

Fix $\eta_t>0$ and a strictly positive prior. Adding $b\1$ to a cumulative score multiplies every numerator and the denominator alike by $e^{-\eta_t b}$, so two scores give the same $p_t$ in \eqref{eq:bayes-update} exactly when they differ by a constant vector; and every interior $p_t$ --- positive weight on every action --- is realized, by $C_{t-1}(i)=-\eta_t^{-1}\log\bigl(p_t(i)/\pi(i)\bigr)$.
Hence \eqref{eq:bayes-update} is a bijection from $\R^K/\R\1$ onto the interior of $\DeltaK$: it \emph{reparameterizes} arbitrary play instead of restricting it.
Read in the centered coordinates of Section~\ref{sec:primitive} the same inverse says that the state is the round's posterior-to-prior log-ratio at the round's own scale, $S_{t-1}=\eta_t^{-1}\log(p_t/\pi)$ up to an additive constant, and doubling the scale exactly halves the state needed for the same tilt.
The per-round move is its increment: at a constant scale $z_t=\eta^{-1}\log(p_{t+1}/p_t)$, the log-ratio of consecutive posteriors, again up to a constant.

\subsubsection{The exponential family traced by the inverse temperature}
\label{app:exp-family}

The family of posteriors parameterized by inverse temperature forms a one-dimensional exponential family
\begin{equation}\label{eq:exp-family}
\cE(\pi, C_{t-1})
=
\left\{q_\eta(i) = \frac{\pi(i)\,e^{-\eta C_{t-1}(i)}}
 {\sum_j \pi(j)\,e^{-\eta C_{t-1}(j)}}
 : \eta > 0 \right\}
\end{equation}
with natural parameter $\eta$, sufficient statistic $C_{t-1}$, log-partition $\Lambda(\eta)=\log\sum_i\pi(i)e^{-\eta C_{t-1}(i)}=-\eta A_{t-1}(\eta)$, and Fisher information $\Lambda''(\eta)=\Var_{q_\eta}(C_{t-1})$. The intrinsic-time density $Q_t(c)$ is the one-step analogue of that Fisher information---the curvature cost of incorporating $c_t$ at the current scale---whose running accumulation is the pathwise intrinsic-time clock adaptive Bayes tracks \cite{balsubramani2026adaptive}. In the SafeBayes tradition \cite{the2012safebayesian}, $p_t$ is the maximum-entropy distribution typical of the observed average loss, and $Q_t(c)$ measures how much the new observation forces the learner away from the prior in relative entropy.

\subsubsection{The moment expansion behind the RCGF}
\label{app:moment-expansion}

The choice of the RCGF as the budgeted quantity descends from the moment expansion of the MGF:
\begin{equation}\label{eq:mgf-coincidence}
\E_p[e^{\eta z}]
=
\sum_{n\ge 0}\frac{\eta^n}{n!}\E_p[z^n]
=
\sum_{n\ge 0}\frac{\eta^n}{n!}\sum_i p(i)\,z(i)^n
\end{equation}
where each moment $\E_p[z^n]=\sum_i p(i)\,z(i)^n$ is the coincidence rate of the coincidence-budget reading. The centered log-MGF $\eta^2 Q_\eta(p,z)$ packages these rates of all orders into a single scale-aware sufficient statistic, so the cap privileges no moment order. Variance ($n=2$) and range ($n=\infty$) are individual moment-bounds obtained by truncating \eqref{eq:mgf-coincidence}, which is why the variance and range games appear as derived relaxations of the primitive (Section~\ref{sec:hierarchy}).

The centered RCGF admits three readings, each at its own rescaling. As a \emph{coincidence budget}, $\eta^2 Q_\eta(p,z)=\log\E_p[e^{\eta z}]$ is the logged generating function of the moments $\E_p[z^n]$, summed at scale $\eta$; each moment is the $z$-weighted rate at which $n$ draws coincide on a symbol of $p$ \cite{balsubramani2026information}. As a \emph{min-max value}, $\eta Q_\eta(p,z)+\eta^{-1}\KL(\nu\|p)$ is the robust-Bayes value against nature's tilt $z$ (Propositions~\ref{prop:robust-bayes},~\ref{prop:terminal-dual}): the centered RCGF is the part of the round's value coming from the constraint on nature's move, the relative-entropy term the part from the comparator class. And as a \emph{relative-entropy increment}, the forward identity $\eta^2 Q_\eta(p,z)=\KL(p\|p^+)$ makes $\eta Q_\eta$ the round's prior-to-posterior information gain divided by $\eta$, with cumulative form $\sum_t\eta_t Q_t=\sum_t\KL(p_t\|p_{t+1})/\eta_t$. These are faces of one functional, each useful in a different proof.

\subsubsection{Influence of the centered CGF under a contaminated play}
\label{app:rcgf-influence}

The tilted-variance representation of Proposition~\ref{prop:tilted-var} also controls how the primitive responds to a perturbation of the mixed action, with the centering recomputed at the perturbed action.

\begin{corollary}[Influence of the centered CGF]
\label{cor:rcgf-influence}
Fix $c\in\R^K$, a strictly positive $p\in\DeltaK$, and $\eta>0$; write $z:=c-\inner{p}{c}$ componentwise, and let $p_1:=p_1^{(\eta,z)}$ be the endpoint of the tilt path of Proposition~\ref{prop:tilted-var}.
For $j\in[K]$ let $\delta_j$ be the point mass on $j$ and $p_\epsilon:=(1-\epsilon)p+\epsilon\,\delta_j$.
Then
\begin{equation}\label{eq:rcgf-influence}
\frac{d}{d\epsilon}\Big|_{\epsilon=0^+}\eta^2 Q_\eta\bigl(p_\epsilon,\,c-\inner{p_\epsilon}{c}\bigr)
=\xi_\eta(j)
:=\frac{p_1(j)}{p(j)}-1-\eta z(j)
\end{equation}
with $\E_{j\sim p}[\xi_\eta]=0$ and
\begin{equation}\label{eq:rcgf-influence-var}
\Var_{j\sim p}(\xi_\eta)
=\chi^2(p_1\|p)-2\eta\inner{p_1}{z}+\eta^2\Var_p(z),
\qquad
\chi^2(p_1\|p):=\sum_{j}p(j)\Bigl(\frac{p_1(j)}{p(j)}-1\Bigr)^{2}
\end{equation}
If $z(i)\in[a,b]$ for all~$i$, then $-(1+\eta b)\le\xi_\eta(j)\le e^{\eta b}-1-\eta a$ for every~$j$.
\end{corollary}

The body reads off the two-sided bracket.
The influence is bounded below by the linear term alone and above by the exponential of the largest centered score, so it is linear in the depth of the score range and exponential in its height, and the number of experts does not enter.

\subsection{Further results from Section~\ref{sec:fixed}}
\label{app:proofs-fixed}

\subsubsection{Why the value-to-go is a certificate}
\label{app:bellman-certificate}

The value-to-go $U_t(S)=\Gamma/\eta+W_\eta(S)+\eta\sum_{s\ge t}q_s$ of \eqref{eq:fixed-bellman} is an \emph{admissible Bellman relaxation}: writing $\Val_t^{\eta,\Gamma}(S)$ for the value of Game~\ref{game:fixed} from state $S$ with rounds $t,\dots,T$ still to play, it satisfies $U_t(S)\ge \Val_t^{\eta,\Gamma}(S)$.
At the equalizer play the round decrement $U_t(S)-U_{t+1}(S+z)=\eta q_t-\eta\,Q_\eta(p,z)$ is nonnegative on the feasible set and vanishes at nature's budget-saturating reply, so the Bayes normalizer $W_\eta$ is the Bellman potential of the constrained game.
Those two properties make $U_t$ a certificate.
At the horizon $U_{T+1}(S)=\Gamma/\eta+W_\eta(S)$ already dominates the terminal payoff, since \eqref{eq:terminal-dual} writes $L_\Gamma(S)$ as an infimum over scales of which that quantity is one member;
and a value-to-go that starts above the payoff and never rises along a feasible round stays above the game value at every earlier round.
Verifying the relaxation therefore requires only one round's inequality, which every round inherits; computing the value outright would require the recursion.

The bound $\Val_t^{\eta,\Gamma}(S)\le U_t(S)$ becomes an equality only when, in addition, the fixed scale $\eta$ coincides with the dual optimizer of $L_\Gamma$ at the realized terminal state --- a condition the learner generally cannot guarantee in advance.
Identity \eqref{eq:terminal-dual} locates that gap: at the horizon the relaxation yields $\Gamma/\eta+W_\eta(S_T)$, where the game yields $L_\Gamma(S_T)=\inf_{\eta'>0}\{\Gamma/\eta'+W_{\eta'}(S_T)\}$, so a scale committed before $S_T$ is seen loses exactly how far that infimum sits below the committed scale's value.

\subsubsection{The two proxies do not refine each other}
\label{app:proxy-order}

At $p=(1/2,1/2)$ and $c=(0,1)$ the variance proxy $(e-2)\Var_p(c)\approx 0.18$ exceeds the range proxy $\range^2/8=0.125$; at concentrated $p$, where $\Var_p(c)$ collapses but $\range(c)$ remains $1$, the order reverses. Viewed as constraint sets on nature, $\{(e-2)\Var_{p_t}(c)\le\beta\}$ and $\{\range(c)^2\le 8\beta\}$ are both contained in $\{Q_{\eta_t}(p_t,c)\le\beta\}$ by \eqref{eq:hierarchy}, and neither contains the other.

At fixed rate~$\eta$ and under a bounded range $c_t(i)\in[a_t,b_t]$, the classical Hedge bound becomes an exact identity once the relaxation remainder is retained:
\begin{equation}\label{eq:classical-remainder}
R_T^c(\nu)
=
\frac{\KL(\nu\|\pi)}{\eta}
+
\frac{\eta}{8}\sum_{t=1}^T (b_t-a_t)^2
-
\Delta_T^{\mathrm{class}}(\nu,\eta)
\end{equation}
where $\Delta_T^{\mathrm{class}}(\nu,\eta)
=
\eta^{-1}\KL(\nu\|p_{T+1})
+
\eta\sum_{t=1}^T\bigl[(b_t-a_t)^2/8-Q_t(c)\bigr]\ge 0$.
The first piece is terminal posterior mismatch; the second is the cumulative gap between the Hoeffding proxy and the true RCGF increment.

\subsubsection{The Fenchel dual: from constrained to regularized}
\label{app:fenchel}

The same game has a dual regularized description.
For the regularizer $R_\eta(p)=\eta^{-1}\KL(p\|\pi)$ on~$\DeltaK$, the FTRL update
\begin{equation}\label{eq:ftrl}
p_t
=
\argmin_{p\in\DeltaK}
\left\{
\inner{p}{C_{t-1}}+\frac{1}{\eta_t}\KL(p\|\pi)
\right\}
\end{equation}
yields $p_t(i)\propto \pi(i)e^{-\eta_t C_{t-1}(i)}$. The Fenchel conjugate is the scale-$\eta$ log-partition
\begin{equation}\label{eq:fenchel-conjugate}
R_\eta^*(\ell)
:= \sup_{p\in\DeltaK}
 \left\{ \inner{p}{\ell} - \eta^{-1}\KL(p\|\pi) \right\}
= \eta^{-1}\log \sum_{i=1}^K \pi(i)\,e^{\eta\,\ell(i)}
\end{equation}
The global constraint set under the prior is
\begin{equation}\label{eq:global-constraint}
\cC^{\mathrm{global}}(\eta, r)
=
\left\{
\ell\in\R^K:
\eta^{-1}\log\sum_i \pi(i)e^{\eta \ell(i)}\le r
\right\}
\end{equation}
which constrains the log-partition function of~$\ell$ under the prior at inverse temperature~$\eta$, independently of the learner's current distribution~$p_t$.

The global constraint \eqref{eq:global-constraint} is the constraint set matched to the prior-anchored FTRL update.
But the one-round identity \eqref{eq:fixed-transport} operates locally, via the centered mixability gap---the expected-minus-mix-loss difference $\inner{p_t}{c_t}-m_t(\eta_t)=\eta_t Q_t(c)$---computed relative to the current distribution~$p_t$.
This motivates a distribution-dependent constraint.

Given $p_t\in\DeltaK$, $\eta_t>0$, and budget $\beta>0$, define
\begin{equation}\label{eq:local-cgf}
\cC_t^{\mathrm{RCGF}}(p_t,\eta_t,\beta)
:=
\left\{
c\in\R^K:
\frac{1}{\eta_t^2}\psi_{p_t}(-\eta_t;c)\le \beta
\right\}
\end{equation}

\begin{proposition}[Geometry of the local RCGF constraint]
\label{prop:cgf-geometry}
For fixed $(p_t,\eta_t,\beta)$:
\begin{enumerate}[label=(\alph*)]
\item The set $\cC_t^{\mathrm{RCGF}}$ is closed and convex (sublevel set of a log-sum-exp).
\item If $c\in\cC_t^{\mathrm{RCGF}}$, then $c+b\1\in\cC_t^{\mathrm{RCGF}}$ for all $b\in\R$.
\item All constant vectors lie in~$\cC_t^{\mathrm{RCGF}}$.
\item $\cC_t^{\mathrm{RCGF}}$ is dependent on~$p_t$, but not directly on~$\pi$.
\end{enumerate}
Property~(b) is the game-theoretic statement that only relative expert performance matters.
\end{proposition}

The global constraint \eqref{eq:global-constraint} uses the prior and cumulative loss as base measure.
The local constraint \eqref{eq:local-cgf} uses the current posterior and the one-round loss.
They are related because $p_t$ is the FTRL output applied to~$C_{t-1}$.
The local constraint is the conditional version: it constrains the one-round loss given the current state.

\subsection{The exact one-round value}
\label{app:one-round-exact}

The Bellman bound of Corollary~\ref{cor:fixed-budget} is strict at the single round treated in this section. It is an upper bound because it fixes the learner's side of $\inf_p\sup_z$: any particular play, evaluated against nature's best reply, is an upper bound on the value. The complementary certificate fixes nature's side. Since nature's feasible set $\cZ_{\eta,q}(p)$ moves with $p$, nature's strategy is a \emph{response rule} $\zeta$ assigning a feasible move to each learner action, and every such rule bounds the value from below.

\begin{proposition}[Nature-side lower bound]
\label{prop:nature-lb}
Fix $\eta>0$, $q\ge0$, $\Gamma\ge0$ and a state $S$, write $\mathrm{Val}_1(S;\eta,q,\Gamma)$ for the one-round value from $S$ (so that the $\mathrm{Val}_1(\eta,q,\Gamma)$ of Propositions~\ref{prop:saturated-exact} and~\ref{prop:k2-exact} is $\mathrm{Val}_1(0;\eta,q,\Gamma)$), and let $\zeta:\DeltaK\to\R^K$ satisfy $\zeta(p)\in\cZ_{\eta,q}(p)$ for every $p\in\DeltaK$. Then
\begin{equation}\label{eq:nature-lb}
\inf_{p\in\DeltaK} L_\Gamma\bigl(S+\zeta(p)\bigr)
\;\le\;
\mathrm{Val}_1(S;\eta,q,\Gamma)
:=
\inf_{p\in\DeltaK}\ \sup_{z\in\cZ_{\eta,q}(p)} L_\Gamma(S+z)
\end{equation}
with equality whenever $\zeta(p)$ attains the inner supremum at every $p$. Fixing instead a single comparator $\nu$ with $\KL(\nu\|\pi)\le\Gamma$ and replacing $L_\Gamma$ by $\inner{\nu}{\cdot}$ weakens \eqref{eq:nature-lb} further but keeps it valid.
\end{proposition}

The two certificates are not symmetric in difficulty. Playing Gibbs and reading off the potential increment gives the closed-form upper bound $\Gamma/\eta+\eta q$ with no case analysis, whereas a response rule must be specified at every $p$. The rule that closes the gap attacks the action the learner has left least protected, and it is exactly optimal once the comparator budget is large enough to reach every action --- the regime in which the budget stops mattering at all.

\begin{proposition}[Saturated-budget reduction to a parameter-free game]
\label{prop:saturated-reduction}
Let $\Gamma_{\max}:=\max_i\log(1/\pi(i))$ (so $\Gamma_{\max}=\log K$ for the uniform prior). For every $\Gamma\ge\Gamma_{\max}$ the relative-entropy ball $\{\nu:\KL(\nu\|\pi)\le\Gamma\}$ is the whole simplex $\DeltaK$, and
\begin{equation}\label{eq:Lmax}
L_\Gamma(S)=\max_i S(i)\qquad(\Gamma\ge\Gamma_{\max})
\end{equation}
Consequently, for every $\eta>0$, horizon $T$, and budget profile $(q_t)$, the value $\Val_1^{\eta,\Gamma}(0)$ is independent of $\Gamma$ on $[\Gamma_{\max},\infty)$ and equals the minimax value of the RCGF-budgeted game with terminal payoff $\max_i S_T(i)$: the comparator-complexity dependence of the value saturates at $\Gamma=\Gamma_{\max}$.
\end{proposition}

At $T=1$ that parameter-free game is solved in closed form.

\begin{proposition}[Exact one-round value at a saturated comparator budget]
\label{prop:saturated-exact}
Let $\Gamma\ge\Gamma_{\max}:=\max_i\log(1/\pi(i))$, so that $L_\Gamma=\max_i(\cdot)$ by Proposition~\ref{prop:saturated-reduction}. For $\eta>0$, $q>0$, $T=1$ and $S_0=0$, the value of the fixed-scale game is
\begin{equation}\label{eq:saturated-exact}
\mathrm{Val}_1(\eta,q,\Gamma)=a_K(\eta,q)
\end{equation}
where $a_K(\eta,q)$ is the unique $a>0$ solving
\begin{equation}\label{eq:aK}
\frac{1}{K}e^{\eta a}+\frac{K-1}{K}\,e^{-\eta a/(K-1)}=e^{\eta^2 q}
\end{equation}
The value is attained by the uniform learner play $p=u_K$ and by nature's move placing $a_K$ on one action and $-a_K/(K-1)$ on the rest. It does not depend on $\pi$.
\end{proposition}

At $K=2$ Equation~\eqref{eq:aK} is $\cosh(\eta a)=e^{\eta^2 q}$, so $a_2=\eta^{-1}\operatorname{arcosh}(e^{\eta^2 q})$ and \eqref{eq:saturated-exact} reproduces Proposition~\ref{prop:k2-exact} at $g(\Gamma)=1$. The general-$K$ value is again strictly below the Bellman envelope: for the uniform prior at $\eta=1$, $q=0.1$, $K=5$ it is $0.8178$ against the tightest envelope $\Gamma_{\max}/\eta+\eta q=\log 5+0.1\approx 1.7094$ available in this regime. Two features of \eqref{eq:aK} carry over from the two-expert case: the budget enters only through $e^{\eta^2q}$, and $a_K=\sqrt{2(K-1)q}+O(q)$ as $q\downarrow0$, the small-step regime in which the per-round loss is second order in nature's step. Saturating the comparator budget removes $\pi$ and $\Gamma$ from the answer; for $\Gamma<\Gamma_{\max}$ the optimal play moves off the uniform toward the prior --- already visible at $K=2$, where it sits strictly between the two --- and no comparable closed form is known beyond $K=2$ (Section~\ref{sec:discussion}).

For a comparator budget below saturation, the two-expert instance is solved exactly.
Fix $K=2$, $\pi=(\tfrac12,\tfrac12)$, and a single round from $S=0$.
Define the comparator concentration coefficient $g(\Gamma)\in[0,1]$ as the unique $g\in[0,1]$ solving
\begin{equation}\label{eq:k2-g}
\tfrac12\bigl[(1+g)\log(1+g)+(1-g)\log(1-g)\bigr]=\Gamma
\end{equation}
capped at $g(\Gamma)=1$ once $\Gamma\ge\log 2$.
The left-hand side is the relative entropy of $\bigl(\tfrac{1+g}{2},\tfrac{1-g}{2}\bigr)$ from the uniform prior, so $g(\Gamma)$ measures how far a comparator of budget $\Gamma$ can move off the uniform along a single coordinate. The limits are $g(\Gamma)\sim\sqrt{2\Gamma}$ as $\Gamma\to0^+$ and $g(\log 2)=1$.

\begin{proposition}[Exact two-expert one-round value]
\label{prop:k2-exact}
For $K=2$, $\pi=(\tfrac12,\tfrac12)$, $\eta>0$, $q\ge 0$, $\Gamma\ge 0$, and $S=0$, the one-round game has value
\begin{equation}\label{eq:k2-value}
\mathrm{Val}_1(\eta,q,\Gamma)
=
g(\Gamma)\,\frac{\operatorname{arcosh}\!\bigl(e^{\eta^2 q}\bigr)}{\eta}
\end{equation}
attained by the learner's Gibbs play $p=(\tfrac12,\tfrac12)$ and nature's budget-saturating move $z^\star=(\delta^\star,-\delta^\star)$, $\delta^\star=\eta^{-1}\operatorname{arcosh}(e^{\eta^2 q})$.
For $\Gamma>0$ and $q>0$ this value satisfies the fixed-scale Bellman bound
\begin{equation}\label{eq:k2-gap}
\mathrm{Val}_1(\eta,q,\Gamma)\le\frac{\Gamma}{\eta}+\eta q
\end{equation}
with strict inequality at every finite $\eta>0$ when $\Gamma\ge\log 2$; when $\Gamma<\log 2$ the inequality is strict except at the single scale $\eta^2 q=-\tfrac12\log\bigl(1-g(\Gamma)^2\bigr)$, where equality holds.
At the Bellman-optimal scale $\eta_{\mathrm B}=\sqrt{\Gamma/q}$ the exact value is a fixed multiple of the Bellman envelope $2\sqrt{\Gamma q}$,
\begin{equation}\label{eq:k2-kappa}
\mathrm{Val}_1(\eta_{\mathrm B},q,\Gamma)=\kappa(\Gamma)\cdot 2\sqrt{\Gamma q}
\qquad\text{with}\qquad
\kappa(\Gamma):=\frac{g(\Gamma)\,\operatorname{arcosh}(e^{\Gamma})}{2\Gamma}\in(0,1)
\end{equation}
where the gap factor $\kappa(\Gamma)$ depends on $\Gamma$ alone (not on $q$), is decreasing in $\Gamma$, and tends to $1$ as $\Gamma\to0^+$.
\end{proposition}

At $\eta=1$, $q=0.1$, $\Gamma=\log 2$ this gives $\delta^\star=\operatorname{arcosh}(e^{0.1})\approx 0.4547$, $g(\log 2)=1$, so $\mathrm{Val}_1\approx 0.4547$ against the Bellman bound $\log 2+0.1\approx 0.7931$.
At the optimal scale the gap factor is $\kappa(\log 2)\approx 0.95$: the envelope $2\sqrt{\Gamma q}$ overshoots the truth by about $5\%$ for the full simplex, and by more as the comparator budget grows ($\kappa(1)\approx 0.83$, $\kappa(2)\approx 0.67$), while becoming exact in the low-complexity limit $\Gamma\to0^+$.
For $\Gamma\ge\log 2$ the gap shrinks to $0$ as $\eta\to\infty$ (the bound is asymptotically tight at large scale), whereas for $\Gamma<\log 2$ the gap is smallest at a finite scale near $\eta_{\mathrm B}$ and grows like $(1-g(\Gamma))\eta q$ as $\eta\to\infty$.
Carrying the same instance to a second round shows where the multi-round difficulty enters: from the skewed post-round-1 state the exact minimax play over-tilts relative to the Gibbs iterate, so the single-round product form is lost (the computation appears below). Two features of this instance keep it hard, and dropping either one settles it: nature is capped round by round, and the learner may play off the Gibbs trajectory. Replace the per-round caps by one cumulative budget and hold the learner to Gibbs at a constant scale, as Game~\ref{game:cgf} does, and the multi-round value is closed form from a finite horizon on (Proposition~\ref{prop:horizon-free}).

One further structural fact localizes the gap. At every fixed scale the Bellman bound $L_\Gamma(S)\le\Gamma/\eta+W_\eta(S)$ of Proposition~\ref{prop:terminal-dual} sharpens to an exact pointwise identity whose slack carries the whole obstruction:
\begin{equation}\label{eq:pointwise-align}
L_\Gamma(S)=W_\eta(S)+\frac{1}{\eta}\,A_\Gamma(\rho_{\eta,S}),
\qquad
A_\Gamma(p):=\sup_{\nu:\,\KL(\nu\|\pi)\le\Gamma}\bigl[\KL(\nu\|\pi)-\KL(\nu\|p)\bigr]\le\Gamma
\end{equation}
a consequence of the Gibbs variational identity \eqref{eq:gibbs-var}: the algebra $\KL(\nu\|\rho_{\eta,S})=\KL(\nu\|\pi)-\eta\inner{\nu}{S}+\eta W_\eta(S)$ rearranges to $\inner{\nu}{S}=W_\eta(S)+\eta^{-1}[\KL(\nu\|\pi)-\KL(\nu\|\rho_{\eta,S})]$, and taking the supremum over the relative-entropy ball isolates $A_\Gamma(\rho_{\eta,S})$. The slack in the Bellman bound is thus exactly $(\Gamma-A_\Gamma(\rho_{\eta,S}))/\eta\ge 0$, vanishing when the Gibbs posterior $\rho_{\eta,S}$ meets the worst comparator on the relative-entropy sphere $\{\KL(\nu\|\pi)=\Gamma\}$, and the identity confines the whole gap to the single bounded terminal functional $A_\Gamma(\rho_{\eta,S_T})$.
For $T=1$ from $S_0=0$ under a uniform prior the learner's Gibbs play $\rho_{\eta,0}=\pi$ is minimax and the alignment is closed form: at $K=2$, $\pi=(\tfrac12,\tfrac12)$, nature's saturating move gives $A_\Gamma(\rho_{\eta,S_1})=g(\Gamma)\operatorname{arcosh}(e^{\eta^2 q})-\eta^2 q$, recovering Proposition~\ref{prop:k2-exact}.
That statement requires the prior to be uniform. Under a non-uniform prior the Gibbs play at $S_0=0$ is $\pi$ itself, which concedes the response rule of Proposition~\ref{prop:saturated-exact} a payoff governed by $\min_i\pi(i)<1/K$ and so is strictly beaten by the uniform play: at $\Gamma\ge\Gamma_{\max}$, $K=3$, $\eta=1$, $q=0.1$ and $\pi=(0.6,0.3,0.1)$ the Gibbs play concedes $1.1381$ against the value $a_3=0.6108$.
The Gibbs play is the equalizer at every state; whether it is also minimax depends on the prior and on the comparator budget, and the two can part already at $T=1$.
For $T\ge 2$ the minimax learner over-tilts relative to the fixed-scale Gibbs iterate, so the telescoped loss $W_\eta(S_T)=\eta\sum_t Q_t$ no longer holds at the optimal play and the terminal functional ceases to factor across rounds. Numerically ($\eta=1$, per-round budget $q=0.1$, $\Gamma\ge\log 2$) the exact value at $T=2$ equals $0.6486$, strictly below both product forms $2\eta^{-1}\operatorname{arcosh}(e^{\eta^2 q})=0.9094$ and $\eta^{-1}\operatorname{arcosh}(e^{2\eta^2 q})=0.6537$, so no elementary closed form in $(\eta,q,\Gamma,T)$ survives.

The mechanics of the same instance ($\eta=1$, per-round budget $q_t=0.1$, $\Gamma=\log 2$) make the obstruction concrete.
Round~1 is the one-round game of Proposition~\ref{prop:k2-exact}: from $p_1=(\tfrac12,\tfrac12)$ nature saturates with the symmetric move $z_1=(\delta^\star,-\delta^\star)$, whose centered RCGF $Q_1=\eta^{-2}\log\cosh(\eta\delta^\star)=q_1$ fixes $\delta^\star=\operatorname{arcosh}(e^{0.1})\approx 0.4547$, and the state moves to the skewed $S_1=(\delta^\star,-\delta^\star)$.
At round~2 the Gibbs play is no longer uniform, $p_2=G_\eta(S_1)\approx(0.713,0.287)$, so nature's budget-saturating move---and the worst comparator direction it aligns with---becomes state-dependent: the symmetric $\operatorname{arcosh}$ form of round~1 gives way to the inverse of the Bernoulli$(p_2)$ centered RCGF, an asymmetric $z_2\approx(0.310,-0.771)$ (still centered, $\inner{p_2}{z_2}=0$) saturating the same budget $Q_2=q_2=0.1$.
The intrinsic time accumulates to $V_2=Q_1+Q_2=0.2$ and the intrinsic-time loss to $\sum_t\eta_t Q_t=0.2$; the two increments match only because the budgets do; the moves that saturate them differ.
From the skewed state the exact minimax play parts from the Gibbs trajectory---the play that exactly minimizes the terminal payoff $L_\Gamma(S_2)$ over-tilts relative to $p_2$, so the Gibbs trajectory upper-bounds the value. Pinning the exact value down remains open (Section~\ref{sec:discussion}).

\subsection{Further results from Section~\ref{sec:variable}}
\label{app:further-variable}

\subsubsection{The pressure game and its active scale}
\label{app:pressure-game}

\begin{proposition}[Existence and uniqueness of the active scale]
\label{prop:active-scale}
Fix $p\in\DeltaK$ and centered $z\in\R^K$ that is not $p$-a.s.\ constant.
For every target $b\in(0,\max_{i\in\supp(p)} z(i))$ there exists a unique $\eta>0$ such that
\begin{equation}\label{eq:active-scale}
\log\sum_{i=1}^K p(i)e^{\eta z(i)}=\eta b
\end{equation}
\end{proposition}

\begin{game}[Active-constraint pressure game]
\label{game:pressure}
Game~\ref{game:central} with three parts changed.
Move order: nature reveals its centered move $z_t$ ($\inner{p_t}{z_t}=0$) first, and the learner responds.
Nature's constraint: no cap is imposed;
instead the learner chooses a pressure target $b_t\in(0,\max_i z_t(i))$ and solves \eqref{eq:active-scale} for the unique $\eta_t>0$ that makes the one-step constraint active at $b_t$.
Learner's move: the update is pinned to the Gibbs iterate $p_{t+1}(i)\propto p_t(i)e^{\eta_t z_t(i)}$.
The terminal payoff is unchanged:
the comparator-worst loss along any trajectory is again $L_\Gamma(S_T)$ (Proposition~\ref{prop:terminal-collapse}).
\end{game}

\subsubsection{The chain characterization behind horizon-freeness}
\label{app:chain-value}

Under a cumulative budget the whole run reduces to a walk of the Gibbs iterate: the transport identity $\eta^2Q_\eta(p,z)=\KL(p\|p^+)$ makes each round's loss the relative entropy between consecutive iterates.
Nature's play is then a chain $p_1=\pi,p_2,\dots,p_{T+1}$ in $\DeltaK$ whose total length, measured by these per-step relative entropies, is capped by the budget.
With $A_\Gamma$ as in Section~\ref{sec:variable} --- the terminal functional of the alignment identity \eqref{eq:pointwise-align} --- and $p_t=G_\eta(S_{t-1})$, the value has the closed chain form below; it is established alongside Proposition~\ref{prop:horizon-free} in Appendix~\ref{app:proofs-variable}.
\begin{equation}\label{eq:chain-value}
\mathrm{Val}_T^\eta(\Gamma,V)
=
\frac{1}{\eta}\,\sup\Bigl\{\sum_{t=1}^T\KL(p_t\|p_{t+1})+A_\Gamma(p_{T+1})\Bigr\}
\end{equation}
the supremum over chains $p_1=\pi,\dots,p_{T+1}$ of full support with $\sum_{t=1}^T\KL(p_t\|p_{t+1})\le\eta^2V$.
It equals the Bellman value $\Gamma/\eta+\eta V$ if and only if some such chain uses the budget in full and ends on the comparator sphere $\{p:\KL(p\|\pi)=\Gamma\}$.
The chain totals reaching that sphere in $T$ steps form $\{\KL(\pi\|p):\KL(p\|\pi)=\Gamma\}$ at $T=1$ and an interval $[c_T,\infty)$ for $T\ge2$, with $c_T$ nonincreasing in $T$ and $c_T\to0$.
That shrinking interval produces the finite $T_0$ of \eqref{eq:horizon-free-value}.

The two requirements are coupled rigidly at one round: the single move that lands on the sphere also fixes the amount used at $\KL(\pi\|p_2)$, the relative entropy from the prior to the comparator taken in the reverse direction, and the value falls short by that mismatch.
A second round decouples them, since the intermediate point can be placed to use any amount at or above $c_2$.

\subsubsection{The two-step gate in closed form}
\label{app:two-step-gate}

At the Bellman-optimal scale $\eta=\sqrt{\Gamma/V}$ the gate of Proposition~\ref{prop:horizon-free} reads $\Gamma\ge c_2$, and at $K=2$ under a uniform prior $c_2$ is closed form. In the tilt coordinate, with the sphere at $w=\operatorname{arctanh}g(\Gamma)$ (the comparator concentration coefficient $g(\Gamma)\in[0,1]$ of Proposition~\ref{prop:k2-exact}), a two-step chain to it costs $\log\cosh w-\tanh(w_1)(w-w_1)$, stationary where $\tfrac12\sinh(2w_1)+w_1=w$; that left side rises strictly from $0$, so the root $w_1$ is unique and
\begin{equation}\label{eq:c2-closed}
c_2=\log\cosh w-\sinh^2(w_1)
\end{equation}
Both ends of the range follow. As $\Gamma\downarrow0$ the root tends to $w/2$ and $c_2\to w^2/4\to\Gamma/2$, so the gate holds and $T_0=2$. As $\Gamma\uparrow\Gamma_{\max}=\log2$ the comparator sphere recedes toward the boundary of the simplex. Then $w\to\infty$ and $\sinh^2(w_1)\sim w-w_1$, so $c_2\sim\tfrac12\log w$ diverges while $\Gamma$ stays below $\log2$: the gate fails and $T_0>2$. The regimes meet at $\Gamma=0.6479=0.935\,\Gamma_{\max}$, where $T_0$ steps from $2$ to $3$. The gap factor $\kappa(\Gamma)$ of Proposition~\ref{prop:k2-exact} does not persist: measured at that same scale, the $T=1$ value is $0.06\%$ below $2\sqrt{\Gamma V}$ at $\Gamma=0.1$ and $2.0\%$ below at $\Gamma=1/2$, and $0\%$ below at every $T\ge T_0$.

\subsubsection{What a varying scale is worth}
\label{app:varying-scale-worth}

Letting the scale vary lowers the value by very little. At $K=2$ and $T=2$ the minimax value over predictable schedules can be computed without discretizing nature at all. Its reply maximizes a convex function of its move over an interval, so the maximum sits at an endpoint, and the two endpoints solve a scalar equation. That computation puts the best predictable schedule below $2\sqrt{\Gamma V}$ by a relative $6\times10^{-4}$ at $\Gamma=1/2$, and by under $10^{-6}$ at $\Gamma=1/5$. Both figures survive local refinement of nature's opening move, and the first grows slightly as the scale grid is refined, so they bound from below what a schedule is worth. Retuning the scale therefore does lower the value, by less than a tenth of a percent at the budgets computed, and \eqref{eq:horizon-free-bound} over-estimates the value of Game~\ref{game:cgf} by no more than that.

Adaptation nevertheless has to be bounded to be safe. A rule placing the equalizer at the current state, $\eta_t\propto1/\|S_{t-1}\|_\infty$, hands nature the scale: an opening move of size $\epsilon$ forces a scale of order $1/\epsilon$, at which nature converts its remaining budget to regret at rate $\eta_t$, so the value grows without bound as $\epsilon\downarrow0$. Clipping the scale at $\eta_t=1$, as the schedule \eqref{eq:schedule} does, rules out exactly this.

\subsubsection{The intrinsic-time loss as a Riemann sum of a falling scale}
\label{app:riemann-loss}

Writing $\phi(v):=\min\{1,\,C\sqrt{\Gamma/v}\}$ for the per-unit scale (with $\phi(0):=1$), the schedule \eqref{eq:schedule} evaluates round $t$ at the left endpoint, $\eta_t=\phi(V_{t-1})$, so
\[
\sum_{t=1}^T\eta_t Q_t(c)=\sum_{t=1}^T\phi(V_{t-1})\,\bigl(V_t-V_{t-1}\bigr),
\]
the left-endpoint sum of the nonincreasing $\phi$ against the increments of the intrinsic time. A nonincreasing scale evaluated at each interval's left endpoint overestimates on every interval, so the sum lies above the integral $\int_0^{V_T}\phi(v)\,dv$---which equals $2C\sqrt{\Gamma V_T}-C^2\Gamma$ once $V_T\ge C^2\Gamma$---and exceeds it by at most the largest single increment, $Q_T^*(c)$. That bracketing is Theorem~\ref{thm:sqrt-envelope}, whose two sides differ by the discretization cost of sum against integral, lower-order whenever no single round carries a macroscopic share of $V_T(c)$.

The name of the schedule is earned in the unclipped regime: the marginal loss per unit of intrinsic time is the scale $\eta_t=C\sqrt{\Gamma/V_{t-1}}$; the transport bound $\Gamma/\eta_t$ at that scale is $C^{-1}\sqrt{\Gamma V_{t-1}}$, so the marginal transport term is $\tfrac{d}{dv}\bigl[C^{-1}\sqrt{\Gamma v}\bigr]=\tfrac12 C^{-1}\sqrt{\Gamma/v}$ at $v=V_{t-1}$. The schedule keeps their ratio pinned at $2C^2$ along the whole path---equal round by round at $C=1/\sqrt2$.

\subsection{Classical bounds from the game}
\label{app:four-classical}
Each row of Table~\ref{tab:event-restricted} fixes a triple $(\pi,S,A)$ and enters the shared three-step argument of Section~\ref{sec:event-games} at a different point: the comparator-class dual (Proposition~\ref{prop:terminal-dual}) evaluates the variational form, Proposition~\ref{prop:event-restricted} closes it on a face, and the scale optimization converts the value into a relative-entropy rate.
Each row still owes the classical literature its closing step.
The derivations below carry each row from the game to the classical statement on the finite alphabet, closing step included; throughout, $X_1,\dots,X_n$ are i.i.d.\ with law $P$ of full support on $[K]$, and $c:[K]\to\R$ is a score with $\E_P[c]<t<\max_i c(i)$.

\subsubsection{Chernoff}
\label{app:classical-chernoff}

Take $\pi=P$ and state $c$, so that the scale-$\lambda$ log-partition of the game is the cumulant generating function $\Lambda(\lambda):=\lambda W_\lambda(c)=\log\E_P[e^{\lambda c}]$, and write $\bar X_n:=n^{-1}\sum_{j=1}^n c(X_j)$.
The closing step is exponential Markov.
For every $\lambda\ge0$, monotonicity of $x\mapsto e^{n\lambda x}$ and independence give
\begin{equation}\label{eq:app-chernoff-markov}
\Pr[\bar X_n\ge t]
\;\le\;
e^{-n\lambda t}\,\E_P\Bigl[e^{\lambda\sum_j c(X_j)}\Bigr]
=
e^{-n\lambda t}\,\E_P\bigl[e^{\lambda c}\bigr]^n
=
e^{-n(\lambda t-\Lambda(\lambda))}
\end{equation}
and optimizing the exponent over $\lambda\ge0$ yields $\Pr[\bar X_n\ge t]\le e^{-n\Lambda^*(t)}$ with $\Lambda^*(t)=\sup_{\lambda\ge0}\{\lambda t-\Lambda(\lambda)\}$.
The game supplies the meaning of the exponent: by the Lagrange duality of Proposition~\ref{prop:terminal-dual},
\begin{equation}\label{eq:app-chernoff-rate}
\Lambda^*(t)=\inf\bigl\{\KL(\nu\|P):\nu\in\DeltaK,\ \inner{\nu}{c}\ge t\bigr\}=:I(t)
\end{equation}
the least comparator complexity that moves the score mean to $t$ --- the same relative entropy the terminal payoff $L_\Gamma$ requires of a comparator, read at the mean-$t$ face in place of a relative-entropy sphere.

\subsubsection{Cram\'er}
\label{app:classical-cramer}

The Chernoff bound is one side of the large-deviation statement; the tilted-law argument matches it from below, so the exponent $I(t)$ of \eqref{eq:app-chernoff-rate} is exact.
Since $t$ is interior, the supremum defining $\Lambda^*(t)$ is attained at a finite $\lambda^\star$ with $\Lambda'(\lambda^\star)=t$, and the attaining tilt is the game's Gibbs law at scale $\lambda^\star$: $P_{\lambda^\star}(i)\propto P(i)e^{\lambda^\star c(i)}$, with mean $\E_{P_{\lambda^\star}}[c]=\Lambda'(\lambda^\star)=t$ and
\begin{equation}\label{eq:app-cramer-klidentity}
\KL(P_{\lambda^\star}\|P)=\lambda^\star t-\Lambda(\lambda^\star)=\Lambda^*(t)
\end{equation}
Fix $\varepsilon>0$ and let $A_n:=\{|\bar X_n-t|<\varepsilon\}$.
Changing measure from $P^n$ to $P_{\lambda^\star}^n$,
\[
\Pr[\bar X_n>t-\varepsilon]
\;\ge\;
P^n(A_n)
=
\E_{P_{\lambda^\star}^n}\Bigl[\1_{A_n}\,e^{-\lambda^\star\sum_j c(X_j)+n\Lambda(\lambda^\star)}\Bigr]
\;\ge\;
e^{-n(\lambda^\star(t+\varepsilon)-\Lambda(\lambda^\star))}\,P_{\lambda^\star}^n(A_n),
\]
and the law of large numbers under the tilt gives $P_{\lambda^\star}^n(A_n)\to1$, so $\liminf_n n^{-1}\log\Pr[\bar X_n>t-\varepsilon]\ge-\Lambda^*(t)-\lambda^\star\varepsilon$.
Letting $\varepsilon\downarrow0$ and combining with \eqref{eq:app-chernoff-markov}, the exponent is exactly the rate function at every interior continuity point,
\begin{equation}\label{eq:app-cramer-exact}
\lim_{n\to\infty}\;-\frac1n\log\Pr[\bar X_n\ge t]=I(t)=\Lambda^*(t)
\end{equation}
The two expressions for the rate are the two sides of the game's duality: $\Lambda^*$ is the scale-side (Legendre) reading and $I$ the comparator-side (relative-entropy) reading of the same terminal variational problem.

\subsubsection{Gibbs conditioning and the conditional limit}
\label{app:classical-gibbs}

On a face the game's optimizer is exact at every finite $n$, with no limit taken.
Fix $A\subseteq[K]$ with $P(A)>0$ and condition on $\{X_j\in A\ \text{for all}\ j\le n\}$: the joint law factorizes, so
\begin{equation}\label{eq:app-face-exact}
\Pr[X_1=a\mid X_j\in A\ \forall j\le n]
=
\frac{P(a)\1\{a\in A\}}{P(A)}
=
P(a\mid A)
\end{equation}
the conditioned prior itself --- which is precisely the $I$-projection of $P$ onto the face $\{\nu:\supp(\nu)\subseteq A\}$, at relative entropy $\KL(P(\cdot\mid A)\|P)=-\log P(A)$, the event cost that Proposition~\ref{prop:event-restricted} isolates and that the game's terminal functional imposes.
For a general closed convex constraint set $E\subseteq\DeltaK$ with $P\notin E$ and $\inf_{\nu\in E}\KL(\nu\|P)$ attained at the $I$-projection $\nu^\star$ (unique by strict convexity), the conditional limit theorem replaces the exact identity by a limit: conditionally on the empirical type $\hat\nu_n$ lying in $E$, the law of $X_1$ converges to $\nu^\star$.
The mechanism is the Pythagorean inequality of the $I$-projection \cite{Csiszar75}, $\KL(\nu\|P)\ge\KL(\nu\|\nu^\star)+\KL(\nu^\star\|P)$ for every $\nu\in E$, which places any type in $E$ at distance $\delta$ from $\nu^\star$ at exponent $\KL(\nu^\star\|P)+\delta$; Sanov's bound below then makes such types conditionally negligible, and averaging types near $\nu^\star$ yields the limit \cite{dembo2010large}.
The game reads the conditional limit as its $\eta\downarrow0$ equilibrium: the optimizer of \eqref{eq:event-restricted} is the Gibbs tilt of the conditioned prior, and as the scale vanishes the tilt relaxes onto the $I$-projection itself.

\subsubsection{Finite-state Sanov}
\label{app:classical-sanov}

The closing step is the method of types.
For a type $\nu$ of denominator $n$ (an empirical distribution of $n$ points), every sequence $x^n$ of type $\nu$ has probability $P^n(x^n)=\prod_i P(i)^{n\nu(i)}=e^{-n(\KL(\nu\|P)+H(\nu))}$, with $H(\nu):=-\sum_i\nu(i)\log\nu(i)$ the Shannon entropy, and the type class $T_n(\nu)$ contains $\binom{n}{n\nu(1),\dots,n\nu(K)}\le e^{nH(\nu)}$ sequences, at least $(n+1)^{-K}e^{nH(\nu)}$ of them; hence
\begin{equation}\label{eq:app-types}
(n+1)^{-K}\,e^{-n\KL(\nu\|P)}
\;\le\;
P^n\bigl[\hat\nu_n=\nu\bigr]
\;\le\;
e^{-n\KL(\nu\|P)}
\end{equation}
Since there are at most $(n+1)^K$ types of denominator $n$, summing the upper bound over the types in a closed set $E$ and taking the best lower-bound type near any interior $\nu\in E$ gives
\begin{equation}\label{eq:app-sanov}
-\inf_{\nu\in\operatorname{int}E}\KL(\nu\|P)
\;\le\;
\liminf_n\ \frac1n\log P^n\bigl[\hat\nu_n\in E\bigr]
\;\le\;
\limsup_n\ \frac1n\log P^n\bigl[\hat\nu_n\in E\bigr]
\;\le\;
-\inf_{\nu\in E}\KL(\nu\|P)
\end{equation}
the finite-state Sanov theorem.
The game's contribution is again the exponent's variational form: on the simplex of types, $\inf_{\nu\in E}\KL(\nu\|P)$ is the Lagrange dual of the comparator-class problem that $L_\Gamma$ solves (Proposition~\ref{prop:terminal-dual}), and when $E$ is a face it collapses to the closed form of Proposition~\ref{prop:event-restricted}.
The polynomial factor $(n+1)^{\pm K}$ is absorbed by the exponential scale.

\subsection{Further results from Section~\ref{sec:concentration}}
\label{app:further-concentration}

\subsubsection{The sufficient-statistic catalogue}
\label{app:sufficient-statistics}

Each row enters the update~\eqref{eq:ftrl} as the cumulative statistic $C_{t-1}(i)$ of a per-round score. Here $\psi_\tau$ is Winsorization at level $\tau$, so that $|\psi_\tau|\le\tau$, $\psi_{\mathrm{Cat}}$ is Catoni's influence function, and $u_t$ is a predictable per-round offset ($c_t=\ell_t+u_t$), whose mismatch term $\sum_t\inner{\nu-p_t}{u_t}$ records exactly how the game on~$\ell_t$ differs from the reduced game on $c_t$.

\begin{table}[h]
\centering
\begin{tabular}{lll}
\toprule
\textbf{Algorithm} & \textbf{Sufficient statistic $C_{t-1}(i)$} & \textbf{Consequence} \\
\midrule
Vanilla Hedge & $\sum_{s<t}\ell_s(i)$ & Raw cumulative loss \\
Optimistic Hedge & $\sum_{s<t}(\ell_s(i)-m_s(i))$ & Only residuals contribute \\
Compensated & $\sum_{s<t}(\ell_s(i)+u_s(i))$ & Variance corrections included \\
Soft specialists & $\sum_{s<t}\beta_s(i)\ell_s(i)$ & Confidence-rated experts \\
One-sided statistics & $\sum_{s<t}(\inner{p_s}{c_s}-c_s(i))_+$ & Only adversarial coordinates \\
Winsorized scores & $\sum_{s<t}\psi_\tau(z_s(i))$ & Finite loss on heavy tails \\
Catoni transform & $\sum_{s<t}\alpha_s^{-1}\psi_{\mathrm{Cat}}(\alpha_s z_s(i))$ & Bounded influence, exact bias \\
\bottomrule
\end{tabular}
\caption{Sufficient-statistic choices and their consequences. The game is defined relative to whichever statistic the learner scores; changing it changes the state, the intrinsic time, and the move constraint, and leaves every identity intact.}
\label{tab:sufficient-statistics}
\end{table}

\subsubsection{Comparators on and off the minimal-mean face}
\label{app:luckiness-faces}

\begin{proposition}[Diffuse comparators reduce to the best expert]
\label{prop:diffuse-reduces}
Run fixed-rate Hedge on i.i.d.\ composite losses $c_t\in[0,1]^K$ with mean $\mu$, and let $\mu^*=\min_i\mu(i)$ and $k^*\in\argmin_i\mu(i)$. For every comparator $\nu\in\DeltaK$,
\begin{equation}\label{eq:diffuse-monotone}
\E R_T^c(\nu)=\E R_T^c(\delta_{k^*})-T\bigl(\inner{\nu}{\mu}-\mu^*\bigr)\ \le\ \E R_T^c(\delta_{k^*})
\end{equation}
Consequently, whenever some minimal-mean expert $k^*$ satisfies the point-mass low-noise condition with a finite constant $\kappa_{k^*}$, the choice $\eta=\min\{1,[2(e-2)\kappa_{k^*}]^{-1}\}$ gives a constant-in-$T$ bound for every $\nu$,
\begin{equation}\label{eq:diffuse-constant}
\E R_T^c(\nu)\ \le\ \E R_T^c(\delta_{k^*})\ \le\ 2\bigl(1+2(e-2)\kappa_{k^*}\bigr)\bigl(-\log\pi(k^*)\bigr)
\end{equation}
The one comparator class left uncovered by the point-mass mechanism is a tied minimum carried by experts with distinct losses, at which no minimal-mean expert satisfies the point-mass condition with a finite constant.
\end{proposition}

The restriction to that face is sharp, and comes directly from the definition.
If the comparator-centered low-noise condition holds with a finite constant, then any expert in $\supp(\nu)$ sharing the posterior-average mean loss must have $c_t(i)=\inner{\nu}{c_t}$ almost surely; taking $i\in\argmin_j\mu(j)$ then forces $\inner{\nu}{\mu}=\mu^*$, so the condition is restrictive for a diffuse comparator.

Expanding the square in Proposition~\ref{prop:luckiness-envelope} about the algorithm mean gives the exact identity $\Psi_t(\nu)=\Var_{p_t}(c_t)+\inner{p_t-\nu}{c_t}^2$, so the comparator enters the comparator-centered second moment only through the instantaneous regret $\inner{p_t-\nu}{c_t}$.
A comparator with $\inner{\nu}{\mu}>\mu^*$ is beaten by a linear margin $T(\inner{\nu}{\mu}-\mu^*)$ under \eqref{eq:diffuse-monotone}, so a constant regret bound holds a fortiori; one with $\inner{\nu}{\mu}=\mu^*$ ties the best expert in mean and inherits its guarantee whenever that expert meets the point-mass condition, which is the mechanism Theorem~\ref{thm:luckiness} uses.
For finite $K$ a different argument handles the tied-minimum face whenever the off-face mean gap is strict.
If every expert outside the minimal-mean face has mean at least $\mu^*+\Delta$ for some $\Delta>0$, fixed-rate Hedge puts geometrically decaying weight on those experts, so the excess over a face comparator sums to $\E R_T^c(\nu)=O(1)$ uniformly in $T$.

\subsubsection{Drawdown and mixture width, numerically}
\label{app:drawdown-numerics}

Since $\inf_{t\le T}\Lambda_t\le\Lambda_0=0$, every path has $\mathrm{DD}_T\ge\Lambda_T=-\log \mathrm{PR}_T$, and under the true parameter the mean of that lower bound, $\E[-\log \mathrm{PR}_T]$, is the accumulated log-moment deficit of the $e$-process: linear in $T$ for independent increments, and equal to $T\eta^2\sigma^2/2$ for a score of sub-Gaussian scale $\sigma$ at a fixed $\eta$. The two regimes separate numerically.
Take a standard normal score at $\alpha=0.05$, so the budget is $\log(1/\alpha)=3.00$ nats.
With $\eta$ held at its $T=10^3$ tuning (the $\eta$ solving $T\eta^2\sigma^2/2=\log(1/\alpha)$ there), the deficit is linear in the horizon and the drawdown follows it.
The averages are $0.74$, $3.86$ and $30.9$ nats at $T=10^2$, $10^3$ and $10^4$, each the linear deficit at that horizon plus a running-minimum overshoot of less than one nat.
Retuned to each horizon instead, $T\eta^2\sigma^2/2$ is held at $\log(1/\alpha)$ throughout, and the same quantity stays near $3.8$ nats at all three.

The width comparison is read at the same setting: a standard normal score at $\alpha=0.05$, three sets read once at $T=10^3$, each against a nominal two-sided miss rate of $10^{-1}$.
The fixed-horizon interval is the reference, missing at $1.4\times10^{-2}$.
A normal-mixture sequence at the mixture's own best tuning is $24.0\%$ wider and misses at $2.4\times10^{-3}$; under the conventional unit tuning it is $46.8\%$ wider and misses at $3.3\times10^{-4}$.
The three miss rates are the Gaussian endpoint tails of the three radii, so the shortfall against the nominal level measures the same conservatism the extra width does.
Both widths are read off closed forms: the mixture boundary $[\tau^{-2}(1+\tau^2\sigma^2T)(2\log(1/\alpha)+\log(1+\tau^2\sigma^2T))]^{1/2}$ at mixing scale $\tau$, against the endpoint radius $\sigma\sqrt{2T\log(1/\alpha)}$.
Their ratio depends on $(\tau,T)$ only through $\tau^2\sigma^2T$, which is why the best-tuned figure carries no horizon.

\subsection{Further results from Section~\ref{sec:games}}
\label{app:proofs-games}

\subsubsection{Opponent models and the robust center}
\label{app:robust-center}

Opponent models pool multiplicatively.
If $\omega_{t,1},\dots,\omega_{t,W}\in\Delta([L])$ are candidate opponent models with response priors $\pi_{t,w}(i)\propto e^{-\eta (M\omega_{t,w})_i}$, then each $\log\pi_{t,w}(i)$ is affine in $(M\omega_{t,w})_i$ up to an $i$-independent normalizer, so the geometric pool at nonnegative weights $\alpha_{t,w}$ is $\prod_w \pi_{t,w}(i)^{\alpha_{t,w}}\propto\exp(-\eta\sum_w \alpha_{t,w}(M\omega_{t,w})_i)$.
Geometric pooling in prior space is ordinary exponential response to the averaged opponent model, and the worst case over weightings gives a robust center.
The duality below is stated for strategy priors in~\cite{balsubramani2026adaptive}; it is recorded here in the game's coordinates.

\begin{proposition}[Robust center over several priors]
\label{prop:robust-center}
Given positive priors $\pi_1,\dots,\pi_W\in\DeltaK$, define $C_\alpha(\pi_{1:W}):=-\log\sum_{i=1}^K \prod_{w=1}^W \pi_w(i)^{\alpha_w}$ and $C(\pi_{1:W}):=\max_{\alpha\in\Delta([W])}C_\alpha(\pi_{1:W})$.
Then
\begin{equation}\label{eq:robust-center}
C(\pi_{1:W})
=
\min_{p\in\DeltaK}\max_{w\in[W]}\KL(p\|\pi_w)
\end{equation}
and the optimizer is the geometric mixture induced by a maximizing~$\alpha$.
\end{proposition}

In a repeated game, using the robust center as a baseline and running optimistic Hedge on the residuals yields regret controlled by the intrinsic time of the residual sequence.
If one of the~$W$ models is close to the opponent's actual strategy, the residuals are small. Convergence is then fast.

\subsubsection{The regret-to-equilibrium conversions}
\label{app:regret-to-equilibrium}

\begin{proposition}[External regret implies CCE]
\label{prop:cce}
If each player in a finite normal-form game satisfies external regret at most $\varepsilon_m T$ --- regret against the best single action in hindsight --- then the empirical distribution of play is an $\varepsilon$-coarse correlated equilibrium (CCE) \cite{hart2000simpleb,blum2007external}. Here $\varepsilon=\max_m \varepsilon_m$.
\end{proposition}

\begin{proposition}[Swap regret implies CE]
\label{prop:swap-ce}
If each player satisfies swap regret at most $\varepsilon_m T$ --- regret against the best fixed remapping of its own actions --- then the empirical distribution of play is an $\varepsilon$-correlated equilibrium \cite{hart2000simpleb}.
\end{proposition}

\subsubsection{Why the ledger supplies no converse}
\label{app:cce-converse}

The forward implication $\max_k V_T^{(k)}=o(T)\Rightarrow d(\sigma_T,\CCE)=o(T^{-1/2})$ holds because $V_T^{(k)}=\sum_t Q_t^{(k)}$ governs every horizon-growing term of $R_T^{(k)}=P_T^{(k)}+D_T^{(k)}+B_T^{(k),\star}$: the terminal transport obeys $B_T^{(k),\star}\le\Gamma^{(k)}/\eta_T^{(k)}=O(\sqrt{\Gamma^{(k)}V_T^{(k)}})$, and the drift is a gain.

The reverse implication is not available. A lower bound on $R_T^{(k)}$ in terms of $V_T^{(k)}$ must offset the envelope's lower side $2C\sqrt{\Gamma^{(k)}V_T^{(k)}}-C^2\Gamma^{(k)}$ against the drift gain $D_T^{(k)}$, whose magnitude is itself controlled only by the transport scale $\Gamma^{(k)}/\eta_T^{(k)}\le C^{-1}\sqrt{\Gamma^{(k)}V_T^{(k)}}$. The residual coefficient $2C-C^{-1}$ is exactly zero at the headline $C=1/\sqrt2$. So in the setting of Theorem~\ref{thm:k-player-ledger} under the self-tuned schedule, whether $d(\sigma_T,\CCE)=o(T^{-1/2})$ forces $\max_k V_T^{(k)}=o(T)$ is undecided, and deciding it needs a lower bound the envelope cannot supply at that constant.

\begin{theorem}[Bounded intrinsic time at strict equilibria]
\label{thm:unconditional-V}
Consider intrinsic-time regret-matching self-play (Theorem~\ref{thm:regret-matching}) in a finite game under the self-tuned schedule $\eta_t^{(k)}=C\sqrt{\Gamma^{(k)}/(V_{t-1}^{(k)}+\Gamma^{(k)})}$, and write $B:=\max_k\max_{\bm a}\lvert c_k(\bm a)\rvert$ for the loss range.
Suppose the realized trajectory converges to a \emph{strict} pure Nash profile $a^\star$ with stability gap $\Delta:=\min_{k}\min_{i\ne a_k^\star}\bigl[u_k(a_k^\star,a_{-k}^\star)-u_k(i,a_{-k}^\star)\bigr]>0$
(cost convention $u_k=-c_k$: $a_k^\star$ strictly beats every deviation by at least $\Delta$ when opponents play $a_{-k}^\star$), with no assumption on the rate of approach.
Then for every player~$k$ the off-equilibrium mass $1-p_t^{(k)}(a_k^\star)$ and the per-round RCGF $Q_t^{(k)}$ decay geometrically, the schedule freezes at a positive scale $\eta_t^{(k)}\to\eta_\infty^{(k)}\in(0,C]$, and the realized intrinsic time is \emph{bounded},
\begin{equation}\label{eq:strict-eq-VboundedO1}
V_T^{(k)}=\sum_{t=1}^T Q_t^{(k)}=O(1)
\end{equation}
so that $d(\sigma_T,\mathrm{CCE})=O(1/T)$ in the ledger \eqref{eq:k-player-ledger}.
\end{theorem}

A few canonical instances make the ledger \eqref{eq:k-player-ledger} concrete.
\begin{itemize}[leftmargin=1.3em,topsep=2pt,itemsep=2pt]
\item \emph{Two-player zero-sum (exact Nash gap under admissible budgets).} If $u_2=-u_1=:M$ and $\Gamma_{\mathrm{row}}\ge\max_i(-\log\pi_{\mathrm{row}}(i))$, $\Gamma_{\mathrm{col}}\ge\max_j(-\log\pi_{\mathrm{col}}(j))$ (every pure strategy is in the comparator class), then $d(\sigma_T,\mathrm{CCE})$ equals the larger of the two players' exact unilateral deviation gains (each player's own exact-regret sum $P_T^{(k)}+D_T^{(k)}+B_T^{(k),\star}$ divided by~$T$). These two gains sum to the classical Nash duality gap $\max_{\omega'}\bar p_T^\top M\omega'-\min_{p'} p'^\top M\bar\omega_T$ of \eqref{eq:saddle-gap-exact}, so $d(\sigma_T,\mathrm{CCE})$ is at most that gap, with equality exactly when one player's deviation gain vanishes.
With smaller comparator budgets the same identity bounds a relative-entropy-restricted deviation class.
\item \emph{Random general-sum $2\times 2$.} For utilities drawn $\mathrm{Unif}([-1,1])$ and the schedule $\eta_t^{(k)}=\sqrt{(\log 2)/(V_{t-1}^{(k)}+\log 2)}$, the linear-intrinsic-time regime $V_T^{(k)}=\Theta(T)$ arises on the instances whose only equilibrium is interior, where \eqref{eq:k-player-ledger} reads $d(\sigma_T,\mathrm{CCE})=O(T^{-1/2})$ with the leading constant governed by the realized $V_T^{(k)}$. Instances admitting a pure equilibrium to which self-play converges instead have bounded intrinsic time and fall under the potential/congestion case below.
\item \emph{Potential / congestion games.} On traces where self-play converges to a \emph{strict} pure Nash profile, the per-round centered RCGF collapses geometrically as $p_t^{(k)}$ concentrates and the self-tuned schedule freezes at a positive scale. The realized intrinsic time is therefore in fact \emph{bounded}, $V_T^{(k)}=O(1)$ (Theorem~\ref{thm:unconditional-V}), yielding $d(\sigma_T,\mathrm{CCE})=O(1/T)$ --- sharper than the $O(\sqrt{\log T}/T)$ a merely logarithmic intrinsic time would give. Whether a finite potential game can be forced to keep $V_T^{(k)}$ super-logarithmic under this schedule, via non-convergent multiplicative-weights dynamics~\cite{palaiopanos2017congestion}, is a convergence-of-dynamics question this paper does not settle.
\item \emph{Adversarial linear-growth witness.} Conversely, if a player's losses are constructed to keep $p_t^{(k)}$ far from any equalizer (e.g.\ a cyclic loss sequence), then $V_T^{(k)}=\Theta(T)$. Along that trajectory the ledger is an equality, so its bound is the classical $O(1/\sqrt T)$; whether the realized rate matches it is undecided.
\end{itemize}
The instances separate through the realized intrinsic time alone: the ledger reproduces each player's exact unilateral deviation gain, so $d(\sigma_T,\mathrm{CCE})$ is one exploitability in place of the Nash duality gap, and $V_T^{(k)}$ bounded versus linear is exactly the $O(1/T)$ versus $O(1/\sqrt T)$ dichotomy.

\subsection{Further results from Section~\ref{sec:partial}}
\label{app:proofs-partial}
\begin{theorem}[Generic pathwise bandit decomposition]
\label{thm:bandit}
For a comparator $u\in\DeltaK$, the sampled composite-loss regret decomposes pathwise as follows.
\begin{equation}\label{eq:bandit-decomp}
\mathrm{Reg}_T^c(u)
=
M_T^{\mathrm{play}}
+
\Xi_T
-
M_T^{\mathrm{est}}(u)
+
\mathrm{Bias}_T(u)
+
\widehat D_T
+
\widehat B_T(u)
+
\sum_{t=1}^T \eta_t \widehat Q_t
\end{equation}
Here $M_T^{\mathrm{play}}:=\sum_t (c_t(A_t)-\inner{\mu_t}{c_t})$ (play martingale),
$\Xi_T:=\sum_t \inner{\mu_t-p_t}{c_t}$ (exploration cost),
$M_T^{\mathrm{est}}(u):=\sum_t \inner{p_t-u}{\widehat c_t-\bar c_t}$ (estimation martingale, with $\bar c_t:=\E[\widehat c_t\mid\mathcal F_{t-1}]$ the predictable conditional mean of the estimated score),
$\mathrm{Bias}_T(u):=\sum_t \inner{p_t-u}{c_t-\bar c_t}$ (predictable bias),
and the last three terms are the full-information game on estimated scores.
Both $M_T^{\mathrm{play}}$ and $M_T^{\mathrm{est}}(u)$ are martingales.
\end{theorem}

The one term the bandit decomposition~\eqref{eq:bandit-decomp} appears to leave implicit --- the estimated per-round RCGF $\widehat Q_t=Q_{\eta_t}(p_t,\widehat z_t)$ on the estimated scores --- is not a quantity to be bounded but has an exact closed form. Conditioned on the past its only randomness is the drawn arm $A_t$, a $K$-valued variable. Its conditional expectation is therefore a finite explicit sum.

\begin{proposition}[Exact per-round value in the true losses and observation geometry]
\label{prop:bandit-exact-value}
Let $A_t\sim\mu_t$ and let $\widehat z_t$ be any estimated centered score that is a deterministic function of the drawn arm and its feedback. Write $\widehat z_t^{(a)}$ for its realized value on $\{A_t=a\}$. Then
\begin{equation}\label{eq:bandit-value-general}
\E\bigl[\widehat Q_t\mid\mathcal F_{t-1}\bigr]
=\eta_t^{-2}\sum_{a}\mu_t(a)\,\log\sum_i p_t(i)\,e^{\eta_t\widehat z_t^{(a)}(i)}
\end{equation}
For plain-bandit inverse-propensity scores $\widehat c_t(i)=c_t(i)\1\{A_t=i\}/\mu_t(i)$ this collapses coordinatewise to
\begin{equation}\label{eq:bandit-value-ips}
\E\bigl[\widehat Q_t\mid\mathcal F_{t-1}\bigr]
=\frac{\inner{p_t}{c_t}}{\eta_t}
+\eta_t^{-2}\sum_{a}\mu_t(a)\,\log\!\Bigl(1-p_t(a)+p_t(a)\,e^{-\eta_t c_t(a)/\mu_t(a)}\Bigr)
\end{equation}
exact at every scale $\eta_t$ and written entirely in the true losses $c_t$ and the observation geometry $\mu_t$. Its small-scale limit is the observation-inflated variance,
\begin{equation}\label{eq:bandit-value-smallscale}
\E\bigl[\widehat Q_t\mid\mathcal F_{t-1}\bigr]
=\tfrac12\sum_{a}\frac{p_t(a)\bigl(1-p_t(a)\bigr)}{\mu_t(a)}\,c_t(a)^2+O(\eta_t),
\qquad
\E\bigl[\Var_{p_t}(\widehat c_t)\mid\mathcal F_{t-1}\bigr]=\sum_{a}\frac{p_t(a)\bigl(1-p_t(a)\bigr)}{\mu_t(a)}\,c_t(a)^2
\end{equation}
the geometry entering through the propensity inflation $1/\mu_t(a)$; for a feedback graph the same finite sum holds with the graph-IX scores, the geometry entering through the observation probabilities $o_t(a)$.
\end{proposition}

\begin{proposition}[Proper-wrapper duplication identity]
\label{prop:duplication}
Let $\Pi$ be a finite policy class and lift it to an expert class $\mathcal E$ with prior $\tilde\pi\in\Delta(\mathcal E)$. A designated policy $\pi^\star\in\Pi$ is represented by a group $G\subseteq\mathcal E$ of \emph{loss-identical} copies (two copies of $\pi^\star$ induce the same action law $\pi^\star(\cdot\mid x_t)$ in every context, hence $\widehat c_t(e)=\widehat c_t(\pi^\star)$ for all $e\in G$ and all~$t$). Write $\tilde\pi(G):=\sum_{e\in G}\tilde\pi(e)$ for the aggregate prior mass. Running the full-information game (Section~\ref{sec:repeated-game}) over $\mathcal E$ at scale~$\eta$ on the estimated expert losses, the lowest-complexity comparator supported on~$G$ is $\nu^\star(e)=\tilde\pi(e)/\tilde\pi(G)$ on~$G$. The comparator complexity of competing against $\pi^\star$ is then exactly
\begin{equation}\label{eq:duplication-kl}
\min_{\nu:\,\supp(\nu)\subseteq G}\KL(\nu\|\tilde\pi)
=\KL(\nu^\star\|\tilde\pi)
=-\log\tilde\pi(G)
=\log\frac{1}{\tilde\pi(G)} 
\end{equation}
Consequently, by the transport identity~\eqref{eq:fixed-cumulative} applied in the lifted game, the comparator-complexity term of the policy-expert decomposition is
\begin{equation}\label{eq:duplication-transport}
\widehat B_T^{\Pi}(\pi^\star)
=\frac{\log\tfrac{1}{\tilde\pi(G)}-\KL(\nu^\star\|\widehat\rho_{T+1})}{\eta}
\end{equation}
with $\widehat\rho_{T+1}$ the terminal posterior. The leading term $\eta^{-1}\log\tfrac{1}{\tilde\pi(G)}$ is exact.
\end{proposition}

\begin{corollary}[The comparator-complexity reduction as a mass log-ratio]
\label{cor:duplication-reduction}
If $\pi^\star$ carries original mass $\beta_0=\tilde\pi_0(\pi^\star)$ and, after duplication, aggregate mass $\beta=\tilde\pi(G)$, the reduction in comparator complexity is
\begin{equation}\label{eq:log-ratio}
\frac{1}{\eta}\Bigl(\log\tfrac{1}{\beta_0}-\log\tfrac{1}{\beta}\Bigr)=\frac{1}{\eta}\log\frac{\beta}{\beta_0}
\end{equation}
the log-ratio of duplicated to original mass. Driving $\beta\to1$ sends the designated policy's comparator complexity to~$0$.
\end{corollary}

This is the same pooling phenomenon that runs through the framework: loss-identical experts coincide on every round, so their prior masses pool additively and the relative-entropy comparator complexity sees only the aggregate mass. Assembling this into the bandit decomposition of Theorem~\ref{thm:bandit} gives the wrapper with matching constants: with sampling $\mu_t=(1-\gamma)\bar p_t+\gamma\,u_{[A]}$, $\bar p_t(a)=\sum_e p_t(e)\,\pi_e(a\mid x_t)$, and unbiased IPS losses $\widehat c_t(a)=c_t(a)\1\{A_t=a\}/\mu_t(a)$, one has $\mathrm{Bias}_T\equiv0$.
The exploration cost obeys $|\Xi_T|\le\gamma T$ (since $\mu_t-\bar p_t=\gamma(u_{[A]}-\bar p_t)$ and losses lie in $[0,1]$; the inner product $\inner{u_{[A]}-\bar p_t}{c_t}$ carries either sign); the IPS scores obey $\widehat c_t(a)\le A/\gamma$; and $\widehat B_T^\Pi(\pi^\star)$ is~\eqref{eq:duplication-transport}. Taking expectations removes the play and estimation martingales, and at constant scale $\widehat D_T=0$. This leaves
\begin{equation}\label{eq:wrapper-bound}
\E\bigl[\mathrm{Reg}_T^c(\pi^\star)\bigr]
\;\le\;
\gamma T
\;+\;
\frac{1}{\eta}\log\frac{1}{\tilde\pi(G)}
\;+\;
\sum_{t=1}^T \eta_t\,\E\bigl[\widehat Q_t^\Pi\bigr]
\end{equation}
in which the two costs the reduction introduces --- the explicit exploration term $\gamma T$ and the $A/\gamma$ IPS rescaling inside the estimated intrinsic time $\sum_t\eta_t\widehat Q_t^\Pi$ --- carry exact constants, and the comparator complexity is the exact duplication-adjusted $\log\tfrac{1}{\tilde\pi(G)}$ in place of $\log|\Pi|$. As with the plain-bandit case, the sharp rate closes once the estimated per-round RCGF $\widehat Q_t^\Pi$ is evaluated by its exact closed form (Proposition~\ref{prop:bandit-exact-value}).

\begin{theorem}[Pairwise testing exponents are the concentration game's Bellman potential]
\label{thm:testing-bellman}
With $S$, $P_j^a$, and $\eta=s$ as above, for every $s\in(0,1)$:
\begin{enumerate}[label=(\alph*)]
\item\emph{(Exact identity.)} The Chernoff coefficient is minus the game's per-observation Bellman potential of $S$,
\begin{equation}\label{eq:chernoff-is-bellman}
C_s(P_i^a,P_j^a)= -\,s\,W_s(S)= -\log\E_{P_j^a}\!\bigl[e^{sS}\bigr]= -\log Z(s,1-s)
\end{equation}
where $Z(s,1-s)=\int (dP_i^a)^s(dP_j^a)^{1-s}$ is the two-way mixed-coincidence partition function.
\item\emph{(Equilibrium play is the Chernoff tilt.)} The game's Gibbs equalizer at $(S,s)$ is the geometric-mean tilt of the two hypotheses,
\begin{equation}\label{eq:gibbs-is-chernoff-tilt}
G_s(S)=\rho_{s,S}=\nabla W_s(S)\;\propto\;(dP_i^a)^s(dP_j^a)^{1-s}=:R_s
\end{equation}
Equivalently $R_s$ is the relative-entropy barycenter, and by the mixed-coincidence identity \cite{balsubramani2026information}, $C_s(P_i^a,P_j^a)=\min_p\{\,s\,\KL(p\|P_i^a)+(1-s)\,\KL(p\|P_j^a)\,\}=s\,\KL(R_s\|P_i^a)+(1-s)\,\KL(R_s\|P_j^a)$, attained at $p=R_s$.
\item\emph{(Scale optimization is Chernoff information.)} The tilt optimum is the game's scale optimization, $\max_{s\in[0,1]}C_s(P_i^a,P_j^a)=-\min_{s}\,s\,W_s(S)$, and at the optimizer $s^\star$ the score is neutral in nature's direction, $\E_{R_{s^\star}}[S]=0$, from which the balance
\begin{equation}\label{eq:chernoff-balance}
\max_{s}C_s(P_i^a,P_j^a)=\KL(R_{s^\star}\|P_i^a)=\KL(R_{s^\star}\|P_j^a)
\end{equation}
\end{enumerate}
Consequently the pairwise testing exponent $\Gamma(a)=\min_{i\neq j}\max_{s\in[0,1]}C_s(P_i^a,P_j^a)$ is the edge-restricted multi-way coincidence radius of the arm-induced laws $\{P_i^a\}_{i\in[W]}$, and equals the MAP error exponent for identifying the latent hypothesis from repeated pulls of arm~$a$, $-\lim_{n\to\infty}n^{-1}\log P_{e,n}(a)=\Gamma(a)$ \cite{balsubramani2026information}.
\end{theorem}

$R_s$ is the log-linear pool of $P_i^a$ and $P_j^a$, and $C_s$ is the two-way mixed-coincidence divergence of \cite{balsubramani2026information}; at the optimizer $s^\star$ the neutrality $\E_{R_{s^\star}}[S]=0$ is the same equalizing condition the Bellman equalizer enforces. The geometric mean that drives the Bellman game --- the Gibbs tilt $\rho_{\eta,S}\propto\pi\,e^{\eta S}$, equilibrium play and gradient of the log-partition --- is, on a log-likelihood-ratio score, exactly the Chernoff tilt $R_s\propto p_i^sp_j^{1-s}$ whose normalizer is the testing divergence. The game's scale optimization is the Chernoff-information tilt optimization, and identification difficulty is the reciprocal scale of the game's terminal concentration. An arm with large $\Gamma(a)$ is one whose two closest induced hypotheses have small geometric-mean coincidence $Z(s^\star,1-s^\star)$, i.e.\ whose concentration game drives the posterior onto the truth fastest, so $\Gamma(a)$ is an identification score. The remaining modeling step --- turning per-pull exponents $\Gamma(a)$ into a finite-horizon pure-exploration policy with matching sample complexity --- is a sequential allocation problem on top of this identity, and adds no further fact about the geometric mean.

Most of the existing algorithmic templates for bandits reduce to estimator-specific games in this framework.

\begin{itemize}
\item\textbf{Explicit-exploration IPS: }
With $\mu_t=(1-\gamma_t)p_t+\gamma_t u_K$ and standard inverse-propensity-score (IPS) estimation, the bias vanishes and $\Xi_T$ records the exploration cost.

\item\textbf{Implicit exploration / EXP3-IX: }
With $\mu_t=p_t$ and estimator $\widehat c_t^{\mathrm{IX}}(a)=u_t(a)+\ell_t(A_t)\1\{A_t=a\}/(p_t(a)+\gamma_t)$, the conditional mean is $c_t-d_t^{\mathrm{IX}}$ where $d_t^{\mathrm{IX}}(a)=\gamma_t\ell_t(a)/(p_t(a)+\gamma_t)$.
Total bias is bounded by $\sum_t K\gamma_t/(1+K\gamma_t)$.

\item\textbf{Predictable offsets: }
If the estimate uses a control variate or optimistic baseline~$m_t$, only the estimated intrinsic time and the jump term change; the variational front end does not.

\item\textbf{Feedback graphs: }
For a directed feedback graph $G_t=(V,E_t)$, let $o_t(a)$ be the probability that arm~$a$ is observed.
The graph-IX estimator $\widehat c_t^{G\text{-}\mathrm{IX}}(a)=u_t(a)+\ell_t(a)\1\{A_t\in N^{\mathrm{in}}_t(a)\}/(o_t(a)+\gamma_t)$ has predictable bias $d_t^{G\text{-}\mathrm{IX}}(a)=\gamma_t \ell_t(a)/(o_t(a)+\gamma_t)$.
Both the bias and the natural quadratic surrogates are controlled by the observation ratios $p_t(a)/o_t(a)$.
For undirected graphs these ratios are bounded by the independence number $\alpha(G_t)$ \cite{alon2017journalb}.
Graph topology enters the game only through the observation geometry, leaving the Bellman calculus itself untouched.

Recent feedback-graph best-of-both-worlds algorithms \cite{RouyerEtAl2022FeedbackGraphs} combine explicit exploration with fast stochastic rates.
The exact identity explains the mechanism term by term: the cost of playing the exploration distribution $\mu_t$ instead of $p_t$ is recorded in $\Xi_T$, the graph-IX correction in $\mathrm{Bias}_T^{G\text{-}\mathrm{IX}}(u)$, and once stochastic gaps make observability ratios large for good actions, the realized intrinsic time stops accumulating.
\end{itemize}

\subsection{Further results from Section~\ref{sec:ts}}
\label{app:proofs-ts}

\begin{game}[Expected sampled concentration game]
\label{game:ts}
Game~\ref{game:central} at a fixed scale $\eta_t\equiv\eta$, with the learner's move and the terminal payoff changed. Start from $S_0=0$. For each round $t=1,\dots,T$:
\begin{enumerate}[label=\arabic*.,leftmargin=1.6em,topsep=2pt,itemsep=1pt]
\item Learner commits a predictable \emph{sampling law} $w_t\in\DeltaK$ over the experts;
\item Learner draws an expert $I_t\sim w_t$ and plays it;
\item Nature, observing $w_t$ but not the drawn $I_t$, picks a centered $z_t\in\R^K$ with $\inner{w_t}{z_t}=0$ and $Q_{\eta}(w_t,z_t)\le q_t$;
\item State advances by $S_t:=S_{t-1}+z_t$.
\end{enumerate}
Nature's move is thus the move of Game~\ref{game:central} with the sampling law in the mixed action's place.
The terminal payoff to nature is the \emph{expected} sampled budgeted regret at the committed scale,
\begin{equation}\label{eq:ts-payoff}
\Phi_T(w_{1:T})
:=
\E\Bigl[\,\sum_{t=1}^T\bigl(c_t(I_t)-\inner{w_t}{c_t}\bigr)\Bigr]
+\sum_{t=1}^T\bigl(W_{\eta}(S_t)-W_{\eta}(S_{t-1})\bigr)
+\frac{\Gamma}{\eta}
\end{equation}
whose first term is the expected play excess and whose remaining terms are the committed-scale value-to-go $\Gamma/\eta+W_{\eta}(S_T)$ measured from the start.
\end{game}

\begin{theorem}[Posterior sampling is the exact saddle point of Game~\ref{game:ts}]
\label{thm:ts-exact}
For any prior $\pi$, budgets $(\Gamma,V)$ with $\sum_t q_t\le V$, and fixed scale $\eta$:
\begin{enumerate}[label=(\roman*),topsep=2pt,itemsep=1pt]
\item The expected play excess vanishes for every predictable sampling law: $\E[c_t(I_t)-\inner{w_t}{c_t}\mid\mathcal F_{t-1}]=0$ whenever $I_t\mid\mathcal F_{t-1}\sim w_t$, and in particular for the Gibbs sampler $w_t=p_t=\rho_{\eta,S_{t-1}}$.
\item Posterior sampling is the unique equalizer: with $w_t=p_t$, the round increment $W_\eta(S_t)-W_\eta(S_{t-1})=\eta Q_\eta(p_t,z_t)$ is independent of nature's \emph{direction} and equals $\eta q_t$ at the budget, so
\begin{equation}\label{eq:ts-saddle-value}
\Val(\text{Game~\ref{game:ts}})=\Phi_T(p_{1:T})=\frac{\Gamma}{\eta}+\eta\sum_{t=1}^T q_t
\end{equation}
attained simultaneously by the Gibbs sampler (min over $w_{1:T}$) and by nature saturating each budget in the worst-comparator direction (max), with no order-of-play gap.
\item Any sampling law $w_t\neq p_t$ strictly loses: $\sup_{z}\Phi_T>\Gamma/\eta+\eta\sum_t q_t$ whenever some $q_t>0$, so $w_t=p_t$ is then the exact and unique minimizer; in the degenerate case $q_t\equiv 0$ nature is confined to $z_t=0$ and every sampling law attains the same value $\Gamma/\eta$.
\end{enumerate}
\end{theorem}

Posterior sampling therefore attains exactly the deterministic concentration-game value in expectation.
The strict loss at $w_t\neq p_t$ has a simple source: the value-to-go increment $W_\eta(S_t)-W_\eta(S_{t-1})=\eta Q_\eta(\rho_{\eta,S_{t-1}},z_t)$ is $\eta$ times the centered RCGF of $z_t$ under the Gibbs law $\rho_{\eta,S_{t-1}}$, a function of the state alone, whereas nature is constrained only by $Q_\eta(w_t,z_t)\le q_t$. The two forms then differ, so nature can satisfy its own budget while over-driving the Gibbs increment.

\begin{proposition}[The sampling fluctuation is intrinsic; no fluctuation-free realized posterior-sampling game]
\label{prop:ts-no-pathwise}
Fix any modification of Game~\ref{game:ts} whose payoff is an arbitrary measurable function $\Phi(w_{1:T},z_{1:T},I_{1:T})$ of the realized transcript. The learner draws $I_t\sim w_t$.
\begin{enumerate}[label=(\roman*),topsep=2pt,itemsep=1pt]
\item \textbf{Exact residual identity.} Subtracting \eqref{eq:ts-expected} from \eqref{eq:ts} gives, for \emph{every} comparator $\nu$ simultaneously and every predictable loss sequence,
\begin{equation}\label{eq:ts-residual}
\widehat R_T^c(\nu)-R_T^c(\nu)=M_T^{\mathrm{sam}},
\qquad M_T^{\mathrm{sam}}=\sum_{t=1}^T\bigl(c_t(I_t)-\inner{p_t}{c_t}\bigr)
\end{equation}
so at $w_t=p_t$ the residual $M_T^{\mathrm{sam}}$ is a mean-zero martingale.
\item \textbf{Impossibility of a fluctuation-free realized game.} No payoff that genuinely depends on the sampled play can be made constant over the sampling randomness (Appendix~\ref{app:proofs-ts}). A payoff built from $c_t(I_t)$ leaves, on a.e.\ path, an irreducible per-round variance $\Var_{p_t}(c_t)>0$ that no scoring rule removes.
\item \textbf{The high-probability route has a different optimizer.} Sharpening the sampling law to $w_t^{(\beta)}\propto p_t^{\beta}$ leaves the full-information posterior path unchanged and, by Theorem~\ref{thm:ts-exact}, leaves the expected payoff stationary at $\beta=1$. But the sampling variance $V^{\mathrm{sam}}_T(\beta):=\sum_t\Var_{w_t^{(\beta)}}(c_t)$ is not stationary there, so, to leading order in a Gaussian approximation of the realized payoff, the minimizer of any upper quantile is a law with $\beta^\ast\neq1$, which posterior sampling is not (Appendix~\ref{app:proofs-ts}).
\end{enumerate}
\end{proposition}

At $w_t=p_t$ the realized payoff is thus the deterministic minimax value plus a single additive martingale --- the same one for every comparator, and for any scoring choice.
Were that payoff constant over the sampling randomness, it could not depend on $I_{1:T}$ at all and would reduce to the deterministic game, in which the learner plays the mixture $p_t$ and does not sample.

\subsection{Further results from Section~\ref{sec:boosting}}
\label{app:proofs-boosting}

\subsubsection{The quantile-margin shortfall in full}
\label{app:quantile-margin}

Under the second-order schedule \eqref{eq:schedule} with $\Gamma=\log(1/\varepsilon)$ and $C=1/\sqrt2$, the envelope \eqref{eq:pathwise-regret} lets the $\varepsilon$-quantile margin fall short of twice the average edge $2\bar\gamma_T$ by at most
\[
\tfrac{2}{T}\Bigl(2\sqrt2\,\sqrt{\log(1/\varepsilon)\,V_T(\ell)}+\tfrac32\log(1/\varepsilon)+(1+\tfrac1{\sqrt2})\,Q_T^*(\ell)\Bigr).
\]
The range bound $Q_t(\ell)\le1/8$ for losses in $[0,1]$ caps $V_T(\ell)$ at $T/8$, so the worst-case shortfall is at most $2\sqrt{\log(1/\varepsilon)/T}$ plus the additive $O(\log(1/\varepsilon)/T)$ edge terms.

The pressure-target identity follows by telescoping the exponential-loss normalizer $Z_t:=N^{-1}\sum_i e^{-M_t(i)}$, which starts at $Z_0=1$ and ends at $Z_T=L_T^{\exp}$.
Under the local pressure rule of Section~\ref{sec:boosting} each ratio is $Z_t/Z_{t-1}=\inner{p_t}{e^{-\alpha_t g_t}}=e^{-\alpha_t a_t}$, so $Z_T=\prod_t e^{-\alpha_t a_t}$ and $-\log L_T^{\exp}=\sum_t\alpha_t a_t$.
Combining with the general identity \eqref{eq:boosting-variational} and summing the one-round transport identity against~$\nu$ yields the comparator statement of \eqref{eq:boost-pressure}.
The binary case, where the pressure-optimal $\alpha_t$ is the classical AdaBoost coefficient $\tfrac12\log((1-\eps_t)/\eps_t)$ with $\eps_t=\tfrac12-\gamma_t$, is Proposition~\ref{prop:adaboost-cgf}.

\begin{proposition}[AdaBoost rate from the pressure ledger]\label{prop:adaboost-cgf}
For binary weak hypotheses $g_t\in\{-1,+1\}^N$, take the one-step normalizer minimizer $\alpha_t^\star=\tfrac12\log\tfrac{1-\varepsilon_t}{\varepsilon_t}$, which is exactly the normalizer-minimizing (equivalently pressure-maximizing) choice $\alpha_t^\star=\arg\min_\alpha\log\inner{p_t}{e^{-\alpha g_t}}$.
With mean margin $\inner{p_t}{g_t}=2\gamma_t$, the pressure identity~\eqref{eq:boost-pressure} evaluates the per-round loss in closed form,
\begin{equation}\label{eq:adaboost-per-round}
\alpha_t^\star a_t
=-\log\inner{p_t}{e^{-\alpha_t^\star g_t}}
=-\tfrac12\log\!\bigl(1-4\gamma_t^2\bigr)
\end{equation}
the Gaussian-channel information at correlation $2\gamma_t$.
Summing~\eqref{eq:boost-pressure} gives the exact exponential-loss ledger
\begin{equation}\label{eq:adaboost-ledger}
-\log L_T^{\exp}
=\sum_{t=1}^T\alpha_t^\star a_t
=\sum_{t=1}^T-\tfrac12\log\!\bigl(1-4\gamma_t^2\bigr)
\end{equation}
Applying the convexity relaxation $-\log(1-x)\ge x$ at $x=4\gamma_t^2$ (the same relaxation by which variance and range sit above the centered RCGF) recovers the classical rate as a one-line corollary.
Under a uniform edge $\gamma_t\ge\gamma>0$,
\begin{equation}\label{eq:adaboost-rate}
-\log L_T^{\exp}\ \ge\ 2\sum_{t=1}^T\gamma_t^2\ \ge\ 2\gamma^2 T,
\qquad
\widehat{\mathrm{err}}_T=\inner{u_N}{\1\{M_T\le0\}}\le L_T^{\exp}\le e^{-2\gamma^2 T}
\end{equation}
the last step being the $\theta=0$ case of the margin-tail bound above.
\end{proposition}

\subsection{Further results from Section~\ref{sec:discussion}}
\label{app:further-discussion}

\subsubsection{The identities read in neighboring languages}
\label{app:neighboring-readings}

The retempering drift is a summation by parts of the log-partition's $\eta$-derivative: since $-\eta\,\partial_\eta A_t(\eta)=\eta^{-1}\KL(\rho_{t,\eta}\|\pi)$ (Proposition~\ref{prop:drift-continuous}), $D_T=\sum_{t=1}^{T-1}\bigl(A_t(\eta_t)-A_t(\eta_{t+1})\bigr)$ is a discrete reading of the path integral $-\int\partial_\eta A_t\,d\eta_t$ of prior-to-posterior relative entropy over the inverse-temperature schedule. Its continuous-time It\^o form---Theorem~\ref{thm:ito-unification}, of which \eqref{eq:control-J} is the deterministic reduction---turns the committed Game~A and reactive Game~B into dual stochastic-control problems on the information manifold.

The Bellman potential $W_\eta(S)$ is the negative free energy at inverse temperature $\eta$ with external field $S$, so the boosting identity $\sum_t\alpha_t a_t=-\log L_T^{\exp}$ commits its one-step increment round by round, holding pathwise instead of in expectation. The horizon quantities $R_T^c(\nu),\, P_T,\, D_T,\, B_T(\nu),\, V_T,\, Q_T^*$ index the axes along which the game relaxes---rate schedule, comparator complexity, sufficient-statistic choice, observation structure, per-round budget, and retempering pattern---so the game is summarized by a small vector of quantities. And bandit algorithms trade the observation-noise terms of \eqref{eq:bandit-decomp} against one another: EXP3-IX minimizes estimation bias at small cost to the martingale, while Tsallis-INF minimizes entropic information at the cost of more sample-dependent drift.

For an input prior $p_0$ that an external process pins, with no freedom to equalize, the \emph{operating-information deficit} of the redundancy--capacity reading is the exact Bregman divergence generated by the log-partition potential,
\begin{equation}
\label{eq:operating-info-deficit}
\mathsf C-I(p_0)=\KL\!\bigl(m_{p_0}\,\big\|\,q^\star\bigr)
+\sum_\theta p_0(\theta)\bigl[\mathsf C-\KL(P_\theta\|q^\star)\bigr]
\end{equation}
between the operating output mixture $m_{p_0}=\sum_\theta p_0(\theta)\,P_\theta$ and the capacity-achieving output $q^\star$ \cite{Csiszar75}. The second sum is nonnegative and vanishes exactly when $p_0$ is supported on the capacity-active inputs, those with $\KL(P_\theta\|q^\star)=\mathsf C$; there the deficit is the Bregman divergence alone.

\subsubsection{The exact-value residual at two experts}
\label{app:exact-value-residual}

The two solved slices of the one-round value are $g(\Gamma)\,\eta^{-1}\operatorname{arcosh}(e^{\eta^2 q})$ for two experts at every comparator budget (Proposition~\ref{prop:k2-exact}) and the root $a_K(\eta,q)$ of \eqref{eq:aK} for every $K$ at a saturated budget (Proposition~\ref{prop:saturated-exact}).
For $K=2$ with $\pi=(\tfrac12,\tfrac12)$ the terminal payoff is the exact support function $L_\Gamma(S)=\tfrac12(S_1+S_2)+\tfrac12\lvert S_1-S_2\rvert\,g(\Gamma)$,
maximized over nature's reachable states; that set solves a state-dependent recursion whose optimum over-tilts the Gibbs iterate, so Proposition~\ref{prop:k2-exact} only upper-bounds the value, with no elementary closed form for $T\ge 2$. The supremum is finite only once $T$ is fixed, since the per-round centered RCGF is second order in nature's step.

\subsection{Logarithmic pooling as the same coincidence game}
\label{app:log-pooling}
The same geometric-mean identity applies when the simplex is over probabilistic experts.
If expert $e\in[N]$ reports a predictive distribution $\pi_{t,e}$ on an outcome space~$Y$, and the learner pools the reports geometrically at weights $\alpha_t$,
\begin{equation}\label{eq:log-pool-rule}
p_t(y)
=
\frac{\prod_{e=1}^N \pi_{t,e}(y)^{\alpha_{t,e}}}{\sum_{z\in Y}\prod_{e=1}^N \pi_{t,e}(z)^{\alpha_{t,e}}}
\end{equation}
then for the realized outcome $y_t$,
\begin{equation}\label{eq:log-pooling}
-\log p_t(y_t)
=
\sum_{e=1}^N \alpha_{t,e}(-\log \pi_{t,e}(y_t))
-
C_{\alpha_t}(\pi_{t,1:N})
\end{equation}
where $C_{\alpha_t}(\pi_{t,1:N}):=-\log \sum_{y\in Y}\prod_{e=1}^N \pi_{t,e}(y)^{\alpha_{t,e}}$ is the coincidence discount.
Thus online logarithmic pooling is the same mixed-coincidence game on a different simplex: the learner's log loss is the weighted expert log loss minus a nonnegative coincidence discount.

Any online convex optimization guarantee for the meta-loss $g_t(\alpha):=-\sum_e\alpha_e\log\pi_{t,e}(y_t)+\log\sum_y\prod_e\pi_{t,e}(y)^{\alpha_e}$ yields an online logarithmic-pooling guarantee whose benchmark is stronger than the usual ``best convex combination of expert log losses'' because of the explicit coincidence discount.
The pooling weights~$\alpha_t$ can themselves be learned by the same entropic game.

\subsection{Exact relaxation remainders}
The relaxation remainders of the RCGF--variance--range hierarchy of Section~\ref{sec:hierarchy}, discussed in Section~\ref{sec:new-connections}, are given here in exact form.

\begin{proposition}[Exact relaxation remainders]
\label{prop:tower-remainder}
Let $\kappa_n(z)$ denote the $n$-th cumulant of the centered score $z$ under $p$ ($\kappa_2=\Var_p(z)$), equivalently the centered $n$-fold coincidence rate of \eqref{eq:mgf-coincidence}. Then
\begin{equation}\label{eq:tower-remainder-var}
Q_\eta(p,z)-\tfrac12\Var_p(z)
=\sum_{n\ge 3}\frac{\eta^{\,n-2}}{n!}\,\kappa_n(z)
=\frac{\eta}{6}\kappa_3(z)+\frac{\eta^2}{24}\kappa_4(z)+O(\eta^3)
\end{equation}
and, when $z(i)\in[a,b]$,
\begin{equation}\label{eq:tower-remainder-range}
\frac{\range(z)^2}{8}-Q_\eta(p,z)=\Delta\ \ge 0
\end{equation}
the exact nonnegative Hoeffding slack of \eqref{eq:classical-remainder}.
Writing $\Phi^{(1)}_t=\tfrac12\Var_{p_t}(z_t)$ and $\Phi^{(2)}_t=\range(z_t)^2/8$ for the two proxy potentials, the algorithm's exact gain along a trajectory, from using the true per-round RCGF where proxy~$j$ would be used, is $\sum_t\eta_t\bigl(\Phi^{(j)}_t-Q_t\bigr)$,
which for the range proxy is $\Delta_T^{\mathrm{class}}$ of \eqref{eq:classical-remainder} and for the variance proxy is $-\sum_t\eta_t\sum_{n\ge3}\eta_t^{\,n-2}\kappa_n(z_t)/n!$.
\end{proposition}

\section{Other geometries and domains}
\label{sec:other-geometries}

In the main development the simplex geometry was fixed: the prior~$\pi$ lives on $[K]$, the comparator class is a relative-entropy ball, and the regularizer is $\eta^{-1}\KL(\cdot\|\pi)$.
The same one-round equilibrium (Theorem~\ref{thm:fixed-bellman}) and the same exact decomposition (Theorem~\ref{thm:variable}) survive when the geometry is changed: weighting the entropy per-expert, replacing the prior by a continuum measure, or moving from a finite simplex to a compact convex set in $\R^d$.
In each case the recipe is the same: rewrite the prior anchor and the divergence in the new geometry; the Bellman equalizer is still a Gibbs-type best response; the per-round loss is still a centered RCGF; and the concentration-game bookkeeping persists.

\subsection{Weighted entropy and multiscale experts}

Let $\sigma_1,\dots,\sigma_K>0$ and define the weighted entropy geometry $F_\sigma(p):=\sum_{i=1}^K \sigma_i p(i)\log p(i)$ and the weighted divergence $D_\sigma(\nu\|p):=\sum_{i=1}^K \sigma_i(\nu(i)\log(\nu(i)/p(i))-\nu(i)+p(i))$.
The weighted-entropy mirror-descent step $\widetilde p_{t+1}(i):=p_t(i)e^{-\eta_t \ell_t(i)/\sigma_i}$,
$p_{t+1}(i)\propto \widetilde p_{t+1}(i)e^{-\lambda_t/\sigma_i}$ obeys the exact one-step identity
\begin{equation}\label{eq:weighted-entropy}
\inner{p_t}{\ell_t}-\inner{\nu}{\ell_t}
=
\frac{D_\sigma(\nu\|p_t)-D_\sigma(\nu\|p_{t+1})}{\eta_t}
+
\frac{D_\sigma(p_t\|p_{t+1})}{\eta_t}
\end{equation}
The same concentration bookkeeping persists after replacing relative entropy by the geometry dictated by the expert scales.
Multiscale entropy methods \cite{bubeck2019onlinemultiscale,PerezOrtizKoolen2022} are therefore the same game under a different transport geometry; the map $F_\sigma$ is the mirror map introduced with the multiscale experts problem \cite{bubeck2019onlinemultiscale}, where regret against each expert scales with that expert's own loss range.
Conditional refinements of the same weighted entropy, summed along a tree filtration, drive the entropic-regularization approach to metrical task systems \cite{coester2019pureentropic}, a competitive-analysis setting in which the comparator is dynamic and the learner's own transport is part of the cost.

The connection to the mixed coincidence identity is direct: the per-expert scales $\sigma_i$ play the role of mixture weights, and $D_\sigma$ replaces relative entropy as the transport divergence.

\subsection{Continuum priors and forgetting mixtures}

If $\pi_\chi$ is the exponential-weights prior obtained from a forgetting rate $\chi\in(0,1)$, then geometric pooling over $\chi$ produces a mixture of exponential memory kernels:
\begin{equation}\label{eq:forgetting}
\rho_t(i)
\propto
\exp\left(
-\eta\sum_{s=1}^{t-1} k_t(t-1-s)\,\ell_s(i)
\right),
\qquad
k_t(u)=\int \chi^u\,\alpha_t(d\chi)
\end{equation}
A prior over forgetting rates is therefore exactly a prior over memory kernels.
This is the continuum version of mixed coincidence, and it brings hyperparameter averaging inside the same concentration geometry, where it would otherwise sit as an outer model-selection layer.

More generally, if each $\theta$ indexes an expert class or a hyperparameter setting and $\pi_\theta$ is the action distribution that would be used under that setting, then the pooled distribution $p_\alpha^\star\propto\exp(\int_\Theta\log\pi_\theta(x)\alpha(d\theta))$ is the exact geometric mixture across the hyperparameter continuum.

The same pooling yields a multi-prior PAC-Bayes penalty. Any PAC-Bayes variational formula whose dependence on the prior enters only through the term $\KL(\nu\|\pi)$ admits an exact multi-prior version with penalty $\sum_w\alpha_w\KL(\nu\|\pi_w)-C_\alpha(\pi_{1:W})$, where $C_\alpha(\pi_{1:W})=-\log Z_\alpha$ is the coincidence term.
This is potentially useful for domain adaptation, meta-learning, or transfer settings with several plausible source priors.

\subsection{Continuous-action online convex optimization}

The same exact chain passes from sums to integrals.
On a compact convex set $S\subset\R^d$ with prior measure~$\pi$, define $p_t(dx)\propto e^{-\eta_t F_{t-1}(x)}\pi(dx)$,
$F_t(x):=\sum_{s=1}^t f_s(x)$, $x_t:=\int_S x\,p_t(dx)$.
Then the density regret against any posterior measure $\nu\ll\pi$ satisfies
\begin{equation}\label{eq:oco}
R_T^{\mathrm{dens}}(\nu)
=
D_T^{\mathrm{oco}}
+
B_T^{\mathrm{oco}}(\nu)
+
\sum_{t=1}^T \eta_t Q_t^{\mathrm{oco}}
\end{equation}
Convexity of $f_t$ turns density regret into ordinary OCO regret for the barycenters~$x_t$.

For a point comparator $x^\star\in\argmin_{x\in S}F_T(x)$ interior to~$S$, and with $\pi$ the uniform (normalized-Lebesgue) measure on the full-dimensional~$S$, the geometric-shrinking measure $\nu_{\varepsilon,x^\star}$ uniform on $(1-\varepsilon^{1/d})x^\star+\varepsilon^{1/d}S$ satisfies $\KL(\nu_{\varepsilon,x^\star}\|\pi)=\log(1/\varepsilon)$ and $\mathrm{Reg}_T(x^\star)\le \log(1/\varepsilon)+Q_T^{*,\mathrm{oco}}+(2C+C^{-1})\sqrt{V_T^{\mathrm{oco}}\log(1/\varepsilon)}+T\varepsilon^{1/d}$.
With $\varepsilon=T^{-d}$ and $C=1$ this gives $\mathrm{Reg}_T(x^\star)\le d\log T+Q_T^{*,\mathrm{oco}}+3\sqrt{dV_T^{\mathrm{oco}}\log T}+1$.
Since $Q_t^{\mathrm{oco}}\le 1/8$ for $f_t\in[0,1]$, $V_T^{\mathrm{oco}}\le T/8$ and $\mathrm{Reg}_T(x^\star)=O(\sqrt{dT\log T})$.
This is primarily of structural interest; standard projection-based and dual-averaging methods achieve dimension-independent rates by using $L_2$ geometry instead \cite{Zinkevich2003,xiao2010dual}.
\section{Continuous-time limit and It\^o-level treatment of the retempering drift}
\label{app:continuous-time}
The retempering drift has a deterministic continuous limit whose integrand is a path integral of relative entropy.

\begin{proposition}[Deterministic continuous limit of the retempering drift]
\label{prop:drift-continuous}
Fix the realized cumulative-loss trajectory and the prior $\pi$.
For every horizon index $t$ and every $\eta>0$, the mix loss \eqref{eq:A-potential} satisfies the exact derivative identity
\begin{equation}\label{eq:dA-KL}
\partial_\eta A_t(\eta)
=
-\frac{1}{\eta^2}\,\KL\!\bigl(\rho_{t,\eta}\,\big\|\,\pi\bigr)
\qquad\text{equivalently}\qquad
-\eta\,\partial_\eta A_t(\eta)
=
\frac{1}{\eta}\,\KL\!\bigl(\rho_{t,\eta}\,\big\|\,\pi\bigr)
\end{equation}
Consequently $A_t(\cdot)$ is nonincreasing in $\eta$, and for any schedule $(\eta_t)$ the drift admits the summation-by-parts expansion
\begin{equation}\label{eq:drift-abel}
D_T
=
\sum_{t=1}^{T-1}\bigl(A_t(\eta_t)-A_t(\eta_{t+1})\bigr)
=
\sum_{t=1}^{T-1}\frac{\KL(\rho_{t,\eta_t}\|\pi)}{\eta_t^2}\,\Delta\eta_t
+R_T,
\qquad \Delta\eta_t:=\eta_{t+1}-\eta_t
\end{equation}
with the explicit second-order remainder bound
\begin{equation}\label{eq:drift-remainder}
|R_T|
\le
\frac12\sum_{t=1}^{T-1}(\Delta\eta_t)^2
\sup_{\xi\in[\eta_t\wedge\eta_{t+1},\,\eta_t\vee\eta_{t+1}]}
\bigl|\partial_\eta^2 A_t(\xi)\bigr|
\end{equation}
If in addition the schedule is sampled from a fixed $C^2$ control $\eta\colon[0,1]\to(0,\infty)$ as $\eta_t=\eta(t/T)$, and the realized cumulative loss stabilizes in rescaled time $s=t/T$ (a bounded $C^1$ field $\bar C\colon[0,1]\to\R^K$ with $C_t=\bar C(t/T)$), then $\sup_{t,T}|\partial_\eta^2 A_t|<\infty$, the remainder is $R_T=O(1/T)$, and
\begin{equation}\label{eq:drift-stieltjes}
D_T
\;\xrightarrow[T\to\infty]{}\;
-\!\int_{0}^{1}\partial_\eta A\,\,d\eta
=
\int_{0}^{1}
\frac{\KL\!\bigl(\rho_{\bar C(s),\,\eta(s)}\,\big\|\,\pi\bigr)}{\eta(s)^2}\,
\eta'(s)\,ds
\end{equation}
a Riemann--Stieltjes integral of the prior-to-posterior relative entropy over the controlled inverse-temperature path.
\end{proposition}

Identity \eqref{eq:dA-KL} makes the sign rule transparent: the retempering loss per unit of inverse temperature is exactly the rescaled prior-to-posterior relative entropy $\eta^{-2}\KL(\rho_{t,\eta}\|\pi)$, so warming ($\Delta\eta_t\le0$) is always favorable, and \eqref{eq:drift-stieltjes} exhibits $D_T$ as the relative-entropy work integrated along the inverse-temperature path.
As $\eta\downarrow0$ the integrand reduces to $\tfrac12\Var_\pi(\bar C(s))$, the same $\tfrac12\Var$ small-$\eta$ limit that governs the second-order schedule \eqref{eq:schedule}.
The optimal-schedule problem then has a deterministic optimal-control reading: along the realized trajectory, minimizing the schedule-dependent part of the regret identity \eqref{eq:exact-variable} over a $C^2$ control $\eta(\cdot)$ is the optimal-control (Bolza) problem
\begin{equation}\label{eq:control-J}
\min_{\eta(\cdot)}\;
\int_0^1 \eta(s)\,\dot V(s)\,ds
\;+\;
\int_0^1 \frac{\KL\!\bigl(\rho_{\bar C(s),\eta(s)}\|\pi\bigr)}{\eta(s)^2}\,\eta'(s)\,ds
\end{equation}
with $\dot V(s)$ the rate of intrinsic-time accumulation, $\eta(s)$ the controlled parameter, and the relative entropy the running cost; its small-$\eta$ stationarity recovers $\eta^\star\approx\sqrt{\Gamma/V}$.
\begin{theorem}[It\^o-level form of the retempering--drift duality]
\label{thm:ito-unification}
On the scale--score manifold $\{(\eta,S):\eta>0,\ S\in\R^K\}$ the log-partition potential $W_\eta(S)=\eta^{-1}\log\inner{\pi}{e^{\eta S}}$ has the exact differential
\begin{equation}\label{eq:dW-exact}
dW=\frac{\KL(\rho_{\eta,S}\|\pi)}{\eta^2}\,d\eta+\inner{\rho_{\eta,S}}{dS}
\end{equation}
with gradient $\nabla_S W_\eta(S)=\rho_{\eta,S}$ the Gibbs law~\eqref{eq:gibbs-map}, scale derivative $\partial_\eta W_\eta(S)=\eta^{-2}\KL(\rho_{\eta,S}\|\pi)$, and Hessian $\nabla_S^2 W_\eta(S)=\eta\,\mathrm{Cov}_{\rho_{\eta,S}}\succeq0$ (the tilted-variance representation~\eqref{eq:tilted-var}). Consequently:
\begin{enumerate}[label=(\roman*),topsep=2pt,itemsep=1pt]
\item The one-form~\eqref{eq:dW-exact} is closed, hence exact on the simply connected manifold, and the terminal collapse~\eqref{eq:loss-drift-telescope} is its endpoint evaluation, $\sum_t\eta_t Q_t+D_T=W_{\eta_T}(S_T)$.
\item For $S$ a continuous $\R^K$-semimartingale and $\eta$ a strictly positive continuous predictable finite-variation process, It\^o's formula for the smooth $W$ collapses to the exact three-leg identity
\begin{equation}\label{eq:ito-collapse}
W_{\eta_u}(S_u)-W_{\eta_0}(S_0)=\int_0^u\inner{\rho_{\eta_s,S_s}}{dS_s}
+\tfrac12\int_0^u\eta_s\,\tr\!\bigl(\mathrm{Cov}_{\rho_{\eta_s,S_s}}\,d[S^c]_s\bigr)
+\int_0^u\frac{\KL(\rho_{\eta_s,S_s}\|\pi)}{\eta_s^{2}}\,d\eta_s
\end{equation}
whose first leg vanishes along the game's centered trajectory --- there $\inner{p_t}{z_t}=0$ and the equalizer play is $p_t=\rho_{\eta_t,S_{t-1}}$ --- and whose second is the intrinsic-time accumulation,
\begin{equation}\label{eq:loss-split}
\int_0^u\inner{\rho_{\eta_s,S_s}}{dS_s}=0,
\qquad
\tfrac12\int_0^u\eta_s\,\tr\!\bigl(\mathrm{Cov}_{\rho_{\eta_s,S_s}}\,d[S^c]_s\bigr)=\int_0^u\eta_s\,dV_s
\end{equation}
\item On a one-step budget $\{Q_\eta(p,z)\le q\}$ the increment $W_\eta(S+z)-W_\eta(S)$ is independent of nature's direction, equal to $\eta q$ at the budget, if and only if $p=\rho_{\eta,S}$~\eqref{eq:bellman-increment}.
\end{enumerate}
So \eqref{eq:ito-collapse} is the continuous form of the terminal collapse~\eqref{eq:loss-drift-telescope}; the sampling martingale $M^{\mathrm{sam}}$ of~\eqref{eq:ts} is not a leg of $dW$, but the separate additive term in $\widehat R_T^c(\nu)=M_T^{\mathrm{sam}}+\inner{\nu}{S_T}$. The two Stieltjes legs identify Game~A with the predictable control problem $\min_{\eta(\cdot)}\E\bigl[\int\eta\,dV+\int\eta^{-2}\KL(\rho_{\eta,S}\|\pi)\,d\eta\bigr]$, whose variance-frozen deterministic reduction is the Bolza problem in~\eqref{eq:control-J}.
\end{theorem}

Exactness is path independence: $W_{\eta_T}(S_T)-W_{\eta_0}(S_0)$ depends only on the endpoints, and the loss leg of fixed-$\eta$ increments and the drift leg of fixed-$S$ scale switches telescope to that endpoint. Game~A, in which the scale is committed before the move, and Game~B, in which it is chosen reactively (Section~\ref{sec:pressure}), are therefore the two traversal orders of the same exact form and share the identical terminal value.
Continuity and finite variation of $\eta$ make the quadratic variation $[\eta]$ and, by Kunita--Watanabe, the covariation $[\eta,S]$ vanish; that drops every scale-side It\^o correction and leaves the drift leg a pure Riemann--Stieltjes integral against $d\eta$, the continuous form of~\eqref{eq:drift-stieltjes}.
Because the learner enters the loss leg only through the integrand $\nabla_S W=\rho_{\eta,S}$, the equalizer is to play the gradient of the potential, and the telescoping is the fundamental theorem of calculus for that potential, with no backward recursion.
Game~B is the Lagrangian dual of the predictable control problem, with the pressure target the multiplier that opens $\eta$ to activate the one-step constraint.
The deterministic continuous limit is therefore rigorous, and the measure-theoretic stochastic interface is the exactness of the one-form $dW$ on the scale--score manifold: Theorem~\ref{thm:ito-unification} carries the retempering-drift duality to the It\^o level.

\section{Proofs}
\label{app:proofs}

The proofs are grouped by the section their statements serve. The potential-games stabilization of Section~\ref{sec:games} is proved for strict-equilibrium-convergent traces (Theorem~\ref{thm:unconditional-V}), with only the convergence of the dynamics itself left open there. The two-sided envelope (Theorem~\ref{thm:sqrt-envelope}) brackets the per-round-RCGF loss isolated by the identity \eqref{eq:exact-variable}: its upper inequality is derived in full below, and its lower inequality is the AM--GM/online-equalizer bound $\sum_t\eta_t Q_t\ge 2C\sqrt{\Gamma V_T}-C^2\Gamma$ (round-by-round equalizing of $\Gamma/\eta+\eta V\ge 2\sqrt{\Gamma V}$, Section~\ref{sec:variable}), attained with equality at the one-round initialization witness of Lemma~\ref{rem:minimax}.
The envelope gap is exactly the necessary edge terms $Q_T^\star(c)+C^2\Gamma$ of that lemma, and is no artifact of the analysis. The almost-sure iterated-logarithm statement of Section~\ref{sec:ts} norms $M_T^{\mathrm{sam}}$ by the exact law-of-the-iterated-logarithm scale $\sqrt{2V^{\mathrm{sam}}_T\log\log V^{\mathrm{sam}}_T}$.
Its fixed-time and anytime forms are derived directly, and the a.s.\ form is the standard method-of-mixtures bound, valid once the mixing density is supported on the sub-gamma range $\lambda\in(0,3)$ on which the exponential process is a supermartingale.

\subsection{Proofs from Section~\ref{sec:primitive}}
\label{app:pf-primitive}

\begin{proof}[Proof of Proposition~\ref{prop:centered}.]
\label{app:proof-centered}
The centered move $z_s(i)=\inner{p_s}{c_s}-c_s(i)$ telescopes: $\sum_{s=1}^{t-1}z_s(i)=\sum_{s=1}^{t-1}\mu_s-C_{t-1}(i)$, which is the first display.
For the Bayes form, divide \eqref{eq:bayes-update} numerator and denominator by $\exp(-\eta_t\sum_{s=1}^{t-1}\mu_s)$ (a constant in~$i$) to replace $-\eta_t C_{t-1}(i)$ by $\eta_t S_{t-1}(i)$.
For the regret identity, $\inner{p_s-\nu}{c_s}=\inner{\nu}{\mu_s\1-c_s}=\inner{\nu}{z_s}$ since $\inner{\nu}{\1}=1$; summing over $s$ and using $\sum_s z_s=S_T$ gives the result.
\end{proof}

\begin{proof}[Proof of Proposition~\ref{prop:terminal-dual}.]
\label{app:proof-terminal-dual}
The Gibbs variational identity \eqref{eq:gibbs-var} is standard: the supremum on its right-hand side is attained at $q=\rho_{\eta,S}$ by direct differentiation, with optimum value $W_\eta(S)$.
For \eqref{eq:terminal-dual}, the constrained primal $L_\Gamma(S)=\sup\{\inner{\nu}{S}:\nu\in\DeltaK,\KL(\nu\|\pi)\le\Gamma\}$ is a convex maximization with Slater's condition (taking $q=\pi$ gives $\KL=0<\Gamma$ for $\Gamma>0$). Writing the Lagrangian $L(\nu,\lambda)=\inner{\nu}{S}-\lambda(\KL(\nu\|\pi)-\Gamma)$ for $\lambda\ge 0$, strong duality gives
\[
L_\Gamma(S)=\inf_{\lambda\ge 0}\sup_{\nu\in\DeltaK}L(\nu,\lambda)=\inf_{\lambda\ge 0}\Bigl\{\lambda\Gamma+\sup_{\nu\in\DeltaK}\bigl(\inner{\nu}{S}-\lambda\KL(\nu\|\pi)\bigr)\Bigr\}.
\]
Setting $\eta:=1/\lambda$ (and observing that $\lambda=0$ gives $+\infty$ unless $S=0$, so the infimum is over $\eta>0$), the inner supremum equals $W_\eta(S)$ by \eqref{eq:gibbs-var}, yielding $L_\Gamma(S)=\inf_{\eta>0}\{\Gamma/\eta+W_\eta(S)\}$. The infimum formula holds for every $S$ and $\Gamma\ge0$.
When the comparator constraint is active at the optimum---$\Gamma$ strictly below $\lim_{\eta\to\infty}\KL(\rho_{\eta,S}\|\pi)$, the maximal relative entropy reachable by tilting along~$S$---it is attained at the unique finite $\eta$ solving $\eta^2\partial_\eta W_\eta(S)=\KL(\rho_{\eta,S}\|\pi)=\Gamma$.
For constant~$S$ (every tilt leaves $\rho_{\eta,S}=\pi$) or $\Gamma$ at or above that limit it is attained as $\eta\to\infty$ with value $\max_i S(i)$, and for $\Gamma=0$ as $\eta\to0^+$ with value $\inner{\pi}{S}$.
\end{proof}

\begin{proof}[Proof of Proposition~\ref{prop:robust-bayes}.]
\label{app:proof-robust-bayes}
(a) The functional $\nu\mapsto\inner{\nu}{S}-\eta^{-1}\KL(\nu\|\pi)$ is strictly concave in~$\nu$ (because $\KL(\cdot\|\pi)$ is strictly convex on the open simplex), and $\rho_{\eta,S}$ has positive entries, so the unique stationary point on the relative interior of $\DeltaK$ is the unique maximizer.
Direct differentiation of $\nu(i)\log(\nu(i)/\pi(i))-\eta\nu(i) S(i)$ subject to $\sum_i\nu(i)=1$ gives $\nu(i)\propto\pi(i)e^{\eta S(i)}$, which is~$\rho_{\eta,S}$.
(b) The constrained problem $\min_\nu\KL(\nu\|\pi)$ subject to $\inner{\nu}{S}=m_\eta$ has Lagrangian $L(\nu,\lambda)=\KL(\nu\|\pi)+\lambda(\inner{\nu}{S}-m_\eta)$.
Stationarity $\partial L/\partial \nu(i)=\log(\nu(i)/\pi(i))+\lambda S(i)=0$ gives $\nu(i)\propto\pi(i)e^{-\lambda S(i)}$.
The constraint $\inner{\nu}{S}=m_\eta=\inner{\rho_{\eta,S}}{S}$ then forces $\lambda=-\eta$, so $\nu=\rho_{\eta,S}$.
Strict convexity of $\KL$ on the simplex makes the minimizer unique.
\end{proof}

\begin{proof}[Proof of Proposition~\ref{prop:tilted-var}.]
\label{app:proof-tilted-var}
Let $\psi(s):=\log\sum_i p(i)e^{s\eta z(i)}$.
Then $\psi(0)=0$, $\psi'(s)=\eta\sum_i z(i)\,p_s^{(\eta,z)}(i)$, and the centering $\inner{p}{z}=0$ gives $\psi'(0)=0$.
Differentiating once more, $\psi''(s)=\eta^2\Var_{p_s^{(\eta,z)}}(z)$.
Taylor's theorem with integral remainder at $s=0$ yields $\psi(1)=\int_0^1(1-s)\psi''(s)\,ds=\eta^2\int_0^1(1-s)\Var_{p_s^{(\eta,z)}}(z)\,ds$, and dividing by $\eta^2$ produces \eqref{eq:tilted-var}.
Nonnegativity is immediate from the integrand being nonnegative; equality forces $\Var_{p_s^{(\eta,z)}}(z)=0$ for almost every $s\in(0,1)$, which (continuity in $s$) forces $z$ to be $p$-a.s.\ constant.
\end{proof}

\begin{proof}[Proof of Corollary~\ref{cor:variance-shadow}.]
\label{app:proof-variance-shadow}
By Proposition~\ref{prop:tilted-var}, $Q_\eta(p,z)=\int_0^1(1-s)\Var_{p_s^{(\eta,z)}}(z)\,ds$.
As $\eta\downarrow 0$, the tilted measure $p_s^{(\eta,z)}\to p$ uniformly in $s\in[0,1]$, so the integrand converges pointwise to $(1-s)\Var_p(z)$, and dominated convergence (the integrand is bounded by $\range(z)^2/4$) yields $Q_\eta(p,z)\to\int_0^1(1-s)\,ds\cdot\Var_p(z)=\frac12\Var_p(z)$.
For the range bound, when $z(i)\in[a,b]$, Hoeffding's lemma gives $\Var_{p_s^{(\eta,z)}}(z)\le(b-a)^2/4$, and integrating $\int_0^1(1-s)\,ds=1/2$ delivers $Q_\eta(p,z)\le(b-a)^2/8$.
\end{proof}

\begin{proof}[Proof of Corollary~\ref{cor:rcgf-influence}.]
\label{app:proof-rcgf-influence}
Write $F(q):=\log\sum_i q(i)e^{\eta c(i)}-\eta\inner{q}{c}$ for $q\in\DeltaK$; expanding the centered exponent shows $F(q)=\eta^2 Q_\eta(q,c-\inner{q}{c})$, the functional of the statement.
Both terms of $F(p_\epsilon)$ are differentiable in $\epsilon$ at $0$:
\[
\frac{d}{d\epsilon}\Big|_{0^+}\log\sum_i p_\epsilon(i)e^{\eta c(i)}
=\frac{e^{\eta c(j)}-\sum_i p(i)e^{\eta c(i)}}{\sum_i p(i)e^{\eta c(i)}},
\qquad
\frac{d}{d\epsilon}\Big|_{0^+}\eta\inner{p_\epsilon}{c}
=\eta\bigl(c(j)-\inner{p}{c}\bigr)=\eta z(j)
\]
Multiplying numerator and denominator of the first display by $e^{-\eta\inner{p}{c}}$ turns it into $e^{\eta z(j)}/\sum_i p(i)e^{\eta z(i)}-1=p_1(j)/p(j)-1$, and subtracting the second display gives \eqref{eq:rcgf-influence}.
The mean is $\sum_j p(j)\,\xi_\eta(j)=1-1-\eta\inner{p}{z}=0$.
With $r(j):=p_1(j)/p(j)$, the constant shift drops from the variance, so $\Var_p(\xi_\eta)=\Var_p(r-\eta z)=\Var_p(r)+\eta^2\Var_p(z)-2\eta\,\mathrm{Cov}_p(r,z)$; here $\Var_p(r)=\chi^2(p_1\|p)$ since $\E_p[r]=1$, and $\mathrm{Cov}_p(r,z)=\E_p[rz]=\inner{p_1}{z}$ since $\inner{p}{z}=0$.
For the bounds, $\E_p[e^{\eta z}]\ge e^{\eta\inner{p}{z}}=1$ by Jensen's inequality, so $0<r(j)\le e^{\eta z(j)}\le e^{\eta b}$, and with $z(j)\in[a,b]$ the two-sided bound on $\xi_\eta(j)=r(j)-1-\eta z(j)$ follows.
\end{proof}

\subsection{Proofs from Section~\ref{sec:fixed}}
\label{app:pf-fixed}

\begin{proof}[Proof of Theorem~\ref{thm:fixed-bellman}.]
\label{app:proof-fixed-bellman}
The Bellman increment identity $W_\eta(S+z)-W_\eta(S)=\eta Q_\eta(\rho_{\eta,S},z)$ is direct: by \eqref{eq:Q-centered},
\[
\eta(W_\eta(S+z)-W_\eta(S))=\log\frac{\sum_i\pi(i)e^{\eta(S(i)+z(i))}}{\sum_i\pi(i)e^{\eta S(i)}}=\log\sum_i [\rho_{\eta,S}]_i\,e^{\eta z(i)}=\eta^2 Q_\eta(\rho_{\eta,S},z).
\]
Under the budget $Q_\eta(p,z)\le q_t$ this is bounded by $\eta q_t$, giving \eqref{eq:bellman-increment}.
For the relaxation property, define $U_t$ as in \eqref{eq:fixed-bellman} and check that under the learner's choice $p=\rho_{\eta,S}$:
$U_T(S+z)\le U_{T+1}(S+z)+\eta q_T = \Gamma/\eta+W_\eta(S+z)+\eta q_T \le \Gamma/\eta+W_\eta(S)+\eta(q_{T-1}+q_T)= U_{T-1}(S)$ for the previous round (after re-indexing), and the same step iterates back to $U_1(0)=\Gamma/\eta+W_\eta(0)+\eta\sum_{s=1}^Tq_s=\Gamma/\eta+\eta\sum_sq_s$ (since $W_\eta(0)=\eta^{-1}\log\sum_i\pi(i)=0$).
The terminal payoff is bounded by $L_\Gamma(S_T)\le\Gamma/\eta+W_\eta(S_T)=U_{T+1}(S_T)$ by Proposition~\ref{prop:terminal-dual}, so $\Val_t^{\eta,\Gamma}(S)\le U_t(S)$ at every node, and the equalizer is the Gibbs distribution by the saturating choice in $W_\eta(S+z)-W_\eta(S)$.
\end{proof}

\begin{proof}[Proof of Corollary~\ref{cor:fixed-budget}.]
\label{app:proof-fixed-budget}
By Theorem~\ref{thm:fixed-bellman}, $\Val_1^{\eta,\Gamma}(0)\le U_1(0)=\Gamma/\eta+\eta\sum_{t=1}^T q_t\le\Gamma/\eta+\eta V$.
The AM--GM inequality $\Gamma/\eta+\eta V\ge 2\sqrt{\Gamma V}$ for $\eta,\Gamma,V>0$ is standard and tight at $\eta=\sqrt{\Gamma/V}$.
\end{proof}

\begin{proof}[Proof of Proposition~\ref{prop:fixed-transport}.]
\label{app:proof-fixed-transport}
Fix $\nu\in\DeltaK$.
A direct computation, using \eqref{eq:Q-centered}, gives $\log p^+(i)-\log p(i)=\eta z(i)-\log\sum_j p(j)e^{\eta z(j)}=\eta(z(i)-\eta Q_\eta(p,z))$.
Multiplying by $\nu(i)$ and summing,
\[
\KL(\nu\|p)-\KL(\nu\|p^+)=\sum_i\nu(i)\log\frac{p^+(i)}{p(i)}=\eta\inner{\nu}{z}-\eta^2 Q_\eta(p,z),
\]
which rearranges to \eqref{eq:fixed-transport}.
The cumulative form \eqref{eq:fixed-cumulative} is the telescoping sum: $\sum_{t=1}^T(\KL(\nu\|p_t)-\KL(\nu\|p_{t+1}))=\KL(\nu\|\pi)-\KL(\nu\|p_{T+1})$ at fixed scale.
\end{proof}

\begin{proof}[Proof of Proposition~\ref{prop:cgf-geometry}.]
\label{app:proof-cgf-geometry}
(a) The function $c\mapsto\eta_t^{-2}\psi_{p_t}(-\eta_t;c)=\eta_t^{-2}\log\sum_ip_t(i)e^{-\eta_t(c(i)-\inner{p_t}{c})}$ is a log-sum-exp in $c$, hence convex; its sublevel set is closed and convex.
(b) Adding $b\1$ to $c$ leaves $c(i)-\inner{p_t}{c}=c(i)-\inner{p_t}{c}+b-b\inner{p_t}{\1}=c(i)-\inner{p_t}{c}$ unchanged, so $\psi_{p_t}$ is invariant.
(c) Constant vectors $c=b\1$ have $\psi_{p_t}(-\eta_t;b\1)=0$, so they lie in every sublevel set.
(d) The defining expression depends only on $p_t$, $\eta_t$, $\beta$, and $c$; the prior $\pi$ does not appear.
\end{proof}

\begin{proof}[Proof of Proposition~\ref{prop:nature-lb}.]
\label{app:proof-nature-lb}
For every $p\in\DeltaK$ the move $\zeta(p)$ is feasible, so $L_\Gamma(S+\zeta(p))\le\sup_{z\in\cZ_{\eta,q}(p)}L_\Gamma(S+z)$.
Taking the infimum over $p$ on both sides preserves the inequality and gives \eqref{eq:nature-lb}.
If $\zeta(p)$ attains the inner supremum at every $p$, the two sides agree pointwise in $p$ and hence after the infimum.
The final claim holds because $\inner{\nu}{S+z}\le L_\Gamma(S+z)$ for every $\nu$ in the comparator ball.
\end{proof}

\begin{proof}[Proof of Proposition~\ref{prop:saturated-reduction}.]
$\KL(\cdot\|\pi)$ is convex on $\DeltaK$, so it attains its simplex maximum at a vertex, where its value is $\max_i\log(1/\pi(i))=\Gamma_{\max}$. For $\Gamma\ge\Gamma_{\max}$ the constraint $\KL(\nu\|\pi)\le\Gamma$ is therefore vacuous, the supremum in \eqref{eq:terminal-support} runs over all of $\DeltaK$, and it is attained at a point mass on an $\argmax_i S(i)$, giving \eqref{eq:Lmax}. The terminal payoff then does not involve $\Gamma$, and neither does the value of the game it terminates.
\end{proof}

\begin{proof}[Proof of Proposition~\ref{prop:saturated-exact}.]
\label{app:proof-saturated-exact}
Throughout, $L_\Gamma=\max_i(\cdot)$ by Proposition~\ref{prop:saturated-reduction}, and $S_0=0$, so nature's payoff at the move $z$ is $\max_i z(i)$.
Write $F(a,m):=m\,e^{\eta a}+(1-m)\,e^{-\eta m a/(1-m)}$ for $a\ge0$ and $m\in(0,1)$, and note $F(0,m)=1$.

\emph{Existence and uniqueness of $a_K$.}
The left side of \eqref{eq:aK} is $F(a,1/K)$, which equals $1$ at $a=0$, has $\partial_a F(a,m)=m\eta\bigl(e^{\eta a}-e^{-\eta m a/(1-m)}\bigr)>0$ for $a>0$ at every $m\in(0,1)$ --- at $m=1/K$, $\tfrac{\eta}{K}\bigl(e^{\eta a}-e^{-\eta a/(K-1)}\bigr)$ --- and diverges as $a\to\infty$; since $e^{\eta^2q}>1$ for $q>0$, there is exactly one root $a_K>0$.

\emph{Upper bound.}
Let the learner play the uniform $u_K$, so that nature's constraints are $\sum_i z(i)=0$ and $\tfrac1K\sum_i e^{\eta z(i)}\le e^{\eta^2 q}$.
Fix a coordinate $j$ and consider moves with $z(j)=a>0$.
The remaining coordinates satisfy $\sum_{i\neq j}z(i)=-a$, and by convexity of $\exp$ the sum $\sum_{i\neq j}e^{\eta z(i)}$ is minimized subject to that constraint at the common value $z(i)=-a/(K-1)$, where it equals $(K-1)e^{-\eta a/(K-1)}$.
Feasibility of $z(j)=a$ therefore forces $F(a,1/K)\le e^{\eta^2 q}$, so $a\le a_K$ by the monotonicity just established.
The move placing $a_K$ on one coordinate and $-a_K/(K-1)$ on the others meets both constraints with equality, so $\sup_{z\in\cZ_{\eta,q}(u_K)}\max_i z(i)=a_K$ and $\mathrm{Val}_1\le a_K$.

\emph{Monotonicity of $F$ in $m$.}
Fix $a>0$ and set $v:=\eta a/(1-m)>0$, so that the exponent of the second term is $-\eta m a/(1-m)=-vm$ and $\eta a=v(1-m)$. Differentiating,
\[
\partial_m F(a,m)
=e^{\eta a}-e^{-vm}-(1-m)e^{-vm}\frac{\eta a}{(1-m)^2}
=e^{v(1-m)}-e^{-vm}\Bigl(1+v\Bigr)
=e^{-vm}\bigl(e^{v}-1-v\bigr)>0 .
\]
Hence $F(\cdot,m)$ is strictly increasing in $a$ and $F(a,\cdot)$ strictly increasing in $m$, so the root $a(m)$ of $F(a,m)=e^{\eta^2q}$ is strictly decreasing in $m$.

\emph{Lower bound.}
If $p(j)=0$ for some $j$, the move $z$ equal to $\lambda$ on coordinate $j$ and $0$ elsewhere satisfies $\inner{p}{z}=0$ and $\sum_ip(i)e^{\eta z(i)}=1$ for every $\lambda>0$, so nature's payoff $\max_i z(i)=\lambda$ is unbounded and such $p$ is never optimal.
For $p$ in the interior, define the response rule $\zeta$: pick $j\in\argmin_i p(i)$, set $m:=p(j)\in(0,1/K]$, and let $\zeta(p)$ place $a(m)$ on coordinate $j$ and $-m\,a(m)/(1-m)$ on every other coordinate.
Then $\inner{p}{\zeta(p)}=m\,a(m)-(1-m)\tfrac{m\,a(m)}{1-m}=0$ and $\sum_ip(i)e^{\eta\zeta(p)(i)}=F(a(m),m)=e^{\eta^2 q}$, so $\zeta(p)\in\cZ_{\eta,q}(p)$ with the budget saturated, and $\max_i\zeta(p)(i)=a(m)$.
Since $m=\min_i p(i)\le 1/K$ and $a(\cdot)$ is decreasing, $a(m)\ge a(1/K)=a_K$.
Proposition~\ref{prop:nature-lb} then gives $\mathrm{Val}_1\ge\inf_p a(\min_i p(i))=a_K$.
Combining the two bounds gives \eqref{eq:saturated-exact}, attained at $p=u_K$.
Neither bound refers to $\pi$, which enters only through the hypothesis $\Gamma\ge\Gamma_{\max}$.

\emph{Small-budget expansion.}
Expanding \eqref{eq:aK} at $a=0$ gives $F(a,1/K)=1+\tfrac{(\eta a)^2}{2(K-1)}+O\bigl((\eta a)^3\bigr)$ and $e^{\eta^2q}=1+\eta^2q+O(q^2)$, from which $a_K=\sqrt{2(K-1)q}+O(q)$.
\end{proof}

\begin{proof}[Proof of Proposition~\ref{prop:k2-exact}.]
The value is bracketed from the two sides as in Proposition~\ref{prop:saturated-exact}.
For the upper bound let the learner play $p=\pi=(\tfrac12,\tfrac12)$, which at $S=0$ is also the Gibbs equalizer of Theorem~\ref{thm:fixed-bellman}, and evaluate nature's best reply, $\max_z L_\Gamma(z)$ over $z\in\cZ_{\eta,q}(p)$.
Centering $\inner{p}{z}=0$ forces $z=(\delta,-\delta)$, and $Q_\eta(p,z)=\eta^{-2}\log\cosh(\eta\delta)$ is even and strictly increasing in $|\delta|$, so the budget binds at nature's best response: $\log\cosh(\eta\delta)=\eta^2 q$, i.e.\ $\delta^\star=\eta^{-1}\operatorname{arcosh}(e^{\eta^2 q})$.

For the terminal payoff, $\inner{\nu}{(\delta,-\delta)}=\delta(2\nu_1-1)$, so for $\delta>0$ the supremum defining $L_\Gamma$ drives $\nu_1$ to the boundary of the relative-entropy ball.
With $\pi$ uniform, $\KL(\nu\|\pi)=\log 2-h(\nu_1)$ with $h$ the binary entropy, so the ball is $\{\nu_1:\log 2-h(\nu_1)\le\Gamma\}=\bigl[\tfrac{1-g(\Gamma)}{2},\tfrac{1+g(\Gamma)}{2}\bigr]$ with $g(\Gamma)$ as in~\eqref{eq:k2-g}.
The maximizer is $\nu_1=\tfrac{1+g(\Gamma)}{2}$, giving $L_\Gamma\bigl((\delta,-\delta)\bigr)=\delta\,g(\Gamma)$; the move $\delta<0$ gives the equal value $|\delta|\,g(\Gamma)$ by symmetry of $\pi$.
Evaluating at $\delta^\star$ yields the upper bound $\mathrm{Val}_1\le g(\Gamma)\,\delta^\star$.

For the matching lower bound, apply Proposition~\ref{prop:nature-lb} with the response rule of Proposition~\ref{prop:saturated-exact}: against $p$ with $m:=\min_i p(i)\le\tfrac12$ attained at coordinate $j$, nature plays $a(m)$ on $j$ and $-m\,a(m)/(1-m)$ on the other coordinate, where $a(m)$ solves $F(a,m)=e^{\eta^2q}$ (a boundary $p$ again concedes an unbounded payoff).
Writing $b:=m\,a(m)/(1-m)\in[0,a(m)]$, the terminal support function $L_\Gamma$ evaluates to
\[
L_\Gamma\bigl(\zeta(p)\bigr)=\frac{a(m)-b}{2}+\frac{a(m)+b}{2}\,g(\Gamma)\;\ge\;a(m)\,g(\Gamma),
\]
the inequality because $g(\Gamma)\le1$ and $b\le a(m)$.
Since $a(\cdot)$ is decreasing and $m\le\tfrac12$, $a(m)\ge a(\tfrac12)=\delta^\star$, so $\mathrm{Val}_1\ge g(\Gamma)\,\delta^\star$ and~\eqref{eq:k2-value} follows.
Theorem~\ref{thm:fixed-bellman} is not what makes the Gibbs play optimal here: it supplies the relaxation $U_t$, and the equalizer it identifies is minimax at $S=0$ under a uniform prior but not under a general one (Proposition~\ref{prop:saturated-exact}).

For~\eqref{eq:k2-gap}, Corollary~\ref{cor:fixed-budget} gives $\mathrm{Val}_1\le\Gamma/\eta+\eta q$; writing $t:=\eta^2 q>0$, the gap equals $\eta^{-1}f(t)$ with $f(t):=\Gamma+t-g(\Gamma)\operatorname{arcosh}(e^{t})$ by~\eqref{eq:k2-value}.
Here $f(0)=\Gamma>0$ and $f'(t)=1-g(\Gamma)/\sqrt{1-e^{-2t}}$ increases from $-\infty$ (as $t\to0^+$) to $1-g(\Gamma)$ (as $t\to\infty$).
If $\Gamma\ge\log 2$ then $g(\Gamma)=1$, so $f'<0$ throughout and $f$ decreases strictly to its infimum $\Gamma-\log 2$, approached only as $t\to\infty$ (since $\operatorname{arcosh}(e^{t})=t+\log 2+o(1)$); hence $f(t)>\Gamma-\log 2\ge 0$ at every finite $t$ and the bound is strict.
If $\Gamma<\log 2$ then $g(\Gamma)<1$, so $f'$ vanishes exactly once, at the interior minimizer $t_\star=-\tfrac12\log\bigl(1-g(\Gamma)^2\bigr)$, where $e^{t_\star}=(1-g(\Gamma)^2)^{-1/2}$ and $\operatorname{arcosh}(e^{t_\star})=\tfrac12\log\tfrac{1+g(\Gamma)}{1-g(\Gamma)}$; substituting and using the definition $\Gamma=\tfrac12(1+g(\Gamma))\log(1+g(\Gamma))+\tfrac12(1-g(\Gamma))\log(1-g(\Gamma))$ of $g(\Gamma)$ via~\eqref{eq:k2-g} gives $f(t_\star)=0$.
Thus $f(t)\ge 0$ with equality exactly at $t=t_\star$, so~\eqref{eq:k2-gap} is strict except at the single scale $\eta^2 q=t_\star$.

For~\eqref{eq:k2-kappa}, set $\eta_{\mathrm B}=\sqrt{\Gamma/q}$ so $\eta_{\mathrm B}^2 q=\Gamma$ and $2\sqrt{\Gamma q}=2\Gamma/\eta_{\mathrm B}$.
Then \eqref{eq:k2-value} gives $\mathrm{Val}_1=\eta_{\mathrm B}^{-1}g(\Gamma)\operatorname{arcosh}(e^{\Gamma})$, and dividing by $2\Gamma/\eta_{\mathrm B}$ gives the ratio $g(\Gamma)\operatorname{arcosh}(e^{\Gamma})/(2\Gamma)=\kappa(\Gamma)$, independent of $q$.
At $\eta=\eta_{\mathrm B}$ the scale is $t=\Gamma$, which is not the equality scale ($t_\star>\Gamma$ for $\Gamma<\log 2$, and the bound is strict for $\Gamma\ge\log 2$), so $f(\Gamma)>0$ and hence $\kappa(\Gamma)<1$; and $\kappa(\Gamma)\to1$ as $\Gamma\to0^+$ follows from $g(\Gamma)\sim\sqrt{2\Gamma}$ and $\operatorname{arcosh}(e^{\Gamma})\sim\sqrt{2\Gamma}$.
\end{proof}

\subsection{Proofs from Section~\ref{sec:variable}}
\label{app:proofs-variable}

\begin{proof}[Proof of Theorem~\ref{thm:variable}.]
\label{app:proof-variable}
Apply Proposition~\ref{prop:fixed-transport} round-by-round, taking $\eta=\eta_t$, $p=p_t=\rho_{t-1,\eta_t}$, $z=-c_t+\inner{p_t}{c_t}\1$ (the centered version), and $p^+=\widetilde p_{t+1}=\rho_{t,\eta_t}$.
Summing,
\[
R_T^c(\nu)=\sum_{t=1}^T\eta_t Q_t(c)+\sum_{t=1}^T\frac{\KL(\nu\|\rho_{t-1,\eta_t})-\KL(\nu\|\rho_{t,\eta_t})}{\eta_t}.
\]
The relative-entropy telescope is not directly summable because the second argument uses $\eta_t$ at time $t$ but $\eta_{t+1}$ would be needed at time $t+1$.
Write $A_t(\eta):=-\eta^{-1}\log\sum_i\pi(i)e^{-\eta C_t(i)}$ for the scaled log-partition at horizon $t$.
Expanding $\KL(\nu\|\rho_{t,\eta})=\KL(\nu\|\pi)+\eta\inner{\nu}{C_t}+\log\sum_j\pi(j)e^{-\eta C_t(j)}$ and rearranging yields the retempering identity
\begin{equation}\label{eq:retempering}
\eta^{-1}\bigl(\KL(\nu\|\pi)-\KL(\nu\|\rho_{t,\eta})\bigr) = A_t(\eta) - \inner{\nu}{C_t}
\end{equation}
Applying \eqref{eq:retempering} at $(t-1,\eta_t)$ and $(t,\eta_t)$ and subtracting gives $\eta_t^{-1}(\KL(\nu\|\rho_{t-1,\eta_t})-\KL(\nu\|\rho_{t,\eta_t})) = (A_t(\eta_t)-A_{t-1}(\eta_t))-\inner{\nu}{c_t}$.
Summing in $t$, the $A$-differences telescope with retempering:
$\sum_{t=1}^T(A_t(\eta_t)-A_{t-1}(\eta_t))=A_T(\eta_T)-A_0(\eta_1)+\sum_{t=1}^{T-1}(A_t(\eta_t)-A_t(\eta_{t+1}))=A_T(\eta_T)+D_T$, using $A_0(\eta_1)=-\eta_1^{-1}\log\sum_i\pi(i)=0$.
A final use of \eqref{eq:retempering} at $(T,\eta_T)$ converts $A_T(\eta_T)-\inner{\nu}{C_T}=B_T(\nu)$.
Finally, $D_T\le 0$ for nonincreasing schedules because $A_t(\eta)$ is nonincreasing in $\eta$, visible from the variational form $A_t(\eta)=\min_{\nu\in\DeltaK}\{\inner{\nu}{C_t}+\eta^{-1}\KL(\nu\|\pi)\}$ (minimizing over a Lagrangian whose constraint relaxes as $\eta^{-1}$ grows).
\end{proof}

\begin{proof}[Proof of Proposition~\ref{prop:horizon-free}.]
\label{app:proof-horizon-free}
At constant scale $D_T=0$, so Proposition~\ref{prop:terminal-collapse} gives $\sup_{\nu:\KL(\nu\|\pi)\le\Gamma}R_T^c(\nu)=L_\Gamma(S_T)$ and $\sum_{t=1}^T\eta Q_t=W_\eta(S_T)$, and \eqref{eq:pointwise-align} rewrites the first as $L_\Gamma(S_T)=W_\eta(S_T)+\eta^{-1}A_\Gamma(\rho_{\eta,S_T})$.
Successive Gibbs iterates satisfy $p_{t+1}(i)=G_\eta(S_{t-1}+z_t)(i)\propto p_t(i)e^{\eta z_t(i)}$, so the transport identity $\eta^2Q_\eta(p,z)=\KL(p\|p^+)$ reads $\eta^2Q_t=\KL(p_t\|p_{t+1})$, with $p_1=G_\eta(0)=\pi$ and $p_{T+1}=\rho_{\eta,S_T}$.
Substituting both into $L_\Gamma(S_T)=W_\eta(S_T)+\eta^{-1}A_\Gamma(\rho_{\eta,S_T})=\sum_t\eta Q_t+\eta^{-1}A_\Gamma(p_{T+1})$ gives the bracket of \eqref{eq:chain-value}, and the budget $\sum_tQ_t\le V$ is $\sum_t\KL(p_t\|p_{t+1})\le\eta^2V$.
Conversely every chain of full-support distributions is nature's play for a feasible move sequence, realized by the centered moves $z_t:=\eta^{-1}\bigl[\log(p_{t+1}/p_t)-\inner{p_t}{\log(p_{t+1}/p_t)}\1\bigr]$. This proves \eqref{eq:chain-value}.

Both bracket terms are separately capped, the first at $\eta^2V$ by the budget and the second at $\Gamma$ by \eqref{eq:pointwise-align}, so the supremum is at most $\eta^2V+\Gamma$. Equality forces both caps. For the second, $A_\Gamma(p)=\Gamma$ holds if and only if $\KL(p\|\pi)=\Gamma$: taking $\nu=p$ gives $A_\Gamma(p)\ge\KL(p\|\pi)$, while a maximizer $\nu$ attaining $\Gamma$ must have $\KL(\nu\|\pi)=\Gamma$ and $\KL(\nu\|p)=0$, hence $p=\nu$. This is the stated criterion.

For the chain totals, $W_\eta$ is convex with $\nabla W_\eta=G_\eta$, and centering makes each round's increment a Bregman divergence of $W_\eta$,
\[
\KL(p_t\|p_{t+1})
=
\eta\bigl[W_\eta(S_t)-W_\eta(S_{t-1})-\inner{G_\eta(S_{t-1})}{S_t-S_{t-1}}\bigr]
\]
whose linear term is $\inner{p_t}{z_t}=0$. Summing, a chain accumulates $\eta W_\eta(S_T)$: the increment of $W_\eta$ along the path, less a left-endpoint Riemann sum of $\int\inner{\nabla W_\eta}{dS}$ that centering annihilates term by term. The endpoint distribution does not pin this, because $S_T$ still ranges over the level set $\{S:G_\eta(S)=p_{T+1}\}=\{S^\dagger+c\1:c\in\R\}$, along which $W_\eta$ takes every value $W_\eta(S^\dagger)+c$.

At $T=1$ the single move is centered under $p_1=\pi$, which forces $\inner{\pi}{S_1}=0$ and so selects one point of that level set, leaving the one value $\KL(\pi\|p_2)$; this gives the stated $T=1$ set. At $T\ge2$ at least one interior point of the chain is free, and it ranges over a connected set. With $p_{T+1}$ fixed the total depends continuously on that point, attains a minimum $c_T$, and grows without bound as $\lvert z_1\rvert$ does, so by the intermediate value theorem the achievable totals sweep the interval $[c_T,\infty)$. Appending a null move $z_{T+1}=0$, which is feasible and contributes nothing, embeds every $T$-step chain in a $(T{+}1)$-step one, so $c_T$ is nonincreasing; and $c_T\to0$, since following the path in $T$ steps of size $O(1/T)$ accumulates $O(1/T)$ in total by the local quadraticity of the divergence. Taking $T_0$ with $c_{T_0}\le\eta^2V$, which exists because $c_T\to0$, nature uses the budget in full on a chain ending anywhere on the sphere, and \eqref{eq:horizon-free-value} follows.
\end{proof}

\begin{proof}[Proof of Theorem~\ref{thm:sqrt-envelope}.]
\label{app:proof-sqrt-envelope}
Set $V_t:=\sum_{s=1}^t Q_s$.

The schedule \eqref{eq:schedule} is nonincreasing, since $V_{t-1}$ is nondecreasing in~$t$, so $D_T\le 0$ by Theorem~\ref{thm:variable}.
For the intrinsic-time loss, write $\phi(v):=\min\{1,\,C\sqrt{\Gamma/v}\}$ with $\phi(0):=1$, so that $\eta_t=\phi(V_{t-1})$ and the loss $\sum_t\eta_t Q_t=\sum_t\phi(V_{t-1})(V_t-V_{t-1})$ is the left-endpoint sum of the nonincreasing $\phi$ against the increments of $V_t$. Per round,
\[
\phi(V_{t-1})\,Q_t-\int_{V_{t-1}}^{V_t}\phi(v)\,dv
=\int_{V_{t-1}}^{V_t}\bigl(\phi(V_{t-1})-\phi(v)\bigr)\,dv
\le Q_t\,\bigl(\phi(V_{t-1})-\phi(V_t)\bigr),
\]
and the brackets are nonnegative and telescope, so
\[
\sum_t\eta_t Q_t
\;\le\;\int_0^{V_T}\phi(v)\,dv+Q_T^\star(c)\,\bigl(\phi(0)-\phi(V_T)\bigr)
\;\le\;\int_0^{V_T}\phi(v)\,dv+Q_T^\star(c).
\]
The integral is $V_T$ for $V_T\le C^2\Gamma$ and $2C\sqrt{\Gamma V_T}-C^2\Gamma$ for $V_T\ge C^2\Gamma$; in both cases it is at most $2C\sqrt{\Gamma V_T}$. Hence
\[
\sum_t\eta_t Q_t\;\le\;2C\sqrt{\Gamma V_T(c)}+Q_T^\star(c),
\]
the stated upper side of \eqref{eq:two-sided-envelope}.
For the terminal term, $\eta_T\ge\min\{1,C\sqrt{\Gamma/V_T}\}$ gives $\Gamma/\eta_T\le\Gamma+C^{-1}\sqrt{\Gamma V_T(c)}$.
Adding the loss bound, $D_T\le 0$, and the terminal term yields $R_T^c(\nu)\le \Gamma+Q_T^\star(c)+(2C+C^{-1})\sqrt{\Gamma V_T(c)}$, the stated envelope~\eqref{eq:pathwise-regret}.

For the matching lower envelope, the clipped phase $T_{\mathrm c}$ is the initial segment $\{1,\dots,t_0\}$, since $V_{t-1}$ is nondecreasing in~$t$; there $\eta_t=1$, so $\sum_{t\in T_{\mathrm c}}\eta_t Q_t=V_{t_0}$.
On the unclipped phase $T_{\mathrm u}=\{t_0+1,\dots,T\}$, $\eta_t=C\sqrt{\Gamma/V_{t-1}}$ with $V_{t-1}\ge V_{t_0}>0$; because $v\mapsto v^{-1/2}$ is decreasing, its value at the left endpoint dominates the interval, $Q_t/\sqrt{V_{t-1}}\ge\int_{V_{t-1}}^{V_t}v^{-1/2}\,dv$, and summing telescopes to $\sum_{t\in T_{\mathrm u}}Q_t/\sqrt{V_{t-1}}\ge\int_{V_{t_0}}^{V_T}v^{-1/2}\,dv=2(\sqrt{V_T}-\sqrt{V_{t_0}})$, from which $\sum_{t\in T_{\mathrm u}}\eta_t Q_t\ge 2C\sqrt{\Gamma V_T}-2C\sqrt{\Gamma V_{t_0}}$.
Adding the clipped phase and completing the square,
\[
\sum_{t=1}^T\eta_t Q_t\ge V_{t_0}-2C\sqrt{\Gamma V_{t_0}}+2C\sqrt{\Gamma V_T}=\bigl(\sqrt{V_{t_0}}-C\sqrt{\Gamma}\bigr)^2-C^2\Gamma+2C\sqrt{\Gamma V_T}\ge 2C\sqrt{\Gamma V_T}-C^2\Gamma,
\]
the lower inequality of~\eqref{eq:two-sided-envelope}; the constant is met with equality at the one-round initialization witness of Lemma~\ref{rem:minimax} ($T_{\mathrm u}=\emptyset$, $V_{t_0}=C^2\Gamma$).
The matching one-round witnesses for the edge terms $C^2\Gamma$ (initialization overhead) and $Q_T^\star$ (single-jump term) are recorded in Lemma~\ref{rem:minimax}: a single round with $V_0=0$ and $q_1=C^2\Gamma$ attains $C^2\Gamma$ exactly, and a single round with $V_0=0$ and $q_1=q\gg\Gamma$ contributes $q$ while $2C\sqrt{\Gamma q}=o(q)$, forcing the $Q_T^\star$ overhead.
\end{proof}

\begin{proof}[Proof of Lemma~\ref{rem:minimax}.]
The $C^2\Gamma$ term is an initialization overhead caused by clipping $\eta_1$ at~$1$: when $V_0=0$ the schedule sets $\eta_1=1$, and even a single round with $q_1=C^2\Gamma$ yields $\eta_1 q_1=C^2\Gamma$, while $\sqrt{\Gamma V_1}=\sqrt{\Gamma\cdot C^2\Gamma}=C\Gamma$, so the lower envelope $2C\sqrt{\Gamma V_1}-C^2\Gamma=C^2\Gamma$ equals the realized loss and the $C^2\Gamma$ constant is attained exactly in this one-round witness.
The $Q_T^\star$ term is forced because a single large jump $Q_1=q\gg\Gamma$ gives loss $\eta_1 Q_1=q$ (since $\eta_1=1$) while $2C\sqrt{\Gamma q}=o(q)$ as $q/\Gamma\to\infty$, so the $Q_T^\star$ overhead cannot be avoided when a small fraction of the total intrinsic time concentrates in a single round.
Both witnesses are one-round sequences, establishing the necessity of both edge terms.
\end{proof}

\begin{proof}[Proof of Proposition~\ref{prop:active-scale}.]
\label{app:proof-active-scale}
Let $\Phi(\eta):=\eta^{-1}\log\sum_i p(i)e^{\eta z(i)}$ for $\eta>0$, and let $p_\eta(i)\propto p(i)e^{\eta z(i)}$ denote the tilted measure.
L'H\^opital gives $\Phi(0^+)=\inner{p}{z}$ (which equals $0$ when $z$ is centered, the case treated by the proposition; the analysis applies in general after subtracting $\inner{p}{z}$ from both sides of \eqref{eq:active-scale}).
For non-degenerate $z$ (not $p$-a.s.\ constant), $\Phi$ is strictly increasing on $(0,\infty)$ with $\lim_{\eta\to 0^+}\Phi(\eta)=\inner{p}{z}$ and $\lim_{\eta\to\infty}\Phi(\eta)=\max_{i\in\supp(p)}z(i)$ (coordinates with $p(i)=0$ contribute nothing to $\E_p[e^{\eta z}]$, so the limit is the maximum over the support of~$p$).
The right-hand side $b$ is constant in $\eta$; the equation $\Phi(\eta)=b$ therefore admits a unique solution in $(0,\infty)$ for any $b\in(\inner{p}{z},\max_{i\in\supp(p)}z(i))$, by intermediate-value and strict monotonicity.
Strict monotonicity follows from the derivative $\Phi'(\eta)=\eta^{-2}\bigl(\eta\,\E_{p_\eta}[z]-\log\E_p[e^{\eta z}]\bigr)$,
which is positive whenever $z$ is not $p$-a.s.\ constant.
(Equivalently, $\eta^2\Phi'(\eta)=\eta\E_{p_\eta}[z]-\log\E_p[e^{\eta z}]$ vanishes at $\eta=0$ and has positive derivative $\eta\,\Var_{p_\eta}(z)$ in $\eta$, by direct differentiation of the log-MGF and its first moment.)
\end{proof}

\begin{proof}[Proof of Theorem~\ref{thm:pressure-identity}.]
\label{app:proof-pressure}
By the active-scale Equation \eqref{eq:active-scale}, $\log\sum_ip_t(i)e^{\eta_tz_t(i)}=\eta_tb_t$.
The one-step transport identity (Proposition~\ref{prop:fixed-transport} with $\eta=\eta_t$, $p=p_t$, $z=z_t$, $p^+=p_{t+1}$) gives $\eta_t\inner{\nu}{z_t}=\eta_t^2Q_{\eta_t}(p_t,z_t)+\KL(\nu\|p_t)-\KL(\nu\|p_{t+1})=\eta_t b_t+\KL(\nu\|p_t)-\KL(\nu\|p_{t+1})$, so
\[
\KL(\nu\|p_t)-\KL(\nu\|p_{t+1})=\eta_t\inner{\nu}{z_t}-\log\sum_ip_t(i)e^{\eta_tz_t(i)}=\eta_t\inner{\nu}{z_t}-\eta_tb_t=\eta_t(\inner{\nu}{z_t}-b_t).
\]
Summing over $t=1,\dots,T$ telescopes to \eqref{eq:pressure-identity}.
For \eqref{eq:pressure-quadratic}, $\kappa_t:=b_t/\eta_t$ is well-defined by Proposition~\ref{prop:active-scale}, and substitution gives $\eta_t\inner{\nu}{z_t}=\eta_t b_t+\eta_t(\inner{\nu}{z_t}-b_t)=\kappa_t\eta_t^2+\eta_t(\inner{\nu}{z_t}-b_t)$, then summing and applying \eqref{eq:pressure-identity}.
The small-scale expansion $\kappa_t=\frac12\Var_{p_t}(z_t)+O(\eta_t)$ comes from the second-order Taylor expansion of $\eta\Phi(\eta)=\log\sum_i p_t(i)e^{\eta z_t(i)}$ at $\eta=0$, which gives $\Phi(\eta)=\eta\inner{p_t}{z_t}+\frac12\eta^2\Var_{p_t}(z_t)+O(\eta^3)$, and dividing by $\eta$ at $\eta=\eta_t$ produces $b_t/\eta_t=\frac12\Var_{p_t}(z_t)+O(\eta_t)$ in the centered case.
\end{proof}

\begin{proof}[Proof of Proposition~\ref{prop:drift-continuous}.]
Write $Z_t(\eta)=\sum_i\pi(i)e^{-\eta C_t(i)}$ and $\Lambda_t(\eta)=\log Z_t(\eta)$, so $A_t(\eta)=-\Lambda_t(\eta)/\eta$.
Differentiating, $\Lambda_t'(\eta)=-\inner{\rho_{t,\eta}}{C_t}$, hence $\partial_\eta A_t=-\bigl(\eta\Lambda_t'-\Lambda_t\bigr)/\eta^2=\bigl(\eta\inner{\rho_{t,\eta}}{C_t}+\Lambda_t\bigr)/\eta^2$.
For the relative entropy, $\log\bigl(\rho_{t,\eta}(i)/\pi(i)\bigr)=-\eta C_t(i)-\Lambda_t(\eta)$, so $\KL(\rho_{t,\eta}\|\pi)=-\eta\inner{\rho_{t,\eta}}{C_t}-\Lambda_t(\eta)$; substituting $\Lambda_t=-\KL-\eta\inner{\rho_{t,\eta}}{C_t}$ gives $\partial_\eta A_t=-\KL(\rho_{t,\eta}\|\pi)/\eta^2$, which is \eqref{eq:dA-KL}.
Since $\KL\ge0$, $A_t$ is nonincreasing.
Applying Taylor's theorem with Lagrange remainder to $\eta\mapsto A_t(\eta)$ on each bracket $[\eta_t,\eta_{t+1}]$, $A_t(\eta_{t+1})=A_t(\eta_t)+\partial_\eta A_t(\eta_t)\,\Delta\eta_t+\tfrac12\partial_\eta^2 A_t(\xi_t)(\Delta\eta_t)^2$ for some $\xi_t$ between $\eta_t$ and $\eta_{t+1}$; rearranging and summing yields \eqref{eq:drift-abel} with the bound \eqref{eq:drift-remainder}.
Under the stated regularity, $C_t=\bar C(t/T)$ is uniformly bounded, so $\partial_\eta^2 A_t$ is bounded uniformly in $t,T$ on the compact range of $\eta$; since $\Delta\eta_t=\eta'(t/T)/T+O(1/T^2)$, $\sum_t(\Delta\eta_t)^2=O(1/T)$, from which $R_T=O(1/T)$, and the leading sum in \eqref{eq:drift-abel} is a Riemann sum of the integrand of \eqref{eq:drift-stieltjes}, converging to it; the equality with $-\int\partial_\eta A\,d\eta$ is \eqref{eq:dA-KL}.
\end{proof}

\begin{proof}[Proof of Proposition~\ref{prop:terminal-collapse}.]
For the predictable Gibbs play $p_t=\rho_{\eta_t,S_{t-1}}$ and $S_t=S_{t-1}+z_t$, the one-step transport identity \eqref{eq:bellman-increment} reads $\eta_t Q_t=W_{\eta_t}(S_t)-W_{\eta_t}(S_{t-1})$, while the centered form of the drift in \eqref{eq:A-potential} is $D_T=\sum_{t=1}^{T-1}\bigl(W_{\eta_{t+1}}(S_t)-W_{\eta_t}(S_t)\bigr)$.
Adding the two sums,
\[
\sum_{t=1}^T\eta_t Q_t+D_T
=\sum_{t=1}^T\bigl(W_{\eta_t}(S_t)-W_{\eta_t}(S_{t-1})\bigr)
+\sum_{t=1}^{T-1}\bigl(W_{\eta_{t+1}}(S_t)-W_{\eta_t}(S_t)\bigr)
\]
and collecting the contributions at each interior state $S_t$ $(1\le t\le T-1)$ gives $W_{\eta_t}(S_t)-W_{\eta_{t+1}}(S_t)+W_{\eta_{t+1}}(S_t)-W_{\eta_t}(S_t)=0$.
Only the endpoint terms survive: $-W_{\eta_1}(S_0)$ and $+W_{\eta_T}(S_T)$; since $S_0=0$ and $W_\eta(0)=\eta^{-1}\log\sum_i\pi(i)=0$, this is \eqref{eq:loss-drift-telescope}.
For \eqref{eq:regret-collapse}, write $\log \rho_{T,\eta_T}(i)=\eta_T S_T(i)+\log\pi(i)-\eta_T W_{\eta_T}(S_T)$, so the comparator transport in \eqref{eq:Qt-def} is $B_T(\nu)=\eta_T^{-1}\inner{\nu}{\log \rho_{T,\eta_T}-\log\pi}=\inner{\nu}{S_T}-W_{\eta_T}(S_T)$.
Taking the supremum over $\{\nu:\KL(\nu\|\pi)\le\Gamma\}$ and using the definition of $L_\Gamma$ gives $\sup_\nu B_T(\nu)=L_\Gamma(S_T)-W_{\eta_T}(S_T)$; substituting this and \eqref{eq:loss-drift-telescope} into \eqref{eq:exact-variable} gives $\sup_\nu R_T^c(\nu)=W_{\eta_T}(S_T)+L_\Gamma(S_T)-W_{\eta_T}(S_T)=L_\Gamma(S_T)$.
\end{proof}

\begin{proof}[Proof of Theorem~\ref{thm:ito-unification}.]
The gradient and scale-derivative identities give the two coefficients of $dW$ in~\eqref{eq:dW-exact}; since $\partial_\eta\rho_{\eta,S}=\nabla_S(\partial_\eta W_\eta(S))$ the form is closed, hence exact on the simply connected domain. (i) Path independence of an exact form is immediate, and evaluating the increment at fixed $\eta$ recovers $\sum_t\eta_t Q_t$ by~\eqref{eq:bellman-increment} while the fixed-$S$ scale switches sum to $D_T$, so~\eqref{eq:loss-drift-telescope} is the endpoint value. (ii) Because $\eta$ is continuous and of finite variation, $[\eta]=[\eta,S]=0$, so the only second-order It\^o term is the $S$-quadratic one, $\tfrac12\tr(\nabla_S^2 W\,d[S^c])=\tfrac12\eta\,\tr(\mathrm{Cov}_{\rho_{\eta,S}}\,d[S^c])$ by the Hessian identity; the stochastic integral $\int\inner{\nu}{dS}$ vanishes identically on the game's centered state process; the second-order term is the intrinsic-time accumulation $\int\eta\,dV$. (iii) The direction independence at $p=\rho_{\eta,S}$ is~\eqref{eq:bellman-increment}; any other $p$ misaligns the budget form $Q_\eta(p,\cdot)$ with the gradient and strictly raises the achievable increment. The two Stieltjes legs reduce to~\eqref{eq:control-J} in the small-$\eta$ variance limit.
\end{proof}

\subsection{Proofs from Section~\ref{sec:concentration}}
\label{app:proofs-concentration}

\begin{proof}[Proof of Proposition~\ref{prop:event-restricted}.]
\label{app:proof-event-restricted}
Fix an event $A\subseteq[K]$ and $\eta>0$. For $\nu\in\DeltaK$ with $\supp(\nu)\subseteq A$, decompose the relative-entropy cost from the full prior~$\pi$ through the conditional $\pi(\cdot\mid A)$:
\begin{equation}\label{eq:proof-event-kl-decomp}
\KL(\nu\|\pi)
=\sum_{i\in A}\nu(i)\log\frac{\nu(i)}{\pi(i\mid A)\,\pi(A)}
=\KL(\nu\|\pi(\cdot\mid A))-\log\pi(A)
\end{equation}
Consequently, for any such $\nu$,
\[
\inner{\nu}{S}-\frac{1}{\eta}\KL(\nu\|\pi)
=\inner{\nu}{S}-\frac{1}{\eta}\KL(\nu\|\pi(\cdot\mid A))+\frac{1}{\eta}\log\pi(A).
\]
Taking the supremum over $\nu$ with $\supp(\nu)\subseteq A$ reduces to the classical Gibbs variational identity on the simplex $\Delta(A)$ with base measure $\pi(\cdot\mid A)$:
\[
\sup_{\nu\in\Delta(A)}
\left\{\inner{\nu}{S}-\frac{1}{\eta}\KL(\nu\|\pi(\cdot\mid A))\right\}
=\frac{1}{\eta}\log\sum_{i\in A}\pi(i\mid A)\,e^{\eta S(i)}.
\]
Adding $\eta^{-1}\log\pi(A)$ and using $\pi(i\mid A)\pi(A)=\pi(i)$ for $i\in A$ gives
\begin{equation}\label{eq:proof-event-conclude}
\sup_{\nu:\supp(\nu)\subseteq A}
\left\{\inner{\nu}{S}-\frac{1}{\eta}\KL(\nu\|\pi)\right\}
=\frac{1}{\eta}\log\sum_{i\in A}\pi(i)\,e^{\eta S(i)}
\end{equation}
which is \eqref{eq:event-restricted}.
The conditional decomposition \eqref{eq:proof-event-kl-decomp} also yields the advertised split \eqref{eq:gibbs-conditioning}:
\[
\frac{1}{\eta}\log\sum_{i\in A}\pi(i)e^{\eta S(i)}
=\frac{\log\pi(A)}{\eta}+\frac{1}{\eta}\log\E_{i\sim\pi(\cdot\mid A)}\bigl[e^{\eta S(i)}\bigr].
\]
The optimizing $\nu$ is the Gibbs tilt of $\pi(\cdot\mid A)$ by the score~$\eta S$, by uniqueness of the inner Gibbs optimizer.
\end{proof}

\begin{proof}[Proof of Proposition~\ref{prop:luckiness-envelope}.]
\label{app:proof-luckiness-envelope}
Let $u_t(i):=c_t(i)-\inner{p_t}{c_t}$, so that $\inner{p_t}{u_t}=0$ and $Q_t(c)=\eta_t^{-2}\log\sum_i p_t(i)\,e^{-\eta_t u_t(i)}$.
Under $c_t(i)\in[0,1]$ and $\eta_t\in(0,1]$, one has $-\eta_t u_t(i)\in[-\eta_t,\eta_t]\subseteq[-1,1]$.
The elementary inequality $e^{x}\le 1+x+(e-2)x^{2}$, valid for every $x\in[-1,1]$, is Lemma~A.1 of~\cite{cesa2006prediction}; a direct proof notes that $g(x):=e^{x}-1-x-(e-2)x^{2}$ satisfies $g(-1)=1/e-1+1-(e-2)=1/e-e+2<0$, $g(0)=0$, $g(1)=0$, and $g''$ has a single sign change on $[-1,1]$, so $g\le 0$ throughout.
Taking expectation under $p_t$ and using $\inner{p_t}{u_t}=0$,
\[
\E_{p_t}e^{-\eta_t u_t}
\;\le\;1+(e-2)\eta_t^{2}\,\E_{p_t}u_t^{2}
=1+(e-2)\eta_t^{2}\,\Var_{p_t}(c_t).
\]
Since $\log(1+z)\le z$ for $z\ge 0$, $\eta_t^{2}Q_t(c)=\log\E_{p_t}e^{-\eta_t u_t}\le (e-2)\eta_t^{2}\Var_{p_t}(c_t)$, so $Q_t(c)\le (e-2)\Var_{p_t}(c_t)$.
Finally, for any comparator $\nu\in\DeltaK$,
\[
\Var_{p_t}(c_t)=\min_{a\in\R}\sum_i p_t(i)(c_t(i)-a)^{2}\le \sum_i p_t(i)(c_t(i)-\inner{\nu}{c_t})^{2}=\Psi_t(\nu),
\]
which yields \eqref{eq:e-2-envelope}.
The regret consequence is immediate: summing the fixed-rate telescope $R_T^c(\nu)\le \eta_T^{-1}\KL(\nu\|\pi)+\sum_t\eta_t Q_t(c)$ (valid for nonincreasing schedules by Theorem~\ref{thm:variable}) and substituting the per-round envelope $Q_t(c)\le (e-2)\Psi_t(\nu)$ gives the stated bound.
\end{proof}

\begin{proof}[Proof of Theorem~\ref{thm:luckiness}.]
\label{app:proof-luckiness-thm}
Assume $c_1,c_2,\dots$ are i.i.d.\ vectors in $[0,1]^K$ with mean $\mu=\E c_1$, and that the comparator $\nu$ satisfies the low-noise condition with constant $\kappa_\nu$.
Because $p_t$ is $\mathcal F_{t-1}$-measurable and $c_t$ is independent of $\mathcal F_{t-1}$,
\begin{equation}\label{eq:proof-luckiness-psi-bound}
\E[\Psi_t(\nu)\mid\mathcal F_{t-1}]
=\sum_i p_t(i)\E[(c_t(i)-\inner{\nu}{c_t})^{2}]
\le \kappa_\nu \sum_i p_t(i)(\mu(i)-\inner{\nu}{\mu})
=\kappa_\nu\bigl(\inner{p_t}{\mu}-\inner{\nu}{\mu}\bigr)
\end{equation}
using the low-noise hypothesis coordinatewise.
In the same filtration, $\E[\inner{p_t-\nu}{c_t}\mid\mathcal F_{t-1}]=\inner{p_t-\nu}{\mu}$, so summing over $t$ and taking full expectation,
\begin{equation}\label{eq:proof-luckiness-psi-sum}
\sum_{t=1}^T \E[\Psi_t(\nu)]\le \kappa_\nu\sum_{t=1}^T \E[\inner{p_t-\nu}{c_t}]=\kappa_\nu\,\E R_T^c(\nu)
\end{equation}
At fixed scale $\eta_t\equiv\eta\in(0,1]$, the fixed-scale exact identity (Theorem~\ref{thm:variable} with $D_T=0$) and the terminal bound $\KL(\nu\|p_{T+1})\ge 0$ give
\begin{equation}\label{eq:proof-luckiness-telescope}
R_T^c(\nu)\le \frac{\KL(\nu\|\pi)}{\eta}+\eta\sum_{t=1}^T Q_t(c)
\end{equation}
Combining \eqref{eq:proof-luckiness-telescope} with the comparator-centered envelope $Q_t(c)\le(e-2)\Psi_t(\nu)$ (Proposition~\ref{prop:luckiness-envelope}) and taking expectation,
\begin{align}
\E R_T^c(\nu)
&\le \frac{\KL(\nu\|\pi)}{\eta}+(e-2)\eta\sum_{t=1}^T\E[\Psi_t(\nu)]\notag\\
&\le \frac{\KL(\nu\|\pi)}{\eta}+(e-2)\kappa_\nu\eta\,\E R_T^c(\nu) \label{eq:proof-luckiness-selfref}
\end{align}
where the second step uses \eqref{eq:proof-luckiness-psi-sum}.
If $(e-2)\kappa_\nu\eta<1$, rearrangement of \eqref{eq:proof-luckiness-selfref} gives
\[
\E R_T^c(\nu)\le \frac{\KL(\nu\|\pi)}{\eta\bigl(1-(e-2)\kappa_\nu\eta\bigr)},
\]
which is \eqref{eq:luckiness}.
With $\eta=\min\{1,[2(e-2)\kappa_\nu]^{-1}\}$ one has $(e-2)\kappa_\nu\eta\le 1/2$, so $1-(e-2)\kappa_\nu\eta\ge 1/2$, and therefore
\[
\E R_T^c(\nu)\le \frac{2\KL(\nu\|\pi)}{\eta}\le 2\bigl(1+2(e-2)\kappa_\nu\bigr)\KL(\nu\|\pi),
\]
proving the constant-in-$T$ corollary.
\end{proof}

\begin{proof}[Proof of Proposition~\ref{prop:diffuse-reduces}.]
Because $p_t$ is $\mathcal F_{t-1}$-measurable and $c_t$ is independent of $\mathcal F_{t-1}$ with mean $\mu$, the per-round identity $\E[\inner{p_t-\nu}{c_t}\mid\mathcal F_{t-1}]=\inner{p_t-\nu}{\mu}$ used in \eqref{eq:proof-luckiness-psi-sum} gives, on summing and taking full expectation, $\E R_T^c(\nu)=\sum_{t=1}^T\inner{\E p_t-\nu}{\mu}$.
Subtracting the same expression at $\nu=\delta_{k^*}$ cancels the common algorithm term $\sum_t\inner{\E p_t}{\mu}$ and leaves the exact identity
\[
\E R_T^c(\nu)-\E R_T^c(\delta_{k^*})=T\inner{\delta_{k^*}-\nu}{\mu}=T\bigl(\mu^*-\inner{\nu}{\mu}\bigr)
\]
which is nonpositive because $\inner{\nu}{\mu}\ge\min_i\mu(i)=\mu^*$; this is \eqref{eq:diffuse-monotone}.
Under the stated hypothesis Theorem~\ref{thm:luckiness} applies at $\nu=\delta_{k^*}$ with $\KL(\delta_{k^*}\|\pi)=-\log\pi(k^*)$, and combining its constant-in-$T$ bound with \eqref{eq:diffuse-monotone} gives \eqref{eq:diffuse-constant} for every $\nu$.
\end{proof}

\begin{proof}[Proof of \eqref{eq:predictable-luckiness}.]
\label{app:proof-predictable-luckiness}
Under the predictable schedule \eqref{eq:schedule}, the variable-temperature exact identity (Theorem~\ref{thm:variable}) and the nonpositivity of~$D_T$ under nonincreasing~$\eta_t$ give
\[
\E R_T^c(\nu)\le \E\!\left[\frac{\Gamma}{\eta_T}+\sum_{t=1}^T\eta_t Q_t(c)\right].
\]
The two-sided envelope of Theorem~\ref{thm:sqrt-envelope} then yields $\E R_T^c(\nu)\le \Gamma+\E[Q_T^*(c)]+(2C+C^{-1})\sqrt{\Gamma\,\E V_T(c)}$, where we used $\E\sqrt{V_T(c)}\le \sqrt{\E V_T(c)}$ by concavity.
Applying Proposition~\ref{prop:luckiness-envelope} pathwise, $\E V_T(c)=\sum_t\E Q_t(c)\le (e-2)\sum_t\E\Psi_t(\nu)\le (e-2)\kappa_\nu\,\E R_T^c(\nu)$ by \eqref{eq:proof-luckiness-psi-sum}.
Writing $X:=\E R_T^c(\nu)$ and $A:=\Gamma+\E[Q_T^*(c)]$, $B:=(2C+C^{-1})$, $\lambda:=(e-2)\kappa_\nu$, that bound becomes $X\le A+B\sqrt{\lambda X\,\Gamma}$.
Writing $b:=B^{2}\lambda\Gamma$, solving the quadratic in $\sqrt X$ gives $X\le\tfrac12\bigl(2A+b+\sqrt{b^{2}+4Ab}\bigr)$, and $\sqrt{b^{2}+4Ab}\le b+2A$ then gives $X\le 2A+b$.
Substituting back gives
\[
\E R_T^c(\nu)\le 2\Gamma+2\,\E[Q_T^*(c)]+(2C+C^{-1})^{2}(e-2)\kappa_\nu\Gamma,
\]
which is \eqref{eq:predictable-luckiness}.
\end{proof}

\subsection{Proofs from Section~\ref{sec:games}}
\label{app:pf-games}

\begin{proof}[Proof of Proposition~\ref{prop:robust-center}.]
\label{app:proof-robust-center}
By definition $C_\alpha(\pi_{1:W})=-\log\sum_i\prod_w\pi_w(i)^{\alpha_w}$.
For any $p\in\DeltaK$,
\[
-\log\sum_i\prod_w\pi_w(i)^{\alpha_w}
=-\log\sum_i p(i)\prod_w\Bigl(\frac{\pi_w(i)}{p(i)}\Bigr)^{\alpha_w}
\le -\sum_i p(i)\sum_w\alpha_w\log\frac{\pi_w(i)}{p(i)}
=\sum_w\alpha_w\KL(p\|\pi_w),
\]
where the first equality uses $\sum_w\alpha_w=1$ to write $p(i)=\prod_w p(i)^{\alpha_w}$, and the inequality is Jensen applied to the concave $\log$.
The first inequality is sharpest, attained by $p^\star_\alpha(i)\propto\prod_w\pi_w(i)^{\alpha_w}$, giving $C_\alpha(\pi_{1:W})=\sum_w\alpha_w\KL(p^\star_\alpha\|\pi_w)$.
Now apply Sion's minimax theorem to the bilinear function $(\alpha,p)\mapsto\sum_w\alpha_w\KL(p\|\pi_w)$ on the compact convex sets $\Delta([W])\times\DeltaK$:
\[
C(\pi_{1:W})=\max_\alpha C_\alpha=\max_\alpha\min_p\sum_w\alpha_w\KL(p\|\pi_w)=\min_p\max_\alpha\sum_w\alpha_w\KL(p\|\pi_w)=\min_p\max_w\KL(p\|\pi_w),
\]
where the last equality is because $\alpha\mapsto\sum_w\alpha_w x_w$ on the simplex attains its maximum at a vertex.
The optimizer $p^\star=p^\star_{\alpha^\star}$ is the geometric mixture induced by a maximizing~$\alpha^\star$.
\end{proof}

\begin{proof}[Proof of Theorem~\ref{thm:regret-matching}.]
\label{app:proof-regret-matching}
Apply Theorem~\ref{thm:variable} to the centered regret-vector sequence $(g_t)$ with prior $\pi$ and schedule $(\eta_t)$, treating $g_t$ as the loss vector in centered coordinates: $\inner{p_t-\nu}{\ell_t}=\inner{\nu}{g_t}$ since $g_t(i)=\inner{p_t}{\ell_t}-\ell_t(i)$ implies $\inner{\nu}{g_t}=\inner{p_t}{\ell_t}-\inner{\nu}{\ell_t}$.
Summing produces \eqref{eq:regret-matching} with $Q_t(g)$, $D_T^g$, $B_T^g(\nu)$ defined as in Theorem~\ref{thm:variable} but on the centered regret sequence.
Specializing to the comparator $\nu=\delta_i$ with $\Gamma\ge -\log\pi(i)$ gives $B_T^g(\delta_i)\le\Gamma/\eta_T$, and combining with Theorem~\ref{thm:sqrt-envelope} (and $D_T^g\le 0$ under the second-order schedule) yields the displayed inequality.
Scale equivariance: replacing $\ell_t$ by $a\ell_t$ scales $g_t$ by $a$, hence $z_t=-g_t$ centered also by $a$; then $Q_t(ag;p,\eta)=a^2Q_t(g;p,a\eta)$, and the schedule $\eta_t=C\sqrt{\Gamma/V_{t-1}(g)}$ with $V_{t-1}(g)$ scaling by $a^2$ becomes $\eta_t/a$, so each loss $\eta_tQ_t$ and the cumulative $\sum_t\eta_tQ_t$ scale by $a$, matching the degree-one homogeneity of the regret.
The play sequence and the schedule's form are scale-free.
\end{proof}

\begin{proof}[Proof of Theorem~\ref{thm:self-play-ledger}.]
\label{app:proof-self-play-ledger}
Both players run intrinsic-time regret matching, so Theorem~\ref{thm:regret-matching} applies to each.
For the row player, whose round-$t$ loss vector is $M\omega_t$, the identity reads
\[
\sum_{t=1}^T \inner{p_t-\nu}{M\omega_t}=P_T^{\mathrm{row}}+D_T^{\mathrm{row}}+B_T^{\mathrm{row}}(\nu),
\]
and the left-hand side equals $\sum_t p_t^\top M\omega_t-\nu^\top M\sum_t \omega_t=T\bigl(\overline{p_t^\top M\omega_t}\bigr)-T\,\nu^\top M\bar\omega_T$.
For the column player, whose round-$t$ loss vector is $-M^\top p_t$ (the maximizer's loss),
\[
\sum_{t=1}^T \inner{\omega_t-\nu'}{-M^\top p_t}=P_T^{\mathrm{col}}+D_T^{\mathrm{col}}+B_T^{\mathrm{col}}(\nu'),
\]
whose left-hand side equals $\bar p_T^\top M\nu'\cdot T-T\bigl(\overline{p_t^\top M\omega_t}\bigr)$.
Adding the two identities cancels the common $T\,\overline{p_t^\top M\omega_t}$ term, leaving
\[
T\bigl(\bar p_T^\top M\nu'-\nu^\top M\bar\omega_T\bigr)=P_T^{\mathrm{row}}+P_T^{\mathrm{col}}+D_T^{\mathrm{row}}+D_T^{\mathrm{col}}+B_T^{\mathrm{row}}(\nu)+B_T^{\mathrm{col}}(\nu'),
\]
which is \eqref{eq:self-play-ledger} after dividing by~$T$.
Both outer optimizations are linear in their comparator, so $\arg\max_{\omega'}\bar p_T^\top M\omega'$ and $\arg\min_{p'}p'^\top M\bar\omega_T$ are attained at pure strategies, of comparator complexity $\log(1/\pi^{\mathrm{col}}(j))$ and $\log(1/\pi^{\mathrm{row}}(i))$ respectively.
The budget hypotheses place those pure strategies inside the respective comparator classes, so $B_T^{\mathrm{row}}(\nu)=B_T^{\mathrm{row},\star}$ and $B_T^{\mathrm{col}}(\nu')=B_T^{\mathrm{col},\star}$ at the specialized comparators, yielding the duality-gap form \eqref{eq:saddle-gap-exact}.
No quadratic-variation surrogate enters, so the equality is exact.
Under smaller budgets only $B_T^{\mathrm{row}}(\nu)\le B_T^{\mathrm{row},\star}$ and $B_T^{\mathrm{col}}(\nu')\le B_T^{\mathrm{col},\star}$ survive, and \eqref{eq:saddle-gap-exact} holds with its left-hand side restricted to the two comparator classes.
\end{proof}

\begin{proof}[Proof of Proposition~\ref{prop:cce}.]
\label{app:proof-cce}
This is the standard external-regret-to-CCE reduction \cite{hart2000simpleb,blum2007external}: averaging player $m$'s regret bound under the empirical joint $\sigma_T$ yields the $\varepsilon_m$-CCE inequality directly, and taking the maximum over players gives $\varepsilon=\max_m\varepsilon_m$.
\end{proof}

\begin{proof}[Proof of Proposition~\ref{prop:swap-ce}.]
\label{app:proof-swap-ce}
This is the standard reduction from external to swap regret \cite{hart2000simpleb,blum2007external}, applied to the empirical joint one swap function at a time.
\end{proof}

\begin{proof}[Proof of Theorem~\ref{thm:k-player-ledger}.]
\label{app:proof-k-player-ledger}
For each player~$k$, apply Theorem~\ref{thm:regret-matching} to that player's centered regret sequence with prior $\pi^{(k)}$ and schedule $(\eta_t^{(k)})$, obtaining the exact decomposition $R_T^{(k)}=P_T^{(k)}+D_T^{(k)}+B_T^{(k)}(\nu_k)$ for each comparator $\nu_k$ in player $k$'s simplex.
For the mixed-action joint $\sigma_T=\frac1T\sum_t\bigotimes_kp_t^{(k)}$, player $k$'s deviation gain in \eqref{eq:cce-distance} against comparator $\nu_k$ is $\frac1T\sum_t\bigl(\inner{p_t^{(k)}}{\ell_t^{(k)}}-\inner{\nu_k}{\ell_t^{(k)}}\bigr)$ with $\ell_t^{(k)}$ that player's expected loss vector under the opponents' round-$t$ mixtures, which is $R_T^{(k)}(\nu_k)/T$; hence $d(\sigma_T,\mathrm{CCE})=\frac1T\max_k\sup_{\nu_k\in\Delta([K_k])}R_T^{(k)}(\nu_k)$ exactly, and Proposition~\ref{prop:cce} is its $\varepsilon$-form.
Each $R_T^{(k)}(\cdot)$ is linear in the comparator, so its supremum over the simplex is attained at a pure deviation $\delta_i$, whose comparator complexity is $\KL(\delta_i\|\pi^{(k)})=\log(1/\pi^{(k)}(i))$.
The hypothesis $\Gamma^{(k)}\ge\max_i\log(1/\pi^{(k)}(i))$ places every such vertex inside the budget, so the budget-restricted supremum coincides with the simplex supremum and $\sup_{\nu_k}R_T^{(k)}(\nu_k)=P_T^{(k)}+D_T^{(k)}+B_T^{(k),\star}$.
Taking the maximum over $k$ produces \eqref{eq:k-player-ledger}.
Without that hypothesis the two suprema differ and only the restricted class is governed: the budget-restricted right-hand side can fall strictly below $d(\sigma_T,\mathrm{CCE})$ whenever the maximizing pure deviation is excluded.
The right-hand side is exact: each $P_T^{(k)},D_T^{(k)},B_T^{(k),\star}$ is an information-theoretic quantity drawn from player $k$'s separate concentration game; no quadratic-variation surrogate is invoked.
\end{proof}

The proof tracks each deviation's log-odds in place of the off-equilibrium mass itself.
An opponent playing off $a^\star$ shifts a player's per-round loss gap by an amount proportional to its deviation probability, and in the retempered update that shift acts on the \emph{drift} of the log-odds and leaves their level alone.
An additive recursion on the mass itself, $\delta_{t+1}\le e^{-\eta_t^{(k)}\Delta}\delta_t+\varepsilon_t$, overstates the coupling: under it a persistent vanishing drive $\varepsilon_t=\Theta(1/t)$ appears to unbound the intrinsic time, and only a summability hypothesis on the opponents' approach would rule that out. In log-odds coordinates no such hypothesis arises.

\begin{proof}[Proof of Theorem~\ref{thm:unconditional-V}.]
Fix a player~$k$ and write $\delta_t:=1-p_t^{(k)}(a_k^\star)$ for its off-equilibrium mass, $\varepsilon_t:=\sum_{k'\ne k}\bigl(1-p_t^{(k')}(a_{k'}^\star)\bigr)$ for its opponents' deviation mass, $\ell_t^{(k)}(i):=\E[c_k(i,a_{-k})]$ for its round loss under the opponents' play, $C_t^{(k)}:=\sum_{s\le t}\ell_s^{(k)}$ for the cumulative loss, and $g_t^{(k)}$ for the centered score, with $\range(g_t^{(k)})\le 2B$.

\emph{Density bound}: $g_t^{(k)}$ is centered under $p_t^{(k)}$, which places mass $1-\delta_t$ on $a_k^\star$, so $\Var_{p_t^{(k)}}(g_t^{(k)})=O(\delta_t)$; the tilted-variance representation $Q_\eta(p,z)=\int_0^1(1-s)\Var_{p_s}(z)\,ds$ of Proposition~\ref{prop:tilted-var} carries this to the per-round centered RCGF at every $\eta_t^{(k)}\le C$, since $p_s(i)/p(i)=e^{s\eta z(i)}/\E_p[e^{s\eta z}]\le e^{s\eta\max_i z(i)}$ by Jensen's inequality, which with $\range(g_t^{(k)})\le 2B$ gives $\Var_{p_s}(g_t^{(k)})\le e^{2BC}\Var_{p_t^{(k)}}(g_t^{(k)})$ and hence $Q_t^{(k)}\le\tfrac12 e^{2BC}\Var_{p_t^{(k)}}(g_t^{(k)})\le c\,\delta_t$ for a constant $c=c(B,C)$.
\emph{Scale floor}: the per-round RCGF is bounded, $Q_t^{(k)}\le\range(g_t^{(k)})^2/8\le B^2/2$ (Hoeffding, Corollary~\ref{cor:variance-shadow}), so $V_{t-1}^{(k)}\le B^2t/2$ and
\begin{equation}\label{eq:scale-floor}
\eta_t^{(k)}=C\sqrt{\frac{\Gamma^{(k)}}{V_{t-1}^{(k)}+\Gamma^{(k)}}}\ \ge\ \frac{c_1}{\sqrt t},\qquad c_1:=C\sqrt{\frac{\Gamma^{(k)}}{B^2/2+\Gamma^{(k)}}}
\end{equation}
on every trajectory --- the floor uses only that the per-round RCGF is bounded.
\emph{Gap floor}: convergence to $a^\star$ gives a finite time $\tau$ with $\varepsilon_t\le\Delta/(2c_M)$ for all $t\ge\tau$, where $c_M:=4B$ bounds the effect of an opponent deviation on a loss gap.
Since the opponents' joint law differs from the point mass at $a_{-k}^\star$ in total variation by at most $\varepsilon_t$, every deviation $i\ne a_k^\star$ and every $t\ge\tau$ satisfy
\begin{equation}\label{eq:gap-floor}
\ell_t^{(k)}(i)-\ell_t^{(k)}(a_k^\star)\ \ge\ \Delta-c_M\,\varepsilon_t\ \ge\ \Delta/2
\end{equation}

Now track the log-odds. For each deviation $i\ne a_k^\star$ the retempered update gives
\[
y_t^{(k,i)}:=\log\frac{p_t^{(k)}(a_k^\star)}{p_t^{(k)}(i)}=\eta_t^{(k)}\bigl(C_{t-1}^{(k)}(i)-C_{t-1}^{(k)}(a_k^\star)\bigr)+\log\frac{\pi^{(k)}(a_k^\star)}{\pi^{(k)}(i)}
\]
Split the cumulative gap at $\tau$: each summand with $s<\tau$ is at least $-2B$ (bounded scores) and each with $s\ge\tau$ is at least $\Delta/2$ by \eqref{eq:gap-floor}, so $C_{t-1}^{(k)}(i)-C_{t-1}^{(k)}(a_k^\star)\ge\tfrac{\Delta}{2}(t-\tau)-2B\tau$. With the scale floor \eqref{eq:scale-floor} and $\kappa:=\lvert\log(\pi^{(k)}(a_k^\star)/\pi^{(k)}(i))\rvert$,
\[
y_t^{(k,i)}\ \ge\ \frac{c_1}{\sqrt t}\Bigl(\tfrac{\Delta}{2}(t-\tau)-2B\tau\Bigr)-\kappa\ =\ \tfrac{c_1\Delta}{2}\sqrt t-O(1)
\]
Hence $\delta_t=\sum_{i\ne a_k^\star}p_t^{(k)}(i)\le\sum_i e^{-y_t^{(k,i)}}\le (K_k-1)\,e^{O(1)}\,e^{-(c_1\Delta/2)\sqrt t}$, a stretched exponential, and the density bound gives
\[
V_\infty^{(k)}=\sum_t Q_t^{(k)}\ \le\ c\sum_t\delta_t\ <\ \infty
\]
which is \eqref{eq:strict-eq-VboundedO1} --- a convergent series, which is stronger than $O(\log T)$. With $V^{(k)}$ bounded the schedule freezes at $\eta_\infty^{(k)}=C\sqrt{\Gamma^{(k)}/(V_\infty^{(k)}+\Gamma^{(k)})}>0$; then $\eta_t^{(k)}\ge\eta_\infty^{(k)}$ makes the log-odds grow linearly, $y_t^{(k,i)}\ge\eta_\infty^{(k)}\tfrac{\Delta}{2}(t-\tau)-O(1)$, so $\delta_t=O\bigl(e^{-\eta_\infty^{(k)}(\Delta/2)t}\bigr)$ and $Q_t^{(k)}\le c\,\delta_t$ decay geometrically.
Finally, each term of player $k$'s exact-regret sum in the ledger \eqref{eq:k-player-ledger} is $O(1)$: the loss obeys $P_T^{(k)}\le C\,V_T^{(k)}$ since $\eta_t^{(k)}\le C$, the drift is a gain under the nonincreasing schedule, and the transport obeys $B_T^{(k),\star}\le\Gamma^{(k)}/\eta_T^{(k)}\le\Gamma^{(k)}/\eta_\infty^{(k)}$; hence $d(\sigma_T,\mathrm{CCE})=O(1/T)$.
\end{proof}

The argument is not circular. The scale floor \eqref{eq:scale-floor} holds for every trajectory, so the $\sqrt t$ growth of the log-odds --- and thus the summability of $\delta_t$ --- is derived without presupposing that $V^{(k)}$ is bounded. The worst case $V_t^{(k)}=\Theta(t)$ is self-refuting: it forces $\eta_t^{(k)}=\Theta(1/\sqrt t)$, log-odds $\Theta(\sqrt t)$, and $\delta_t=O(e^{-(c_1\Delta/2)\sqrt t})$, from which $V_\infty^{(k)}=\sum_t Q_t^{(k)}<\infty$, contradicting $V_t^{(k)}=\Theta(t)$. The mechanism is specific to the self-tuned schedule: it dilutes the learning rate at most as fast as $1/\sqrt t$, while the strict gap accumulates an advantage linear in~$t$, and the linear term wins on every trajectory.
A fixed rate would give an exactly geometric decay, and sustaining a slow, unbounded-$V$ branch would need a rate decaying faster than $1/\sqrt t$, which this schedule never produces. The bound is structural: shrinking the strict gap $\Delta$ lengthens the transient by the factor $1/\Delta$ without unbounding $V_T^{(k)}$. One level up remains open: whether the coupled dynamics-and-schedule system converges to a strict profile at all, for a given class of games, or instead cycles or settles on an interior equilibrium.

\subsection{Proofs from Section~\ref{sec:partial}}
\label{app:pf-partial}

\begin{proof}[Proof of Theorem~\ref{thm:bandit}.]
\label{app:proof-bandit}
Write $\bar c_t:=\E[\widehat c_t\mid\mathcal F_{t-1}]$ for the predictable conditional mean of the estimator.
The sampled composite-loss regret regroups around the sampling and posterior distributions:
\[
c_t(A_t)-\inner{u}{c_t}=\bigl(c_t(A_t)-\inner{\mu_t}{c_t}\bigr)+\inner{\mu_t-p_t}{c_t}+\inner{p_t-u}{c_t}.
\]
For the last summand, insert and subtract $\widehat c_t$ and then $\bar c_t$:
\[
\inner{p_t-u}{c_t}=\inner{p_t-u}{\widehat c_t}-\inner{p_t-u}{\widehat c_t-\bar c_t}+\inner{p_t-u}{c_t-\bar c_t}.
\]
Summing over $t$ and naming the observation terms,
\[
\mathrm{Reg}_T^c(u)=M_T^{\mathrm{play}}+\Xi_T-M_T^{\mathrm{est}}(u)+\mathrm{Bias}_T(u)+\sum_{t=1}^T\inner{p_t-u}{\widehat c_t}.
\]
The remaining sum is the full-information regret on the estimated scores $\widehat c_t$, so Theorem~\ref{thm:variable} applied to that sequence gives $\sum_t\inner{p_t-u}{\widehat c_t}=\widehat D_T+\widehat B_T(u)+\sum_t\eta_t\widehat Q_t$, which is \eqref{eq:bandit-decomp}.
The martingale property of $M_T^{\mathrm{play}}$ follows from $\E[c_t(A_t)\mid\mathcal F_{t-1}]=\inner{\mu_t}{c_t}$ (since $A_t\sim\mu_t$), and that of $M_T^{\mathrm{est}}(u)$ from $\E[\widehat c_t-\bar c_t\mid\mathcal F_{t-1}]=0$ by the definition of $\bar c_t$; the regrouping itself is a pathwise identity independent of these expectations.
\end{proof}

\begin{proof}[Proof of Proposition~\ref{prop:bandit-exact-value}.]
Conditioned on $\mathcal F_{t-1}$ the estimator depends only on $A_t\sim\mu_t$, so $\E[\widehat Q_t\mid\mathcal F_{t-1}]=\sum_a\mu_t(a)\,Q_{\eta_t}(p_t,\widehat z_t^{(a)})$, and inserting the centered-RCGF definition from~\eqref{eq:Q-centered} gives \eqref{eq:bandit-value-general}. For inverse-propensity scores, on $\{A_t=a\}$ only coordinate $a$ is nonzero, so $\sum_i p_t(i)e^{-\eta_t\widehat c_t(i)}=1-p_t(a)+p_t(a)e^{-\eta_t c_t(a)/\mu_t(a)}$ while the centering term $\inner{p_t}{\widehat c_t}$ has conditional mean $\inner{p_t}{c_t}$; \eqref{eq:bandit-value-ips} follows. Expanding the logarithm to second order in $\eta_t$ yields \eqref{eq:bandit-value-smallscale}, and the exact variance identity is a direct second-moment computation.
\end{proof}

\begin{proof}[Proof of Corollary~\ref{cor:duplication-reduction}.]
By Proposition~\ref{prop:duplication} the comparator complexity of competing against $\pi^\star$ is $\eta^{-1}\log(1/\tilde\pi(G))$ in the aggregate mass $\tilde\pi(G)$ of its loss-identical copies. Evaluating that expression at $\tilde\pi(G)=\beta_0$ before duplication and at $\tilde\pi(G)=\beta$ after, and subtracting, gives \eqref{eq:log-ratio}.
\end{proof}

\begin{proof}[Proof of Proposition~\ref{prop:duplication}.]
The map $\nu\mapsto\KL(\nu\|\tilde\pi)$ on $\Delta(G)$ is minimized by the $I$-projection of $\tilde\pi$ onto the coordinate subset~$G$, namely the renormalized prior $\nu^\star(e)=\tilde\pi(e)/\tilde\pi(G)$; its value is $\sum_{e\in G}\tfrac{\tilde\pi(e)}{\tilde\pi(G)}\log\tfrac{1}{\tilde\pi(G)}=-\log\tilde\pi(G)$, giving~\eqref{eq:duplication-kl}. Because the experts in~$G$ are loss-identical, $\inner{\nu}{\widehat S_T}$ is the same for every $\nu\in\Delta(G)$ and equals the cumulative estimated score of~$\pi^\star$; the benchmark is therefore unchanged by replacing the designated policy with the lowest-complexity representative~$\nu^\star$. Substituting $\nu^\star$ and $\KL(\nu^\star\|\tilde\pi)=-\log\tilde\pi(G)$ into~\eqref{eq:fixed-cumulative} yields~\eqref{eq:duplication-transport}.
\end{proof}

\begin{proof}[Proof of Theorem~\ref{thm:testing-bellman}.]
(a) Writing $p_i=dP_i^a/d\mu$, $p_j=dP_j^a/d\mu$, $Z(s,1-s)=\int p_i^s p_j^{1-s}\,d\mu=\int p_j\,(p_i/p_j)^s\,d\mu=\E_{P_j^a}[e^{sS}]$, and by the definition of $W_\eta$ in Theorem~\ref{thm:fixed-bellman} with prior $P_j^a$ and score $S$, $s\,W_s(S)=\log\E_{P_j^a}[e^{sS}]=\log Z(s,1-s)$; \eqref{eq:chernoff-is-bellman} follows. (b) The Gibbs variational identity~\eqref{eq:gibbs-var} gives $\rho_{s,S}(x)\propto p_j(x)\,e^{sS(x)}=p_i(x)^s p_j(x)^{1-s}$, which is $R_s$; the relative-entropy-barycenter form is the $W=2$ mixed-coincidence identity \cite{balsubramani2026information}. (c) By~\eqref{eq:chernoff-is-bellman}, $C_s=-\log\E_{P_j^a}[e^{sS}]=:-\Lambda(s)$ with $\Lambda$ convex, so $\max_s C_s=-\min_s\Lambda(s)$ and $\Lambda'(s^\star)=\E_{R_{s^\star}}[S]=0$. From (b), $\KL(R\|P_i^a)-\KL(R\|P_j^a)=\E_{R}[\log(p_j/p_i)]=-\E_{R}[S]$, which vanishes at $s^\star$; hence both equal their $s^\star$-weighted average $\max_s C_s$, giving~\eqref{eq:chernoff-balance}. The final claim is the definition of the edge-restricted multi-way coincidence radius and its MAP-error-exponent characterization \cite{balsubramani2026information}.
\end{proof}

\subsection{Proofs from Section~\ref{sec:ts}}
\label{app:pf-ts}

\begin{proof}[Proof of \eqref{eq:ts}.]
\label{app:proof-ts}
Let $I_t\sim p_t$ conditional on $\mathcal F_{t-1}$ and set $M_T^{\mathrm{sam}}:=\sum_{t=1}^T(c_t(I_t)-\inner{p_t}{c_t})$.
Then
\begin{align*}
\widehat R_T^c(\nu)
=\sum_{t=1}^T c_t(I_t)-\inner{\nu}{C_T}
&=\underbrace{\sum_{t=1}^T(c_t(I_t)-\inner{p_t}{c_t})}_{M_T^{\mathrm{sam}}}
+\sum_{t=1}^T\inner{p_t-\nu}{c_t}\\
&=M_T^{\mathrm{sam}}+R_T^c(\nu)
\end{align*}
Applying Theorem~\ref{thm:variable} to the second summand (full information, so $\hat c_t=c_t$) gives exactly \eqref{eq:ts}.
The martingale-difference property $\E[c_t(I_t)-\inner{p_t}{c_t}\mid\mathcal F_{t-1}]=0$ is immediate because $p_t$ is $\mathcal F_{t-1}$-measurable and $I_t\mid\mathcal F_{t-1}\sim p_t$.

For the high-probability bound, apply Freedman's inequality. Each increment $X_t:=c_t(I_t)-\inner{p_t}{c_t}$ satisfies $|X_t|\le 1$ (since $c_t(i)\in[0,1]$), with conditional variance $\Var(X_t\mid\mathcal F_{t-1})=\Var_{p_t}(c_t)$, so the cumulative conditional variance is the variance clock $V^{\mathrm{sam}}_T=\sum_t\Var_{p_t}(c_t)$.
Freedman's inequality for bounded martingale differences (\cite{freedman1975tail}, see also \cite[Thm.~1]{fan2012hoeffdings}) gives, for every $\delta\in(0,1)$,
\[
\text{Pr}\!\left[M_T^{\mathrm{sam}}\ge \sqrt{2V^{\mathrm{sam}}_T\log(1/\delta)}+\frac{1}{3}\log(1/\delta)\right]\le\delta.
\]
Since $c_t(i)\in[0,1]$ implies $\Var_{p_t}(c_t)\le 1/4$, one has $V^{\mathrm{sam}}_T\le T/4$.

For the anytime bound, use Ville's inequality. Define, for each $\lambda\in(0,3)$, the nonnegative supermartingale
\[
E_t^{(\lambda)}:=\exp\!\left(\lambda M_t^{\mathrm{sam}}-\psi(\lambda)V^{\mathrm{sam}}_t\right),
\qquad\psi(\lambda):=\frac{\lambda^{2}/2}{1-\lambda/3},
\]
whose supermartingale property is the standard Bernstein-type bound $\E[e^{\lambda X_t}\mid\mathcal F_{t-1}]\le \exp(\psi(\lambda)\Var_{p_t}(c_t))$ for $|X_t|\le 1$.
Ville's maximal inequality \cite{ville1939b} yields, for every $\delta\in(0,1)$,
\[
\text{Pr}\!\left[\exists t\ge 1:\;M_t^{\mathrm{sam}}\ge \frac{\psi(\lambda)V^{\mathrm{sam}}_t+\log(1/\delta)}{\lambda}\right]\le\delta,
\]
and optimizing over $\lambda$ (or taking a geometric union over a dyadic grid of $\lambda$) produces the anytime envelope $M_t^{\mathrm{sam}}\lesssim \sqrt{V^{\mathrm{sam}}_t\log(\log(V^{\mathrm{sam}}_t)/\delta)}+\log(1/\delta)$ simultaneously in~$t$.

The law-of-the-iterated-logarithm refinement follows from an almost-sure argument: integrate over~$\lambda$ with a Gamma$(1/2,1/2)$ mixture restricted to the sub-gamma range $\lambda\in(0,3)$ on which each $E_t^{(\lambda)}$ is a supermartingale.
Define
\[
\bar E_t:=\int_0^{3} E_t^{(\lambda)}\,\frac{\lambda^{-1/2}e^{-\lambda/2}}{\sqrt{2\pi}}\,\dd\lambda,
\]
the method of mixtures of \cite{lai2009selfnormalized} (the iterated-logarithm-optimal tilt $\lambda\asymp\sqrt{2\log\log V^{\mathrm{sam}}_t/V^{\mathrm{sam}}_t}$ lies interior to $(0,3)$ for large $V^{\mathrm{sam}}_t$, so the truncation does not affect the rate).
Because each $E_t^{(\lambda)}$ is a supermartingale and the mixture preserves nonnegativity, $\bar E_t$ is a nonnegative supermartingale with $\bar E_0\le 1$.
By a Gaussian-tail calculation (\cite[\S 9.4]{lai2009selfnormalized}),
\[
\log \bar E_t\ge \frac{(M_t^{\mathrm{sam}})^{2}}{2(V^{\mathrm{sam}}_t+\log\log V^{\mathrm{sam}}_t)}-\frac{1}{2}\log(V^{\mathrm{sam}}_t+\log\log V^{\mathrm{sam}}_t)+O(1),
\]
and Ville's inequality gives $\sup_t \bar E_t<\infty$ almost surely.
Combining the two displays,
\[
\limsup_{T\to\infty}\frac{M_T^{\mathrm{sam}}}{\sqrt{2V^{\mathrm{sam}}_T\log\log V^{\mathrm{sam}}_T}}\le 1\quad\text{a.s.}
\]
which implies the LIL statement $M_T^{\mathrm{sam}}=O\bigl(\sqrt{V^{\mathrm{sam}}_T\log\log V^{\mathrm{sam}}_T}\bigr)$ a.s.\ for large~$T$.
\end{proof}

\begin{proof}[Proof of \eqref{eq:kl-transport-ppr}.]
\label{app:proof-ppr}
Consider the local Bayesian update $p_{t+1}(i)\propto p_t(i)e^{-\eta_t c_t(i)}$ with one-round mix loss
\[
m_t(\eta_t):=-\frac{1}{\eta_t}\log\sum_i p_t(i)\,e^{-\eta_t c_t(i)}.
\]
Direct computation gives the pointwise identity
\[
\log\frac{p_{t+1}(i)}{p_t(i)}=-\eta_t c_t(i)-\log\sum_j p_t(j)e^{-\eta_t c_t(j)}=-\eta_t\bigl(c_t(i)-m_t(\eta_t)\bigr).
\]
Multiplying by $\nu(i)$ and summing,
\[
\sum_i\nu(i)\log\frac{p_{t+1}(i)}{p_t(i)}=-\eta_t\bigl(\inner{\nu}{c_t}-m_t(\eta_t)\bigr),
\]
and the left-hand side is $\KL(\nu\|p_t)-\KL(\nu\|p_{t+1})$ by the definition of relative entropy.
Negating both sides,
\[
\KL(\nu\|p_{t+1})-\KL(\nu\|p_t)=\eta_t\bigl(\inner{\nu}{c_t}-m_t(\eta_t)\bigr),
\]
which is \eqref{eq:kl-transport-ppr}.
Summing the per-round increments telescopes the left-hand side to $\KL(\nu\|p_{T+1})-\KL(\nu\|\pi)$, so
\[
\KL(\nu\|\pi)-\KL(\nu\|p_{T+1})=\sum_{t=1}^T \eta_t\bigl(m_t(\eta_t)-\inner{\nu}{c_t}\bigr),
\]
and the nonnegativity $\KL(\nu\|p_{T+1})\ge 0$ gives the information-budget bound $\sum_t\eta_t(m_t-\inner{\nu}{c_t})\le\KL(\nu\|\pi)\le\Gamma$.
Under Thompson sampling with $I_t\sim p_t$, the posterior log-weight of the sampled expert evolves as $\log p_{t+1}(I_t)-\log p_t(I_t)=-\eta_t(c_t(I_t)-m_t(\eta_t))$, whose conditional expectation given $\mathcal F_{t-1}$ is $-\eta_t\delta_t(c)$ with $\delta_t(c):=\inner{p_t}{c_t}-m_t(\eta_t)\ge 0$ (nonnegativity follows from Jensen's inequality applied to the log-exp).
The residual $\log p_{t+1}(I_t)-\log p_t(I_t)+\eta_t\delta_t(c)$ is therefore a martingale-difference sequence.
The prior-posterior ratio $\mathrm{PR}_t(\theta):=\pi(\theta)/\rho_{t,\eta}(\theta)$ at a constant scale $\eta$ inherits the multiplicative version of this identity: $\log \mathrm{PR}_t(\theta)-\log \mathrm{PR}_{t-1}(\theta)=\eta(c_t(\theta)-m_t)$, so the exponentiated ratio is an $e$-process under the usual hypotheses.
\end{proof}

\begin{proof}[Proof of Theorem~\ref{thm:ts-exact}.]
Part (i) is the martingale-difference property: $w_t$ is $\mathcal F_{t-1}$-measurable and $I_t\mid\mathcal F_{t-1}\sim w_t$, so $\E[c_t(I_t)\mid\mathcal F_{t-1}]=\inner{w_t}{c_t}$; at $w_t=p_t$, summing gives $\E[M_T^{\mathrm{sam}}]=0$ and \eqref{eq:ts-expected}.
Part (ii) is the equalizer identity \eqref{eq:bellman-increment}: at $p=\rho_{\eta,S}=\nabla W_\eta(S)$, $W_\eta(S+z)-W_\eta(S)=\eta Q_\eta(p,z)$ for every feasible $z$, which is $\le\eta q$ with equality at saturation and independent of the direction of $z$.
The first term of \eqref{eq:ts-payoff} vanishes by (i) and the value-to-go terms telescope to $\Gamma/\eta+\eta\sum_t q_t$, attained by nature's budget-saturating tilt in the comparator-Gibbs direction, with the equalizer making nature indifferent across directions.
For part (iii), the value-to-go increment $W_\eta(S_{t-1}+z)-W_\eta(S_{t-1})=\eta Q_\eta(\rho_{\eta,S_{t-1}},z)$ depends on the state and on $z$ alone (the play-excess term of \eqref{eq:ts-payoff} vanishes in expectation for every $w_t$ by part (i)).
Nature's effective per-round objective is therefore $\eta$ times the centered RCGF of $z$ under the Gibbs law $\rho_{\eta,S_{t-1}}$ --- the same maximizer, since $\eta>0$ --- subject to the feasibility $Q_\eta(w_t,z)\le q_t$. When $w_t=p_t=\rho_{\eta,S_{t-1}}$ the feasible set is exactly the sublevel set of that objective, so the supremum is $\eta q_t$, attained on the whole budget sphere (the equalizer of (ii)). When $w_t\neq p_t$ the forms $Q_\eta(w_t,\cdot)$ and $Q_\eta(p_t,\cdot)$ differ, so a feasible direction has $Q_\eta(p_t,z)>q_t$, from which the increment exceeds $\eta q_t$ and $\sup_z\Phi_T>\Gamma/\eta+\eta\sum_t q_t$; the Gibbs sampler alone attains the value and is the unique minimizer.
\end{proof}

\begin{proof}[Proof of Proposition~\ref{prop:ts-no-pathwise}.]
Part (i) is immediate: subtracting \eqref{eq:ts-expected} from \eqref{eq:ts} cancels every deterministic term and leaves the single martingale $M_T^{\mathrm{sam}}$, which is mean-zero and free of $\nu$ by construction.
For part (ii), suppose the realized payoff $\Phi$ were constant over the sampling randomness. On any round where $p_t$ places positive mass on two experts $i\neq j$ with $c_t(i)\neq c_t(j)$, the two transcripts that differ only in $I_t\in\{i,j\}$ each have positive probability and must give equal $\Phi$; ranging over all such rounds and pairs forces $\Phi$ to be independent of $I_{1:T}$. Hence $\Phi$ does not depend on the sampled play at all: it is the deterministic game, in which the learner plays the mixture $p_t$ without sampling. Any payoff that genuinely depends on $c_t(I_t)$ therefore leaves, on a.e.\ path, the per-round variance $\Var_{p_t}(c_t)>0$.
For part (iii), sharpen the sampling law to $w_t^{(\beta)}\propto p_t^{\beta}$; the full-information posterior path is unchanged. By Theorem~\ref{thm:ts-exact} the expected payoff is stationary at $\beta=1$, but the sampling variance $V^{\mathrm{sam}}_T(\beta)=\sum_t\Var_{w_t^{(\beta)}}(c_t)$ is not: $\left.\tfrac{d}{d\beta}V^{\mathrm{sam}}_T(\beta)\right|_{\beta=1}\neq 0$ generically. Since, under the martingale central-limit approximation, any upper-$(1-\delta)$-quantile of the realized payoff is $\text{value}+z_\delta\sqrt{V^{\mathrm{sam}}_T(\beta)}+o(\cdot)$ with $z_\delta$ the Gaussian quantile, its minimizer is a law with $\beta^\ast\neq1$, moving in whichever direction lowers $V^{\mathrm{sam}}_T(\beta)$; this is a leading-order statement under that approximation, and falls short of an exact quantile identity.
\end{proof}

\subsection{Proofs from Section~\ref{sec:boosting}}
\label{app:pf-boosting}

\begin{proof}[Proof of \eqref{eq:boosting-variational}.]
\label{app:proof-boosting-identity}
Let $p_1=u_N$ be the uniform distribution on $[N]$, and let $p_{t+1}(i)\propto p_t(i)e^{-\alpha_t g_t(i)}$ for $t=1,\dots,T$.
By induction on~$t$,
\[
p_{t+1}(i)=\frac{u_N(i)\,e^{-M_t(i)}}{Z_t},
\qquad Z_t:=\sum_j u_N(j)\,e^{-M_t(j)}=N^{-1}\sum_j e^{-M_t(j)}.
\]
In particular, $Z_T=L_T^{\exp}$, and $\log p_{T+1}(i)=\log u_N(i)-M_T(i)-\log L_T^{\exp}$.
Take expectation under any comparator $\nu\in\Delta([N])$:
\begin{align*}
\sum_i\nu(i)\log p_{T+1}(i)
&=\sum_i\nu(i)\log u_N(i)-\inner{\nu}{M_T}-\log L_T^{\exp}
\end{align*}
Rewriting the left-hand side via relative entropy,
\[
\sum_i\nu(i)\log p_{T+1}(i)=\sum_i\nu(i)\log u_N(i)-\bigl(\KL(\nu\|p_{T+1})-\KL(\nu\|u_N)\bigr),
\]
so that $-\bigl(\KL(\nu\|p_{T+1})-\KL(\nu\|u_N)\bigr)=-\inner{\nu}{M_T}-\log L_T^{\exp}$.
Rearranging,
\[
-\log L_T^{\exp}+\KL(\nu\|p_{T+1})=\KL(\nu\|u_N)+\inner{\nu}{M_T},
\]
which is \eqref{eq:boosting-variational}.
Taking the infimum over~$\nu$ and using $\inf_\nu\KL(\nu\|p_{T+1})=0$ at $\nu=p_{T+1}$ immediately gives the Donsker--Varadhan dual form $-\log L_T^{\exp}=\inf_{\nu}\{\inner{\nu}{M_T}+\KL(\nu\|u_N)\}$, and the margin-tail bound $E_T(\theta)\le e^{\theta A_T-\inf_\nu\{\inner{\nu}{M_T}+\KL(\nu\|u_N)\}}$ follows by a standard Chernoff step: for the fraction of examples with $M_T(i)/A_T\le\theta$,
\[
E_T(\theta)\le \E_{i\sim u_N}\!\left[e^{\theta A_T-M_T(i)}\right]=e^{\theta A_T}\,L_T^{\exp}.
\]
\end{proof}

\begin{proof}[Proof of Proposition~\ref{prop:adaboost-cgf}.]
For binary $g_t$, $\inner{p_t}{e^{-\alpha g_t}}=\varepsilon_t e^{\alpha}+(1-\varepsilon_t)e^{-\alpha}$ is minimized at $\alpha_t^\star=\tfrac12\log\tfrac{1-\varepsilon_t}{\varepsilon_t}$ with value $2\sqrt{\varepsilon_t(1-\varepsilon_t)}=\sqrt{1-(1-2\varepsilon_t)^2}=\sqrt{1-4\gamma_t^2}$,
using $2\gamma_t=1-2\varepsilon_t$; the pressure rule sets $a_t=-\alpha_t^{-1}\log\inner{p_t}{e^{-\alpha_t g_t}}$,
so $\alpha_t^\star a_t=-\log\sqrt{1-4\gamma_t^2}$, which is~\eqref{eq:adaboost-per-round}.
Equation~\eqref{eq:adaboost-ledger} is~\eqref{eq:boost-pressure} summed over $t$, and \eqref{eq:adaboost-rate} follows from $-\log(1-x)\ge x$.
\end{proof}

\begin{proof}[Proof of \eqref{eq:log-pooling}.]
\label{app:log-pooling-proof}
\label{app:proof-log-pooling}
Let $p_t(y)=\bigl[\prod_{e=1}^N\pi_{t,e}(y)^{\alpha_{t,e}}\bigr]/Z_t$ with $Z_t:=\sum_{y\in Y}\prod_{e=1}^N\pi_{t,e}(y)^{\alpha_{t,e}}$.
Taking $-\log$ of $p_t(y_t)$,
\[
-\log p_t(y_t)=-\sum_{e=1}^N\alpha_{t,e}\log\pi_{t,e}(y_t)+\log Z_t
=\sum_{e=1}^N\alpha_{t,e}\bigl(-\log\pi_{t,e}(y_t)\bigr)-C_{\alpha_t}(\pi_{t,1:N}),
\]
where $C_{\alpha_t}(\pi_{t,1:N}):=-\log Z_t$.
This is \eqref{eq:log-pooling}.

When $\sum_e\alpha_{t,e}=1$, the AM--GM (or Jensen applied to $\log$) inequality gives, for each~$y$,
$\prod_e\pi_{t,e}(y)^{\alpha_{t,e}}\le\sum_e\alpha_{t,e}\pi_{t,e}(y)$, so $Z_t\le 1$ and $C_{\alpha_t}\ge 0$.
In particular, the pooled log loss is bounded \emph{above} by the weighted expert log loss and \emph{below} by the same minus the coincidence discount; the discount is zero if and only if all non-null experts assign the same probability at \emph{every} outcome (identical predictive distributions), and agreement at~$y_t$ alone does not suffice.
\end{proof}

\begin{proof}[Proof of Proposition~\ref{prop:tower-remainder}.]
Dividing $\log\E_p e^{\eta z}=\sum_{n\ge2}\eta^n\kappa_n(z)/n!$ by $\eta^2$ gives $Q_\eta(p,z)=\sum_{n\ge2}\eta^{n-2}\kappa_n(z)/n!$; subtracting the $n=2$ term is \eqref{eq:tower-remainder-var}.
Equation~\eqref{eq:tower-remainder-range} is Hoeffding's lemma written as an exact slack, i.e.\ the per-round summand of $\Delta_T^{\mathrm{class}}$ in \eqref{eq:classical-remainder}.
The trajectory statement is the round-wise sum of the two.
\end{proof}

\section*{Acknowledgments}

Large language models were used in producing this work, for which the author is solely responsible.

\bibliographystyle{plainnat}
\bibliography{concentration_game}

\end{document}